%% file: preprint.tex
\documentclass{article} 
\usepackage{iclr2027_conference,times}
\usepackage[T1]{fontenc}  

\input{math_commands.tex}

\usepackage{appendix}
\usepackage{url}
\usepackage{amsthm}
\usepackage{graphicx}
\usepackage{booktabs}
\usepackage{xcolor}   
\usepackage{colortbl} 
\usepackage{listings}
\usepackage{tcolorbox}
\tcbuselibrary{breakable,skins}
\definecolor{verylightgrey}{rgb}{0.95,0.95,0.95}
\definecolor{darkgrey}{rgb}{0.25,0.25,0.25}
\newtcolorbox{promptbox}[2][]{%
  colback=verylightgrey,
  colframe=black,
  coltitle=white,
  colbacktitle=darkgrey,
  boxrule=0.5pt,
  arc=5mm,
  outer arc=5mm,
  leftrule=1pt,
  rightrule=1pt,
  toprule=1pt,
  bottomrule=1pt,
  left=10pt,
  right=10pt,
  top=10pt,
  bottom=10pt,
  boxsep=0pt,
  enhanced,
  title={\centering\strut#2\strut},
  fonttitle=\bfseries\large,
  #1
}
\lstdefinestyle{promptstyle}{%
  basicstyle=\ttfamily\scriptsize,
  breaklines=true, breakatwhitespace=false,
  breakindent=0pt, breakautoindent=false,
  columns=fixed, keepspaces=true, showstringspaces=false,
  aboveskip=0pt, belowskip=0pt,
}
\lstnewenvironment{promptlisting}{\lstset{style=promptstyle}}{}
\usepackage{hyperref}

\usepackage{placeins}
\definecolor{darkblue}{rgb}{0,0,0.5}
\definecolor{rpgblue}{HTML}{0072B2}    
\definecolor{rpggreen}{HTML}{097D05}   
\definecolor{rpgamber}{HTML}{C88A00}   
\definecolor{rpgorange}{HTML}{D55E00}  
\definecolor{rpgtint}{HTML}{E8F2F9}
\hypersetup{
    colorlinks = true,
    citecolor  = brown,      
    linkcolor  = blue,       
    urlcolor   = darkblue,   
}
\usepackage{subcaption}  
\usepackage{cleveref}
\usepackage{bbm}
\usepackage{titletoc}
\usepackage[acronym,nowarn]{glossaries}[=v4.46]
\glsdisablehyper

\theoremstyle{plain}
\newtheorem{theorem}{Theorem}
\newtheorem{lemma}{Lemma}
\newtheorem{proposition}{Proposition}
\newtheorem{corollary}{Corollary}
\theoremstyle{definition}
\newtheorem{assumption}{Assumption}

\theoremstyle{remark}

\theoremstyle{plain}
\crefname{theorem}{Theorem}{Theorems}\Crefname{theorem}{Theorem}{Theorems}
\crefname{lemma}{Lemma}{Lemmas}\Crefname{lemma}{Lemma}{Lemmas}
\crefname{proposition}{Proposition}{Propositions}\Crefname{proposition}{Proposition}{Propositions}
\crefname{corollary}{Corollary}{Corollaries}\Crefname{corollary}{Corollary}{Corollaries}
\crefname{definition}{Definition}{Definitions}\Crefname{definition}{Definition}{Definitions}
\crefname{remark}{Remark}{Remarks}\Crefname{remark}{Remark}{Remarks}

\usepackage{fontawesome5}
\newcommand{\fnicon}[1]{\makebox[1.3em][l]{#1}}
\makeatletter
\newcommand{\blfootnote}[1]{%
  \begingroup
  \renewcommand{\thefootnote}{}%
  \renewcommand{\thempfootnote}{}%
  \long\def\@makefntext##1{\leftskip=1.8em\noindent##1}%
  \footnotetext{#1}%
  \endgroup}
\makeatother
  
\iclrfinalcopy
\title{Reward-rate Policy Gradient for Efficient Machine Learning Engineering Agents}

\author{Muhang Tian \quad Sherry Yang \\
New York University \\
\texttt{mt4193@nyu.edu} \\
}

\newlength{\algsetuplabel}
\algrenewcommand\algorithmicrequire{\makebox[\algsetuplabel][l]{\textbf{Require:}}}
\algrenewcommand\algorithmicensure{\makebox[\algsetuplabel][l]{\textbf{Ensure:}}}
\algnewcommand\algorithmicinitialize{\makebox[\algsetuplabel][l]{\textbf{Initialize:}}}
\algnewcommand\Initialize{\item[\algorithmicinitialize]}

\crefname{assumption}{Assumption}{Assumptions}\Crefname{assumption}{Assumption}{Assumptions}
\input{glossary}

\begin{document}

\maketitle
\ificlrfinal\blfootnote{\fnicon{\faGlobe}Website: \url{https://muhang-tian.com/reward-rate/}\\ \fnicon{\faGithub}Code: \url{https://github.com/MuhangTian/Reward-rate-Policy-Gradient}}\fi

\begin{abstract}
\input{contents/abstract}
\end{abstract}

\section{Introduction \label{intro}}
\input{contents/intro}

\section{Preliminaries \label{prelim}}
\input{contents/prelim}

\section{Reward-rate Policy Gradient \label{method}}
\input{contents/method}

\section{Reward-rate for MLE Agent \label{results}}
\input{contents/results}

\section{Related work \label{related}}
\input{contents/related}

\section{Conclusion}
\input{contents/conclusion}

\subsection*{AI Use Statement}
\input{contents/ai}

\bibliography{iclr2027_conference}
\bibliographystyle{iclr2027_conference}

\newpage
\appendix
\crefname{appendix}{appendix}{appendices}
\Crefname{appendix}{Appendix}{Appendices}
\begin{appendices}
    \startcontents[appendices]
    \printcontents[appendices]{l}{1}{%
        \section*{Appendix}%
        \setcounter{tocdepth}{3}%
    }
    \vspace{1em}
    \newpage
    \section{Derivation}
    \input{appendix/derivation}
    \newpage
    \section{Theoretical Analysis}\label{appendix:theory}
    \input{appendix/theory}
    \newpage
    \section{Implementation}
    \input{appendix/implementation}
    \newpage
    \section{Experiments}
    \input{appendix/experiments}
    \newpage
\end{appendices}

\end{document}

%% file: math_commands.tex
\usepackage{amsmath,amsfonts,bm}
\usepackage{algorithm}      
\usepackage[noEnd, commentColor=rpggreen]{algpseudocodex}

\def\1{\bm{1}}

\def\vone{{\bm{1}}}

\def\vpi{{\bm{\pi}}}

\def\va{{\bm{a}}}

\def\vs{{\bm{s}}}

\def\vz{{\bm{z}}}

\def\mF{{\mathbf{F}}}

\DeclareMathAlphabet{\mathsfit}{\encodingdefault}{\sfdefault}{m}{sl}
\SetMathAlphabet{\mathsfit}{bold}{\encodingdefault}{\sfdefault}{bx}{n}

\def\gA{{\mathcal{A}}}
\def\gB{{\mathcal{B}}}

\def\gN{{\mathcal{N}}}

\def\gS{{\mathcal{S}}}

\def\gU{{\mathcal{U}}}

\def\gX{{\mathcal{X}}}

\newcommand{\diag}{\mathrm{diag}}
\newcommand{\E}{\mathbb{E}}

\newcommand{\R}{\mathbb{R}}

\DeclareMathOperator*{\argmax}{arg\,max}

\newcommand{\indic}{\mathbbm{1}}

%% file: glossary.tex
\newacronym{rl}{RL}{reinforcement learning}
\newacronym{lm}{LM}{language model}
\newacronym{ml}{ML}{machine learning}

\newacronym{smdp}{SMDP}{Semi-Markov Decision Process}
\newacronym{mle}{MLE}{machine learning engineering}
\newacronym{ppo}{PPO}{Proximal Policy Optimization}
\newacronym{ours}{RPG}{Reward-rate Policy Gradient}
\newacronym{niw}{NIW}{Normal-Inverse-Wishart}
\newacronym{npg}{NPG}{natural policy gradient}
\newacronym{spg}{SPG}{standard policy gradient}
\newacronym{c-ucb}{C-UCB}{continuous-time upper confidence bound}

%% file: contents/abstract.tex
Traditional \gls{rl} techniques focus on maximizing expected cumulative reward, where each action assumes to take a constant unit of time.
However, this assumption does not hold for agentic \gls{rl} tasks such as \gls{mle} agents, where actions involve data loading, feature engineering, and model training that take variable durations.
Efficiency matters in modern agentic \gls{rl} where actions are costly.
To address this limitation, we adapt from continuous-time \gls{rl} and \gls{smdp} formulation and propose \gls{ours}, where we focus on optimizing the \textit{reward rate} --- the long-term reward per unit of time.
\gls{ours} estimates the reward rate from off-policy samples, then charges each action for the time it consumes at that rate.
We first conduct theoretical analysis in the bandit setting to establish that \gls{ours} approximates the optimal reward rate and empirically demonstrate it outperforms baselines while avoiding enumeration over the policy space, a known issue for an existing method.
We then further apply \gls{ours} on a small language model (Qwen3.5-4B) with self-improvement loops and empirically show it obtains higher rewards within a fixed time budget than vanilla \gls{rl} on MLE-Bench and NanoGPT, with a 19.2\% and 85.7\% margin, respectively.
Our method provides a practical solution for optimizing performance under wait time considerations in modern agentic \gls{rl} tasks, where actions interact with external environments and cost time.


%% file: contents/intro.tex
\gls{rl} has been used extensively to train \glspl{lm} and agents, where the standard approach is to maximize expected cumulative reward \citep{christiano_deep_2017,zhang2025landscape}.
Recently, there has been interest in \gls{mle} agents, where the goal is to automate the \gls{mle} pipeline to achieve best performance on held-out data \citep{chan2025mle,yang_reinforcement_2025,nam2026mle}.
Current approaches mostly apply a strong \gls{lm} in a search procedure: AIDE explores the space of candidate programs as a tree \citep{jiang_aide_2025}, while MLE-STAR retrieves an initial solution with a search engine and then refines its targeted components \citep{nam2026mle}.
A separate line trains the model itself with \gls{rl} on performance score \citep{yang_reinforcement_2025}.
However, none of these methods perform gradient updates that fundamentally change agents' behavior to \emph{learn} to generate more efficient solutions.
In \gls{mle}, each action takes variable time and resources, so it is important to consider samples' costs for practicality --- we want agents that are not only good, but also \emph{efficient}.


One solution is to use multi-objective \gls{rl} and search for the Pareto frontier between performance and time \citep{roijers_survey_2013,hayes_practical_2022}, but computing the set of optimal solutions is often infeasible and compute-heavy \citep{hayes_practical_2022}.
These limitations motivate an objective that \emph{intrinsically} aims for long-run efficiency.
One natural formulation is the long-run reward per unit of time studied in \glsreset{smdp}\gls{smdp} \citep{sutton_between_1999,das_solving_1999,gyorgy_continuous_2007}.
Intuitively, the method uses \gls{rl} to maximize relative rewards, calculated from total rewards minus the amount that it would have earned if time were spent optimally.
However, performing this computation requires knowing the optimal reward rate \textit{apriori}, which is often not the case in practice, giving a chicken-or-egg dilemma.
Existing approaches mainly use two ways to tackle this.
Directly optimizing the reward rate with \gls{smdp} policy gradient objective \citep{sutton_policy_1999} requires estimating the rate of the current policy from on-policy rollouts, an operation that is still expensive when actions are costly.
The continuous-time bandit approach \citep{gyorgy_continuous_2007} estimates the optimal reward rate from historical samples, but it requires an exhaustive search over the policy space, infeasible in modern agentic \gls{rl} settings where action space is large.

To that end, we propose \glsreset{ours}\gls{ours}, which uses off-policy historical samples to estimate the reward rate under the greedy policy and maximizes for relative reward (\Cref{fig:overview}).
\gls{ours} avoids the limitations of the two previous methods.
It does not need time-costly on-policy rollouts for each gradient step, and does not require searching over the policy space.
Additionally, \gls{ours} is simple to adopt in the current \gls{rl} pipeline, since it only requires a reward-rate estimator and uses the relative reward as the learning signal.
We first study \gls{ours} in the bandit setting and theoretically prove that, in expectation, if the greedy policy's reward rate is estimated well, then \gls{ours} is equivalent to the Dinkelbach iteration \citep{dinkelbach_nonlinear_1967,schaible_fractional_1976}, and \gls{npg} \citep{amari_natural_1998} is suited for reward-rate optimization due to preconditioning on Fisher information.
Experiments on the bandit setting confirm that \gls{ours} obtains lower regret than the reward-rate upper-confidence-bound algorithm by \cite{gyorgy_continuous_2007} without the need to search over the policy space.
These findings motivate us to apply \gls{ppo}, a first-order approximation of \gls{npg} methods \citep{kakade_natural_2001,schulman_trust_2015}, on MLE-Bench and NanoGPT.
Compared against vanilla \gls{rl} baselines on the same time budget, \gls{ours} shows a 19.2\% average improvement in performance across 20 out of 22 MLE-Bench tasks and an 85.7\% margin on NanoGPT.
Our findings demonstrate that \gls{ours} is a practical solution for future agentic \gls{rl} frameworks where actions are costly.

\begin{figure}[t!]
    \centering
    \includegraphics[width=\textwidth]{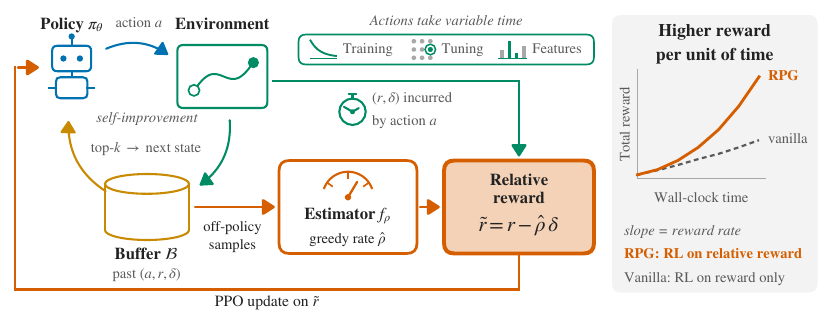}
    \caption{\textbf{Overview of \gls{ours} in \gls{mle}, where our goal is to optimize for reward per unit of time in agentic \gls{rl} contexts}.
    Our agent under policy $\pi_{\theta}$ is asked to self-improve upon its previous \gls{mle} scripts from buffer $\gB$, and each action $a$ is a proposed script that executes in variable time $\delta$ and admits a performance score $r$ such as AUC. 
    $\hat\rho$ is estimated from past samples in buffer $\gB$ and captures the cost per unit of time.
    Relative reward $\tilde r$ is then used in \gls{ppo} to perform \gls{rl} training.
    Maximizing the relative reward optimizes for reward rate, giving higher reward per unit of time and more efficient \gls{rl} agents.
    }
    \label{fig:overview}
\end{figure}

%% file: contents/prelim.tex
\subsection{Semi-MDP Policy Gradient\label{prelim:smdp}}
We introduce the notation and setup for the \gls{smdp} setting where each action takes a variable amount of time \citep{sutton_between_1999,das_solving_1999}. 
Let states be $s \in \gS$  and actions be $a \in \gA$. At each step $t$, the policy $\pi(a_t \mid s_t)$ is a conditional distribution on the actions. 
Taking the action $a_t$ gives reward $r_{t+1} := r(s_t, a_t)$ and time $\delta_{t+1} := \delta(s_t,a_t)$, and the transition model is a joint conditional distribution $p(s_{t+1},\delta_{t+1},r_{t+1} \mid s_t,a_t)$. Our objective is to maximize the long-term reward per unit of time:
\begin{equation}
    \rho(\pi) = \lim_{t \to \infty} \frac{\E_\pi \big[\sum_{j=1}^{t} r_j\big]}{\E_\pi \big[\sum_{j=1}^{t} \delta_{j}\big]}, \quad \pi^* = \argmax_{\pi} \rho(\pi).
    \label{eq:rr}
\end{equation}
For a parameterized policy $\pi_\theta$, taking the gradient gives
\begin{equation}
    \nabla_{\theta} \, \rho(\pi_{\theta}) = \frac{\sum_{s} d^{\pi_{\theta}}(s) \sum_{a} \nabla_{\theta} \pi_{\theta}(a \mid s) Q^{\pi_{\theta}}(s,a)}{\E_{\pi_{\theta}}[\delta]},
    \label{prelim:pg thm}
\end{equation}
where $Q^{\pi_\theta}(s,a) = \E_{\pi_\theta}\big[r - \rho(\pi_\theta)\,\delta + V^{\pi_\theta}(s') \mid s,a \big]$ is the \textit{differential} action-value function (see full derivation in \Cref{supp:method:smdp}). 
Empirically, the division by $\E_{\pi_{\theta}}[\delta]$ can be dropped since modern adaptive optimizers like Adam \citep{kingma_adam_2015} and gradient clipping reduce its impact on the gradients.
Thus, the implementation is a straight-forward plug-in using the \textit{relative reward} $\tilde r_t$ into existing \gls{rl} algorithms
\begin{align}
    \tilde{r}_t &:= r_t - \hat{\rho}_t\,\delta_t,
    \qquad
    e_t = \tilde{r}_t + V_\phi(s_{t+1}) - V_\phi(s_t),
    \qquad
    \hat{A}_t = \sum_{l \ge 0} \lambda^{l}\, e_{t+l},
     \\
    \mathcal{L}(\theta) &= \E_t\Big[\min\big\{w_t(\theta)\hat{A}_t,\;
    \mathrm{clip}\big(w_t(\theta), 1-\epsilon, 1+\epsilon\big)\hat{A}_t\big\}\Big],
    \qquad
    w_t(\theta) = \frac{\pi_\theta(a_t \mid s_t)}{\pi_{\theta_{\text{old}}}(a_t \mid s_t)},
    \label{eq:ppo}
\end{align}
where $\hat\rho_t$ is a Monte Carlo approximation of $\rho(\pi)$ from on-policy rollouts.
Intuitively, the relative reward $\tilde{r}_t$ is the gain obtained for the selected action when $\delta_t$ units of time are used at a cost of $\rho(\pi)$, so actions are ranked by their relative performance per unit of time.

\subsection{Continuous Time Bandit\label{prelim:bandit}}
Reward-rate \gls{rl} has also been studied in the continuous-time bandit setting \citep{gyorgy_continuous_2007}. 
The agent observes context $x_t$ drawn i.i.d. from a fixed distribution on a finite set $\gX$, and each action $a_t \in \gA$ yields a reward $r_{t+1}$ and time $\delta_{t+1}$, with $r \in [r_{\min}, r_{\max}]$ and $\delta \in [\delta_{\min},\delta_{\max}]$, $\delta_{\min} > 0$.
The policy $u \in \gU = \{u: \gX \to \gA\}$ is assumed to be stationary and deterministic. 
Rewards $\{r_j\}_{j=1}^{t}$ and durations $\{\delta_j\}_{j=1}^{t}$ are drawn i.i.d. across steps from a fixed distribution given $(x,a)$.
The objective is to maximize the \textit{relative value} under the optimal reward rate:
\begin{equation}
    q^*(x,a) \;=\; \E\big[r(x,a)\big] - \rho^*\,\E\big[\delta(x,a)\big],
    \qquad
    \rho^* = \max_{u \in \gU} \rho(u),
    \label{bandit:obj}
\end{equation}
so that a policy is optimal when it selects $a^* = \argmax_{a\in\gA} q^*(x,a)$ at every context.
When the reward rate is not optimal, $\rho \neq \rho^*$, we denote the relative value as $q_{\rho}(x,a)$.
Regret is thus defined as
\begin{equation}
    R_t = \rho^* \sum_{j=1}^{t} \delta_j - \sum_{j=1}^{t} r_j.
    \label{bandit:regret}
\end{equation}
To minimize regret, the algorithm starts by estimating $\rho^*$ with $\hat\rho_t$ using the past history $\{(r_j,\delta_j)\}_{j=1}^{t}$ and performs an exhaustive search over all possible policies
\begin{equation}
    \hat{\rho}_t = \max_{u \in \gU}\,\big\{\bar{\rho}_t(u) - b_t(u)\big\},
    \qquad
    \bar{\rho}_t(u) = \frac{\sum_{j=1}^{t} \indic\{a_{j-1} = u(x_{j-1})\}\, r_j}
                           {\sum_{j=1}^{t} \indic\{a_{j-1} = u(x_{j-1})\}\, \delta_j},
    \label{bandit:rate}
\end{equation}
which subtracts a confidence width $b_t(u)$, giving a lower bound satisfying $\hat{\rho}_t \le \rho^*$ with high probability. 
Actions are then selected optimistically with the estimated relative value, where $\bar{r}_t$ and $\bar{\delta}_t$ are sample averages and $c_t$ is a confidence bonus
\begin{equation}
    a_t = \argmax_{a \in \gA}\;\big\{\bar{r}_t(x_t,a) - \hat{\rho}_t\, \bar{\delta}_t(x_t,a) + c_t(x_t,a)\big\}.
    \label{bandit:ucb}
\end{equation}
This algorithm yields $O(\log T)$ expected regret bound but requires enumerating $|\gU| = |\gA|^{|\gX|}$ policies at each step.
More details on the algorithm's construction of $c_t$ and $b_t(u)$ are in \Cref{supp:bandit:bonus}.

%% file: contents/method.tex
In this section, we propose \gls{ours} for optimizing the reward rate practically.
We provide an outline for \gls{ours} in \Cref{method:algo} that shows how we avoid previous algorithms' limitations.
\Cref{method:theory} shares our theoretical insights, which were validated through experiments in \Cref{method:bandit}.

\subsection{Main Algorithm \label{method:algo}}
\paragraph{Current limitations.}
We first describe limitations of current approaches and then share the intuition behind our algorithm design.
The \gls{smdp} policy gradient theorem in \Cref{prelim:pg thm} tells us it is viable to do gradient-based optimization on a parameterized policy to maximize reward rate.
However, the core constraint in practice is that it requires multiple rollouts of the on-policy trajectory $\{(r_j,\delta_j)\}_{j=1}^{T}$ to estimate $\rho(\pi)$.
Since the time $\delta_j$ for each step is variable and costly, this makes the estimation take a long time, at least $\sum_{j\le T} \delta_j$, for each gradient update, making it impractical for training efficient \gls{mle} agents.
On the other hand, the continuous-time bandit approach is an alternative by optimizing for the relative reward using an off-policy reward rate estimate $\hat\rho_t$ (\Cref{bandit:rate}), but the limitation being that it involves enumeration over all possible policies.

\paragraph{Proposed algorithm.}
Given the limitations discussed above, our design principle for creating a practical reward-rate \gls{rl} algorithm is as follows.
We first use a predictor $f_{\rho}$ fitted on the historical buffer $\gB$ to approximate the reward rate and use the predicted rate $\hat\rho$ to calculate relative reward $\tilde r$.
This design allows us to avoid sampling on-policy long-horizon rollouts.
Unlike the continuous-time bandit setting where policies are deterministic \citep{gyorgy_continuous_2007}, \glspl{lm} have stochastic policies, so we apply policy gradient updates on expected relative reward to train \glspl{lm} agents.
Our algorithm is shown in \Cref{method:algo:algorithm}. 
Ideally, the reward-rate predictor $f_\rho$ should approximate the optimal $\rho^*$, which then serves as $\hat\rho$ for the relative reward.
However, directly estimating $\hat\rho \approx \rho^*$ is difficult, even in the bandit setting, since it requires enumeration over the policy space.
Thus, to make the approach practical, we set $f_\rho$ to estimate the greedy policy's reward rate.
This choice is supported by our results in the bandit setting (\Cref{method:theory,method:bandit}), where we demonstrate that a softmax policy with \gls{npg} updates converges to the optimal $\rho^*$ through both theory and experiments.
Empirically on MLE-Bench and NanoGPT, we show that PPO also performs well despite it being a first-order approximation of \gls{npg}.
The estimator $f_\rho$ can be anything that estimates the greedy reward rate, such as Bayesian inference, Robbins-Monro iteration, or neural networks.

\begin{algorithm}[t]
\normalcolor 
\caption{\gls{ours}}
\label{method:algo:algorithm}
\begin{algorithmic}[1]
\Require Policy $\pi_\theta$, greedy reward rate estimator $f_{\rho}$, buffer $\gB$, $\hat\rho_1 \leftarrow 0$.
\LComment {\footnotesize $f_{\rho}$ is a \gls{niw} posterior, but it can also be Robbins-Monro or neural networks.}
\For{$t = 0$ \textbf{to} $T-1$}
    \State Sample actions $a_{t} \sim \pi_{\theta_t}(\cdot \mid s_{t})$, execute $a_t$ and observe reward $r_{t+1}$ and time $\delta_{t+1}$.
    \State Store $a_t$, $\delta_{t+1}$ and $r_{t+1}$ in buffer $\gB$.
    \State Fit $f_{\rho}$ on $\gB$, then get $\hat\rho_{t+1} = f_{\rho}(\gB) $. \Comment {\footnotesize Fit \acrshort{niw} posterior and draw samples.}
    \State Calculate relative reward $\tilde{r}_{t+1} = r_{t+1} - \hat{\rho}_{t+1}\,\delta_{t+1}$ and advantage $\hat{A}_{t+1}$.
    \State $\theta_{t+1} \leftarrow \texttt{PPO}(\theta_t, \hat A_{t+1})$.
\EndFor
\end{algorithmic}
\end{algorithm}

\subsection{Theoretical Analysis\label{method:theory}}
\paragraph{Setup.}
We share our findings in the bandit setting, which provides motivation for approximating the greedy reward rate for $\hat\rho$ and using \gls{ppo} in our algorithm design.
Our analysis is focused on the expected setting with a softmax policy $\pi_{\theta}(a\mid x) = \exp(\theta_{a,x}) / \sum_{b} \exp(\theta_{b,x})$.
Our assumptions are as follows: (1) boundedness in reward and time, (2) positive gap condition for relative value $\min_x \big[q^*(x,a^*(x)) - \max_{a \ne a^*(x)} q^*(x,a)\big] > 0$, and (3) contexts are drawn i.i.d. from $p(x)$.
Details of the analysis are in \Cref{appendix:theory}.

\paragraph{Approximate Dinkelbach iteration.}
We first analyze under which conditions an arbitrary sequence of reward rates $\{\rho_k\}_{k=1}^{\infty}$ would be near the optimal $\rho^*$.
\Cref{results:theory:contraction} shows that if $\rho_{k+1}$ is within a asymmetric $\gamma$-tolerance band around the greedy reward rate $\rho(u_{\rho_k})$ at every step $k \ge 2$, then the sequence would be contained within an interval around the optimal reward rate $\rho_k \in [(1 - R\gamma/ (1+\gamma)) \rho^*, (1+\gamma)\rho^*]$ as $k \to \infty$, which suggests we are approximately doing a Dinkelbach iteration \citep{dinkelbach_nonlinear_1967}.
When the estimation is exact, $\gamma = 0$, the sequence reaches $\rho^*$ in finite steps.
\Cref{results:theory:contraction} suggests that we can use an estimation of the greedy reward rate as $\hat\rho$ for relative reward calculation and iterate greedily to approximate $\rho^*$.
This finding holds tremendous value in practice, since predicting the greedy reward rate at the current $\rho_k$ is an easier task than estimating $\rho^*$, where the latter requires exhaustive search over the policy space but the former can be approximated with historical samples in $\gB$ by selecting top solutions according to $\tilde r (\rho_k) = r - \rho_k \, \delta$.

\begin{theorem}\label{results:theory:contraction}
    Let $\gamma \in [0,1]$ be a tolerance term, $\rho_1 \in [0, \rho_{\max}]$, and $\chi := 1 - \delta_{\min} / \delta_{\max} \in (0,1)$.
    For every $k \ge 1$, let $u_{\rho_k}(x) \in \argmax_a q_{\rho_k}(x,a)$ be a greedy policy under any tie-breaking and $\rho(u_{\rho_k})$ is the reward rate of the greedy policy, let $\tilde\rho_{k+1}$ be any value with the sequence satisfying the condition $\rho(u_{\rho_k}) / (1+\gamma) \le \tilde\rho_{k+1} \le (1+\gamma)\,\rho(u_{\rho_k})$, $\rho_{k+1} = \mathrm{clip}_{[0,\rho_{\max}]}(\tilde\rho_{k+1})$
    where $\mathrm{clip}_{[0,\rho_{\max}]}(z) := \min\{\max\{z, 0\},\, \rho_{\max}\}$, $\rho_{\max} = r_{\max} / \delta_{\min}$.
    Then, for every $k \ge 2$,
    \begin{equation}
        \rho_k \in \Big[\big(1 - R\gamma/ (1+\gamma)\big) \rho^* - \big(\chi / (1+\gamma)\big)^{k-2} \rho_{\max}, \, (1+\gamma)\rho^*\Big],
        \label{supp:theory:eq:thm1:bound}
    \end{equation}
    where $R = \delta_{\max} / \delta_{\min} > 1$. If $\gamma = 0$, then the sequence increases monotonically $\rho_2 \le \rho_3 \le \cdots \le \rho^*$ and the iteration reaches $\rho_k = \rho^*$ after finite steps.
\end{theorem}

\paragraph{\Gls{npg} and reward-rate \gls{rl}.}
\Cref{results:theory:contraction} suggests that if the predicted rate approximates the greedy rate, then it would be contained within an interval around the optimal $\rho^*$, but it does not make any statements about the training dynamics, so we further analyze the behavior of $\pi_\theta$.
Define $A_{\theta}(\rho, x) = (q_{\rho}(x, \cdot) - \langle \pi_{\theta}(\cdot \mid x), q_{\rho}(x, \cdot) \rangle \, \vone)$, where $\vone = [1\ldots 1] \in \R^{|\gA|}$.
The gradient update rules are
\begin{equation*}
    \theta'_x \leftarrow \theta_{x} + \eta \, A_{\theta}(\rho, x) \quad \text{(NPG)},\quad \theta'_x \leftarrow \theta_{x} + \eta \, \pi_{\theta}(\cdot \mid x) \odot A_{\theta}(\rho, x) \quad \text{(Standard PG)}.
    \label{method:updates}
\end{equation*}
We thus hypothesize that \gls{npg} is more suited for reward-rate \gls{rl} since it has less coupling between $\pi_\theta$ and the magnitude of $A_\theta$.
Coupling would be undesirable since $\theta_x$ is updated proportionally to the policy weight $\pi_{\theta}(\cdot \mid x)$, reducing the signal for actions with large advantage $A_{\theta}$ but low $\pi_{\theta}(\cdot \mid x)$\footnote{In standard \gls{rl} setting, this finding has already been observed \cite{kakade_natural_2001}, but in reward-rate \gls{rl}, $A_{\theta}$ is a function of $\rho$ which is estimated from policy $\pi_{\theta}$, so the amount of coupling is stronger.}.
We establish the validity of \gls{npg} updates in \Cref{results:theory:npg-dynamics}.
It shows that the natural gradient updates are equivalent to a Boltzmann policy on the relative value $q_{\bar\rho_{k}}(x,a)$ with a running average rate $\bar\rho_{k}$.
If $\rho_1$ is initialized within $[0, \rho^*]$, and greedy reward rate estimation is exact $\rho_{k+1} = \rho(u_{\bar\rho_k})$, then the error between $\bar\rho_k$ and $\rho^*$ shrinks in $O(k^{-(1-\chi)})$.
Furthermore, when $\bar\rho_k$ is within $\bar\varepsilon$ of $\rho^*$, which \Cref{eq:thm2:km-rate} guarantees after finitely many steps, the greedy action at $\bar\rho_k$ becomes the optimal one, and the policy increases the optimal arm's probability towards $1$.
This finding suggests that \gls{npg} methods bring performance benefits in reward-rate \gls{rl} and motivates us to apply \gls{ppo} in practice, a first-order approximation of \gls{npg} \citep{schulman_proximal_2017}.

\begin{theorem}\label{results:theory:npg-dynamics}
    Let $\rho_1,\rho_2,\ldots$ be a sequence of reward rates and let $\theta^{(k)}$ be the softmax
    logits after $k$ \gls{npg} update steps with learning rate $\eta$ from the zero initialization, with
    $\bar\rho_k = \sum_{n \le k}\rho_n / k$.
    Then for every $x$ and $k \ge 1$, $\pi_{\theta^{(k)}}$ is the Boltzmann policy of $q_{\bar\rho_k}(x,a)$
    with inverse temperature $\eta k$.
    If moreover $\rho_1 \in [0,\rho^*]$ and $\rho_{k+1} = \rho(u_{\bar\rho_k})$, then $\bar\rho_k$
    increases to $\rho^*$ in $O(k^{-(1-\chi)})$
    \begin{equation}
        \rho^* - \bar\rho_k \;\le\; (\rho^* - \bar\rho_1)\,\exp(1-\chi)\,(k+1)^{-(1-\chi)}.
        \label{eq:thm2:km-rate}
    \end{equation}
    Furthermore, let $\delta_{\mathrm{spr}} \in (0, \delta_{\max} - \delta_{\min}]$ be the maximum time gap, and $\Delta^*$ be the minimum relative value gap under $\rho^*$, define $\bar\varepsilon := \Delta^*/(2\delta_{\mathrm{spr}})$, for every $k \ge k_0$, the greedy policy is optimal
    \begin{equation}
        \pi_{\theta^{(k)}}(a^* \mid x) \ge 1 - (|\gA| - 1)\,\exp\big({-\eta\,\Delta^* k/2}\big), \quad k_0 := \Big\lceil \big(\exp(1-\chi)\,(\rho^* - \bar\rho_1)/\bar\varepsilon\big)^{1/(1-\chi)} \Big\rceil.
        \label{eq:thm2:tail}
    \end{equation}
\end{theorem}

\subsection{Bandit Validation\label{method:bandit}}
\begin{figure}[th]
    \centering
    \begin{subfigure}{\textwidth}
        \centering
        \includegraphics[width=\textwidth]{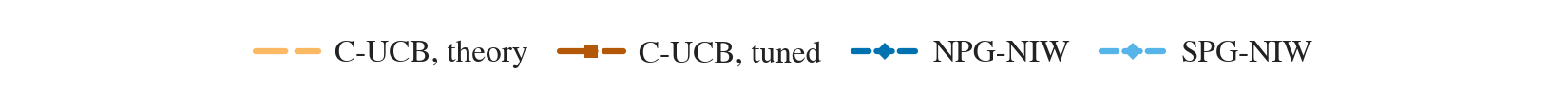}
    \end{subfigure}\\
    \begin{subfigure}{0.49\textwidth}
        \centering
        \includegraphics[width=\textwidth]{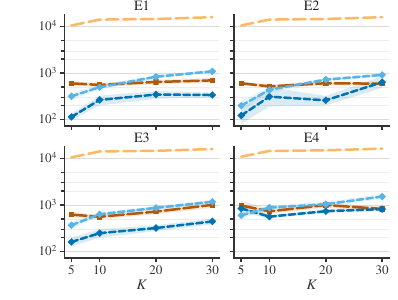}
        \caption{4 contexts $|\gX| = 4$.}
        \label{fig:bandit:summary:ctx}
    \end{subfigure}
    \hfill
    \begin{subfigure}{0.49\textwidth}
        \centering
        \includegraphics[width=\textwidth]{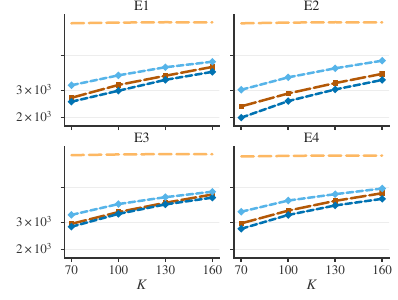}
        \caption{64 contexts $|\gX| = 64$.}
        \label{fig:bandit:summary:x64}
    \end{subfigure}
    \caption{\textbf{Our proposed approach \gls{npg}-\acrshort{niw} obtains similar or lower regret than C-UCB while avoiding exhaustive search over policies}. 
    Standard policy gradient (SPG-\acrshort{niw}) also performs worse than \gls{npg}, verifying our theory that \gls{npg}'s decoupling helps reward-rate \gls{rl}.
    y-axis is regret at the final step over $100$ trials, x-axis is the number of arms $K$, and 4 environments are considered: (E1) independent reward and time, (E2) time and reward have $0.8$ correlation, (E3) time and reward have $-0.8$ correlation, and (E4) reward and time are lognormal with $+0.5$ correlation.
    }
    \label{fig:bandit:summary}
\end{figure}
Our theory suggests that \gls{npg} has a decoupling effect on reward-rate optimization, and using greedy reward rate as $\hat\rho$ can also approximate the optimal rate as iterations grow.
The former brings performance benefits, whereas the latter is a feasibility benefit in large action and context spaces where an exhaustive search over the policy space is intractable.
We conduct experiments to examine whether these benefits hold true in practice.
Our finding is in \Cref{fig:bandit:summary}, and \gls{npg}-\acrshort{niw} uses softmax policy with \gls{npg} updates and \gls{niw} posterior fitted on greedy historical samples for $\hat\rho$ estimation.
\Gls{c-ucb} is the approach by \cite{gyorgy_continuous_2007}.
We first observe that if we use \gls{c-ucb}'s constants exactly as in the paper, it breaks since the constants are too large (\gls{c-ucb}, theory), but treating them as tunable hyperparameters allows it to work (\gls{c-ucb}, tuned).
When only having small context and action spaces (\Cref{fig:bandit:summary:ctx}), \gls{c-ucb} uses exhaustive search for $\hat\rho$ estimation.
When the context and action space is large (\Cref{fig:bandit:summary:x64}), \gls{c-ucb} cannot afford to enumerate over all policies, so we use an approximation where 2048 policies are drawn uniformly at random from past visited ones to predict $\hat\rho$.
Across both settings, we observe that \gls{npg}-\gls{niw} obtains lower or similar regret as \gls{c-ucb} baselines, thus confirming our hypothesis that greedy reward rate estimation can be used for $\hat\rho$ without enumerating the policy space.
When comparing \gls{npg} with standard policy gradient updates (SPG-\gls{niw}), we observe that it consistently obtains higher regret than \gls{npg}, supporting our theoretical finding that \gls{npg} provides performance benefits.
More experiments and implementation details are in \Cref{appendix:exp:bandit,appendix:exp:exact,appendix:exp:approx}.

%% file: contents/results.tex
In this section, we present our results demonstrating that \gls{ours} finds good solutions in less time than vanilla \gls{rl}, suggesting that reward-rate \gls{rl} is a practical solution for training efficient agents.
We provide our empirical results on MLE-Bench and NanoGPT in \Cref{results:mle,results:nanogpt}.

\subsection{Instantiating Language Models for Reward-rate Policies \label{method:mle}}
\paragraph{Environment setup.}
The theory in \Cref{method:theory} suggests that we should approximate the greedy reward rate for $\hat\rho$ and use PPO for training, but training \gls{mle} agents also requires deliberate environment setup, which we share in this section.
We start with a base prompt as our initial state $s_0$ containing task descriptions such as data loading path, meaning of features, and the performance metric used for the evaluation.
We also explicitly ask the \gls{lm} agent to take efficiency into consideration by encouraging it to use a subset of features.
An action $a_0 \sim \pi(\cdot \mid s_0)$ is the generated script by the agent, and the code is executed in a sandbox environment to produce time $\delta_1$ and reward $r_1$.
If the script successfully executes, $r_1$ would be a performance score.
A timeout budget $\delta_{\max}$ is set as the upper limit of the time allowed for script execution.
For scripts that fail to run, such as when they have a bug, we encounter trouble assigning their reward rate since their reward is $0$ and $\tilde{r} = -\hat{\rho}\,\delta$ is maximized by $\delta \to 0$.
This implies a script that fails on its first line dominates all honest attempts.
Thus, we charge every failure the full time limit $\delta_{\max}$ and a constant penalty of $-10$, preventing the policy from reward-hacking by failing fast, and executable scripts have higher rewards than non-executable ones.

\paragraph{State transition with self-improvement.}
Besides setting up the environment and reward properly, we also want our agent to improve based on its own previous attempts, so we would ideally have a self-improving agent making progress on the task over time.
We discuss how we connect the agent's own solutions to form self-improvement loops.
At every step, we store the observed transition $(a_t,r_{t+1},\delta_{t+1})$ into the buffer $\gB$.
The next state $s_{t+1}$ is constructed by concatenating the base prompt $s_0$ with the previous code $a'$ and a self-improvement phrase $i_0$ explicitly asking the agent to improve upon the performance score $r'$ and time $\delta'$.
The selection of $(a',r',\delta')$ from buffer $\gB$ is based on a top-$k$ selection according to the relative reward $\tilde r(\hat\rho_{t+1}) = r' - \hat\rho_{t+1}\,\delta'$.
This naturally gives a self-improvement chain where the agent is tasked with making progress on its best proposed solutions observed thus far.

\paragraph{Reward-rate predictor.}
At each step $t$, our agent uses its history of attempts to predict the reward rate $\hat\rho$.
The reward-rate estimator $f_{\rho}$ is a \gls{niw} posterior with joint Gaussian likelihood \citep{murphy_machine_2012} fitted from historical samples observed in buffer $\gB$.
We model the pair as $(r_{\mathrm{acc}}, \log \delta_{\mathrm{acc}})$, where $r_{\mathrm{acc}}$ and $\delta_{\mathrm{acc}}$ are the accumulated reward and time, and we take the logarithm because we need time to be positive. 
We fit the posterior at each step, and sample $M$ reward rates from it and take the 95th percentile $\hat\rho_{t+1} = \mathrm{Quantile}_{0.95} (\{ \rho_m\}_{m=1}^{M})$.
More details are in \Cref{appendix:exp:rr-pred}.
\begin{figure}[t]
    \centering
    {\captionof{table}{\textbf{\gls{ours} vs. vanilla \gls{rl} on 22 MLE-Bench Lite tasks}, evaluated on mean performance score across self-improving chains within a fixed-time budget.
    $(\uparrow)$ and $(\downarrow)$ means higher and lower is better, respectively.
    $\%$ improv. is the relative improvement of \gls{ours} over vanilla. On average, \gls{ours}
    performs 19.2\% better than vanilla, after dropping the highest and lowest numbers.}%
    \label{tab:mlebench_meanatT}}
    {\scriptsize
    \setlength{\tabcolsep}{3pt}
    \begin{tabular}{l>{\columncolor{rpgtint}}lll@{\hskip 8pt}l>{\columncolor{rpgtint}}lll}
    \toprule
    Task & \gls{ours} & Vanilla & \% improv. & Task & \gls{ours} & Vanilla & \% improv. \\
    \midrule
        ranzcr-clip $(\uparrow)$ & \textbf{0.089} $\pm$ \textbf{0.046} & 0.014 $\pm$ 0.014 & $+$530.5 & plant-pathology $(\uparrow)$ & \textbf{0.526} $\pm$ \textbf{0.048} & 0.498 $\pm$ 0.047 & $+$5.6 \\
        aptos2019 $(\uparrow)$ & \textbf{0.564} $\pm$ \textbf{0.046} & 0.315 $\pm$ 0.120 & $+$78.9 & pizza $(\uparrow)$ & \textbf{0.760} $\pm$ \textbf{0.008} & 0.728 $\pm$ 0.019 & $+$4.4 \\
        mlsp-birds $(\uparrow)$ & \textbf{0.333} $\pm$ \textbf{0.154} & 0.226 $\pm$ 0.120 & $+$46.9 & tabular-2022 $(\uparrow)$ & \textbf{0.840} $\pm$ \textbf{0.013} & 0.806 $\pm$ 0.028 & $+$4.2 \\
        siim-isic $(\uparrow)$ & \textbf{0.660} $\pm$ \textbf{0.021} & 0.456 $\pm$ 0.112 & $+$44.8 & whale $(\uparrow)$ & 0.493 $\pm$ 0.005 & \textbf{0.749} $\pm$ \textbf{0.091} & $-$34.2 \\
        textnorm-en $(\uparrow)$ & \textbf{0.880} $\pm$ \textbf{0.048} & 0.714 $\pm$ 0.068 & $+$23.2 & leaf-classif $(\downarrow)$ & \textbf{0.259} $\pm$ \textbf{0.063} & 0.564 $\pm$ 0.059 & $+$54.0 \\
        textnorm-ru $(\uparrow)$ & \textbf{0.889} $\pm$ \textbf{0.021} & 0.730 $\pm$ 0.022 & $+$21.7 & spooky-author $(\downarrow)$ & \textbf{0.440} $\pm$ \textbf{0.015} & 0.690 $\pm$ 0.003 & $+$36.2 \\
        tabular-2021 $(\uparrow)$ & \textbf{0.907} $\pm$ \textbf{0.010} & 0.778 $\pm$ 0.046 & $+$16.5 & nyc-taxi $(\downarrow)$ & \textbf{4.309} $\pm$ \textbf{0.029} & 5.349 $\pm$ 0.387 & $+$19.4 \\
        insults $(\uparrow)$ & \textbf{0.834} $\pm$ \textbf{0.025} & 0.761 $\pm$ 0.024 & $+$9.6 & dog-breed $(\downarrow)$ & \textbf{4.786} $\pm$ \textbf{0.001} & 4.924 $\pm$ 0.139 & $+$2.8 \\
        jigsaw-toxic $(\uparrow)$ & \textbf{0.948} $\pm$ \textbf{0.009} & 0.868 $\pm$ 0.028 & $+$9.2 & dogs-vs-cats $(\downarrow)$ & \textbf{0.688} $\pm$ \textbf{0.005} & 0.703 $\pm$ 0.001 & $+$2.1 \\
        histopathologic $(\uparrow)$ & \textbf{0.594} $\pm$ \textbf{0.107} & 0.554 $\pm$ 0.021 & $+$7.2 & nomad2018 $(\downarrow)$ & \textbf{0.064} $\pm$ \textbf{0.002} & 0.064 $\pm$ 0.002 & $+$0.3 \\
        aerial-cactus $(\uparrow)$ & \textbf{0.954} $\pm$ \textbf{0.018} & 0.893 $\pm$ 0.010 & $+$6.8 & denoising-docs $(\downarrow)$ & 0.094 $\pm$ 0.002 & \textbf{0.085} $\pm$ \textbf{0.013} & $-$10.1 \\
    \bottomrule
    \end{tabular}}
    \includegraphics[width=.96\textwidth]{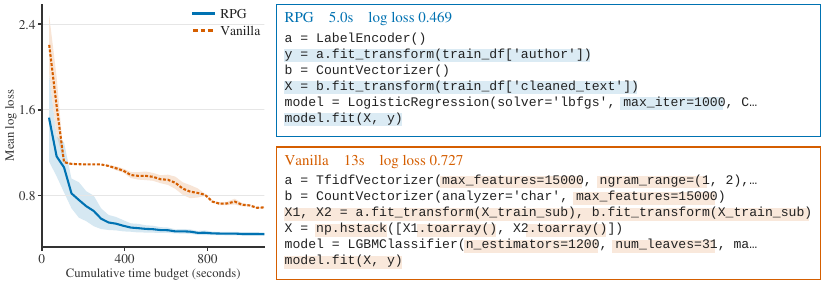}
    \captionof{figure}{\textbf{Example of generated scripts by \gls{ours} and vanilla \gls{rl}}. \gls{ours} proposes scripts that use time more efficiently by training on one type of feature with logistic regression, whereas vanilla \gls{rl} chooses to train on two types of features with 1200 gradient-boosted trees.}
    \label{fig:mlebench_programs}
\end{figure}

\subsection{Evaluation Results on MLE-Bench\label{results:mle}}
We provide our results for MLE-Bench in \Cref{tab:mlebench_meanatT}.
We use Qwen3.5-4B as our \gls{lm} and apply the method in \Cref{method:mle}.
Our baseline is vanilla \gls{rl}, where rewards are performance scores only and time budget is $\delta_{\max}$.
We run \gls{ours} and the baseline until convergence.
Our evaluation metric is the mean reward obtained by the self-improvement chains at every time budget.
Specifically, at every step, there are $\texttt{batch\_size}$ self-improvement chains, and for each chain, a time $\delta$ and reward $r$ are earned.
We calculate the mean reward $r$ across chains at every cumulative time budget $t \in [0, 200,\ldots, T]$ with $200$-second intervals.
The final time $T$ is decided by taking the cumulative time containing 95\% of the self-improvement chains.
Each task is evaluated with three repetitions.
Across the $22$ tasks \gls{ours} show an average 19.2\% improvement over vanilla at the final time budget, after dropping the task with highest and lowest improvements.
The two methods' proposed \gls{mle} scripts have distinct differences in wall clock time.
An example of the budget curve along with the generated script is shown in \Cref{fig:mlebench_programs} for the spooky-author task.
Vanilla uses a concatenation of different features and trains a gradient-boosted tree in 13 seconds with log loss 0.727, whereas \gls{ours} uses count features along with logistic regression to obtain 0.469 log loss in 5 seconds.
The budget curve behind every task is provided in \Cref{appendix:exp:mlebench}, and implementation details are in \Cref{appendix:exp:mle}.

\subsection{Evaluation Results on NanoGPT\label{results:nanogpt}}
\begin{figure}[t]
    \centering
    \includegraphics[width=.8\textwidth]{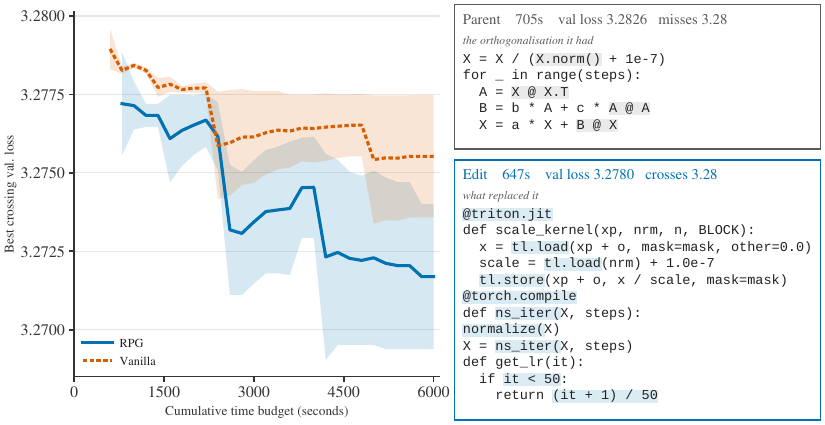}
    \caption{\textbf{Budget curve and an example self-improvement step on NanoGPT}. (Left) At the final budget, \gls{ours} improves over vanilla by $85.7\%$. (Right) Example of a self-improvement edit by \gls{ours}, which rewrites Newton--Schulz in the Muon optimizer as a triton kernel with a compiled iteration and adds a learning-rate warmup. The new edit crosses the loss target in less time.}
    \label{fig:nanogpt:budget}
\end{figure}
\paragraph{Motivation and setup.}
The MLE-Bench setting discussed above is based on Kaggle competitions.
To extend the scope of reward-rate \gls{rl} agents, we also apply \gls{ours} to the NanoGPT speedrun challenge \citep{karpathy_nanogpt_2022,jordan_modded_2024}, where the task is to train a tiny \gls{lm} down to a validation cross-entropy loss of $3.28$ on FineWeb \citep{penedo_fineweb_2024} with eight GPUs as fast as possible.
This task makes the time taken to reach the $3.28$ target loss a natural choice to optimize for.
Thus, as a baseline, we implement vanilla \gls{rl} using time to reach the $3.28$ target loss as reward, and compare with \gls{ours} using relative reward as the learning signal.

\paragraph{Setup.}
Both vanilla and \gls{ours} have a linear bonus terms for crossing the $3.28$ target loss, $w \, (3.28-L)$.
NanoGPT requires writing tens of thousands of tokens in its code.
Thus, to ensure efficient training and focus on making modifications rather than writing executable code, we ask the agent to read prompt $s_t$ and propose high-level delta code changes along with their reasoning for the modification, which is the action $a_t$.
A separate, independent, frontier model with deterministic output (zero temperature) serves as the implementer, which reads the action $a_t$ proposed by the agent and implements the changes.
More details are in \Cref{appendix:exp:nanogpt}.

\paragraph{Results.}
We present our results on NanoGPT in \Cref{fig:nanogpt:budget}, where we observe that \gls{ours} gives lower validation loss given the same time budget than vanilla \gls{rl}.
At the final time budget, we observe that \gls{ours} has an 85.7\% margin below $3.28$ compared to the baseline.
This can be explained by its objective of using the reward rate, which intrinsically optimizes for the loss per unit of time.
An example of a self-improvement edit by \gls{ours} is shown on the right side of \Cref{fig:nanogpt:budget}, where the agent writes a triton kernel, a compiled iteration, and a learning rate warm-up in its parent script to drive the loss to the target in less time.

%% file: contents/related.tex
\paragraph{Machine learning engineering agents.}
Many works currently ask whether an agent can carry out machine learning engineering tasks in an end-to-end fashion \citep{nam2026mle,tang2023ml,chan2025mle,li2024autokaggle,zhang2026mlrc, qiang2026mle}.
Similar lines of work also include using agents on data science tasks such as data preprocessing and hyperparameter tuning \citep{grosnit2024large,bendinelli2025exploring,zhang2023automl,liu2024large}.
Most agentic approaches in these tasks involve scaffolding a frozen model, such as by searching the solution space as a tree \citep{jiang_aide_2025}, which do not train \gls{mle}-specific experience into the model weights.
One of the closest works to ours is by \cite{yang_reinforcement_2025}, who trained a small \gls{lm} with \gls{rl} on MLE-Bench and observed that actions take a variable amount of time, but the focus was on mitigating the time bias favoring faster solutions in asynchronous \gls{rl} updates.
By contrast, our work directly treats time as a cost term in the \gls{rl} objective using the \gls{smdp} reward-rate formulation and trains agents to learn to take efficient actions.

\paragraph{Reward-rate RL.}
We acknowledge that the reward-rate objective itself is classical.
The \gls{smdp} formulation provides the theoretical framework for our method \citep{sutton_between_1999}, and the continuous-time bandit \citep{gyorgy_continuous_2007} gives the idea of using off-policy estimates for the optimal reward rate.
Similar problems also appear in renewal processes with bandit feedback \citep{cayci_learning_2019} where the objective is to maximize reward within a fixed horizon.
A core mathematical toolkit is Dinkelbach iteration for fractional programming \citep{dinkelbach_nonlinear_1967,schaible_fractional_1976}, where all of these works aforementioned rely on it to transform a ratio objective into a sequence of linear ones.
Budgeted and knapsack bandits \citep{badanidiyuru_bandits_2013,ding_multi-armed_2013,xia_budgeted_2016} also consider actions with costs, but they enforce a hard constraint on a finite budget.
In \gls{rl} for \glspl{lm} literature, the environment is episodic with actions assumed to have a constant cost \citep{schulman_proximal_2017,deepseek-ai_deepseek-r1_2025}.
A survey on agentic \gls{rl} with \glspl{lm} \citep{zhang2025landscape} also uses the same formulation, with no notion of how long an action can take.
\cite{guan_dynamic_2025} trades off latency against cost, but requires exploring the Pareto frontier in a multi-agent framework.
To our knowledge, no existing work applies reward-rate \gls{rl} in modern \gls{lm} agents.

\paragraph{Self-improvement.}
Performing self-improvement by asking an agent to revise its own previous attempts is a well-established pattern, such as critiquing its outputs directly \citep{madaan_self-refine_2023}, reflecting on feedback from the environment \citep{shinn_reflexion_2023}, training on the attempts that happened to succeed \citep{zelikman_star_2022}, and rewriting the scaffolding \citep{zelikman_stop_2023}.
The idea of self-improvement loops has also been explored in discovery tasks \citep{yuksekgonul_learning_2026}.
Closer to our setting, the loop can be driven by execution feedback, where a model repairs its own program from the error message that the program produced \citep{chen_teaching_2023}.
The loop can also be used on an agent archive to create a tree of diverse and capable agents through evolutionary search \citep{zhang_darwin_2025}.
\cite{zeng2025bstar} proposes using a balance score for the exploration-exploitation trade-off and updating models on samples that are above a score threshold.
Our work focuses on using self-improving loops as state transitions to train agents.

%% file: contents/conclusion.tex
We proposed a practical reward-rate \gls{rl} algorithm, namely \gls{ours}, for training efficient \gls{mle} agents using the \gls{smdp} formulation, where each action takes variable time, and the objective is long-run reward per unit of time.
\gls{ours} uses greedy policy's reward rate estimates for the relative reward calculation and directly maximizes it through \gls{ppo}.
It avoids enumeration over all policies and time-costly on-policy rollouts, two limitations for directly applying the \gls{smdp} methods in practice.
Through theoretical analysis, simulation experiments, and empirical experiments on MLE-Bench and NanoGPT, we demonstrated that \gls{ours} produces an agent policy that learns to take actions efficiently and maximizes reward.
There are some limitations in this work.
The scope of the empirical experiments is focused on \gls{mle} tasks, but other agentic \gls{rl} tasks where actions cost variable time exist, so it would be interesting to see how reward-rate \gls{rl} does on these tasks.
Theoretically, our analysis focuses on the population-level setting in expectation.
A finite-sample analysis could provide further insight.

%% file: contents/ai.tex
In this work, we used generative AI tools for: implement methods, support qualitative and thematic data analysis, process experiment data, assist in the writing of proofs, and design or provide feedback on research methodology or experiments.
We have not used generative AI tools for: generate synthetic data sets, help develop theoretical models or conceptual frameworks, provide critical ingredients for proving mathematical claims, formulate mathematical claims, propose or refine hypotheses, assist with translation, clean and reformat dataset, interpret results.
Additionally, we used generative AI tools for: create or modify scientific figures or images, create or edit software code, creation of artifacts, draft parts of a research paper, summarize or analyse existing literature.
We have reviewed all AI-assisted work. 
We have checked through all the code edited by AI, and all experiments involving AI were verified.
We take responsibility for the final content of this work, including text, claims or artifacts produced with the aid of generative AI.

%% file: appendix/derivation.tex
\subsection{Semi-MDP Policy Gradient Theorem}\label{supp:method:smdp}
\paragraph{Setup and notation.}
We consider a finite \gls{smdp} \citep{sutton_between_1999,das_solving_1999}
$\mathcal{M} = (\gS, \gA, p, \delta_{\min}, \delta_{\max})$ with finite state space
$\gS$ and finite action space $\gA$. At each \emph{decision epoch} the agent
observes a state $s \in \gS$, selects an action $a \in \gA$, and the environment
returns a triple $(r, \delta, s')$ drawn from the transition kernel
\begin{equation}
    p(s', r, \delta \mid s, a),
    \qquad
    p(s' \mid s, a) \;=\; \int\!\!\int p(s', r, \delta \mid s,a)\,\mathrm{d}r\mathrm{d}\delta,
    \label{eq:smdp-kernel}
\end{equation}
where $r \in [r_{\min}, r_{\max}]$ is the reward accumulated over the transition
and $\delta \in [\delta_{\min}, \delta_{\max}]$ is its time (or duration).
We write $\bar{r}(s,a) := \E[r \mid s, a]$ and $\bar{\delta}(s,a) := \E[\delta \mid s,a]$
for the conditional means under \Cref{eq:smdp-kernel}. Crucially, the kernel does
not depend on the policy parameters, so $\nabla_\theta \bar{r}(s,a) = 0$ and
$\nabla_\theta \bar{\delta}(s,a) = 0$.
 
Let $\pi_\theta(a \mid s)$ be a differentiable stochastic policy parameterized by
$\theta \in \mathbb{R}^{d}$. We drop the subscript $\theta$ for brevity when clear from context. We abbreviate, for a fixed policy $\pi$,
\begin{equation}
    \bar{\delta}^{\pi}(s) \;:=\; \E_{\pi}[\delta \mid s] \;=\; \sum_{a} \pi(a \mid s)\, \bar{\delta}(s,a).
    \label{eq:mean-dwell-state}
\end{equation}

Let $p^{\pi}(s' \mid s) := \sum_{a} \pi(a \mid s)\, p(s' \mid s,a)$ denote the
transition matrix of the Markov chain. We write $d^{\pi}$ for its unique stationary distribution,
\begin{equation}
    d^{\pi}(s') \;=\; \sum_{s} d^{\pi}(s) \sum_{a} \pi(a \mid s)\, p(s' \mid s, a)
    \qquad \forall\, s' \in \gS .
    \label{eq:stationarity}
\end{equation}
We assume that
\begin{itemize}
    \item $0 < \delta_{\min} \le \delta \le \delta_{\max} < \infty$.
    \item $|r| \le \max\{|r_{\min}|,|r_{\max}|\} < \infty$.
    \item $\gS$ and $\gA$ are finite.
\end{itemize}
Under these assumptions, the long-run \emph{reward rate} (gain per unit
time) of $\pi$ is well defined and independent of the initial state:
\begin{equation}
    \rho(\pi) \;:=\; \lim_{T \to \infty} \frac{\E_{\pi}\big[\sum_{t=1}^{T} r_t\big]}{\E_{\pi}\big[\sum_{t=1}^{T} \delta_t\big]}
    \;=\; \frac{\E_{\pi}[r]}{\E_{\pi}[\delta]},
    \qquad
    \E_{\pi}[\delta] \;:=\; \sum_{s} d^{\pi}(s)\, \bar{\delta}^{\pi}(s),
    \label{eq:reward-rate}
\end{equation}
with $\E_{\pi}[r] := \sum_{s} d^{\pi}(s) \sum_{a} \pi(a\mid s) \bar r(s,a)$.
 
The differential state- and action-value functions of $\pi$ are
\begin{equation}
    V^{\pi}(s) = \E_{\pi}\big[\, r - \rho(\pi)\delta + V^{\pi}(s') \mid s \,\big],
    \qquad
    Q^{\pi}(s,a) = \E\big[\, r - \rho(\pi)\delta + V^{\pi}(s') \mid s,a \,\big],
    \label{eq:bellman}
\end{equation}
so that $V^{\pi}(s) = \sum_{a} \pi(a \mid s) Q^{\pi}(s,a)$.

\paragraph{Proof for the policy gradient.}
\begin{equation}
    \nabla_{\theta} \, \rho(\pi_{\theta}) = \frac{\sum_{s} d^{\pi_{\theta}}(s) \sum_{a} \nabla_{\theta} \pi_{\theta}(a \mid s) Q^{\pi_{\theta}}(s,a)}{\E_{\pi_{\theta}}[\delta]},
    \label{appendix:method:smdp-policy-gradient}
\end{equation}
where
\begin{equation*}
    V^\pi(s) = \E_{\pi}\big[ r - \rho(\pi)\delta + V^\pi(s') \mid s \big], \quad Q^\pi(s,a) = \E_{\pi}\big[r - \rho(\pi)\delta + V^\pi(s') \mid s,a \big].
\end{equation*}
\begin{proof}
For simplicity, we drop $\theta$ when clear from context; all gradients are with respect to
$\theta$. Differentiating $V^{\pi}(s) = \sum_{a} \pi(a\mid s) Q^{\pi}(s,a)$ and
applying the product rule,
\begin{align}
    \nabla V^{\pi}(s)
    &= \sum_{a} \Big[ \nabla \pi(a \mid s)\, Q^{\pi}(s,a) + \pi(a\mid s)\, \nabla Q^{\pi}(s,a)\Big]. \label{eq:pf-1}
\end{align}
By definition,
$Q^{\pi}(s,a) = \bar r(s,a) - \rho(\pi)\bar\delta(s,a) + \sum_{s'} p(s' \mid s,a) V^{\pi}(s')$.
Since the transition kernel does not depend on $\theta$, we have
$\nabla \bar r(s,a) = 0$ and $\nabla \bar\delta(s,a) = 0$, hence
\begin{equation}
    \nabla Q^{\pi}(s,a)
    \;=\; -\,\nabla \rho(\pi)\, \bar\delta(s,a) \;+\; \sum_{s'} p(s' \mid s,a)\, \nabla V^{\pi}(s').
    \label{eq:pf-2}
\end{equation}
Substituting \Cref{eq:pf-2} into \Cref{eq:pf-1} and using
\Cref{eq:mean-dwell-state},
\begin{equation}
    \nabla V^{\pi}(s)
    = \sum_{a} \Big[\nabla \pi(a \mid s) Q^{\pi}(s,a) + \pi(a\mid s) \sum_{s'} p(s'\mid s,a) \nabla V^{\pi}(s')\Big]
      \;-\; \nabla \rho(\pi)\, \bar\delta^{\pi}(s).
    \label{eq:pf-3}
\end{equation}
Multiplying \Cref{eq:pf-3} by $d^{\pi}(s)$ and summing over $s \in \gS$,
\begin{align}
    \sum_{s} d^{\pi}(s) \nabla V^{\pi}(s)
    &= \sum_{s} d^{\pi}(s) \sum_{a} \nabla \pi(a \mid s) Q^{\pi}(s,a) \nonumber \\
    &\quad + \underbrace{\sum_{s} d^{\pi}(s) \sum_{a} \pi(a \mid s) \sum_{s'} p(s' \mid s,a) \nabla V^{\pi}(s')}_{(\star)}
      \;-\; \nabla \rho(\pi) \sum_{s} d^{\pi}(s)\, \bar\delta^{\pi}(s).
    \label{eq:pf-4}
\end{align}
Exchanging the order of summation in $(\star)$ and using \Cref{eq:stationarity} gives,
\begin{equation}
    (\star) \;=\; \sum_{s'} \Big[\sum_{s} d^{\pi}(s) \sum_{a} \pi(a\mid s) p(s' \mid s,a)\Big] \nabla V^{\pi}(s')
            \;=\; \sum_{s'} d^{\pi}(s') \nabla V^{\pi}(s').
    \label{eq:pf-5}
\end{equation}
The left-hand side of \Cref{eq:pf-4} and the term \Cref{eq:pf-5} are the same
finite quantity, they therefore cancel,
leaving
\begin{equation}
    0 \;=\; \sum_{s} d^{\pi}(s) \sum_{a} \nabla \pi(a \mid s) Q^{\pi}(s,a)
        \;-\; \nabla \rho(\pi)\, \E_{\pi}[\delta],
    \label{eq:pf-6}
\end{equation}
where we used $\sum_{s} d^{\pi}(s) \bar\delta^{\pi}(s) = \E_{\pi}[\delta]$ from
\Cref{eq:reward-rate}. Since $\E_{\pi}[\delta] \ge \delta_{\min} > 0$ by the boundeness assumption,
we may divide by it, which completes the proof.
\end{proof}
 
\subsection{More on \cite{gyorgy_continuous_2007}}\label{supp:bandit:bonus}
 
We give the explicit form of the two confidence terms $b_t(u)$ and $c_t(a,s)$ appearing in \Cref{bandit:rate,bandit:ucb} \citep{gyorgy_continuous_2007}.
 
Let 
\begin{equation}
    T_u(t) = \sum_{j=1}^{t} \indic\{a_j = u(s_j)\},
\end{equation}
be the number of rounds on which the played action agreed with policy $u$, and
\begin{equation}
    T_a(s,t) = \sum_{j=1}^{t} \indic\{a_j = a,\, s_j = s\},
\end{equation}
be the number of times action $a$ was played in context $s$. Both widths are built from the common scale factor
\begin{equation}
    \omega_{t,n} \;=\; \sqrt{\frac{\log\big(t\sqrt{|\gU| + 1}\big)}{n}},
    \qquad |\gU| = |\gA|^{|\gS|}.
    \label{eq:omega}
\end{equation}
 The width on the rate estimate is
\begin{equation}
    b_t(u) \;= \omega_{t,\,T_u(t)}\; \sqrt{2\kappa},
    \qquad
    \kappa \;=\; 2\max\left\{
        \frac{(r_{\max}-r_{\min})^2}{\delta_{\min}^2},\;\;
        \frac{r_{\max}^2\,(\delta_{\max}-\delta_{\min})^2}{\delta_{\min}^4}
    \right\},
    \label{eq:b-width}
\end{equation}
and the width on the action index is
\begin{equation}
    c_t(a,s) \;=\; (\alpha_0 + \alpha_1)\,\omega_{t,\,T_a(s,t)},
    \qquad
    \alpha_1 = \delta_{\max}\sqrt{2\kappa},
    \label{eq:c-width}
\end{equation}
\begin{equation}
    \alpha_0 = \sqrt{8\max\left\{(r_{\max}-r_{\min})^2,\;\;
        \frac{r_{\max}^2\,(\delta_{\max}-\delta_{\min})^2}{\delta_{\min}^2}\right\}}.
    \label{eq:alpha0}
\end{equation}
The two widths are indexed differently: $b_t(u)$ shrinks in the number of rounds consistent with policy $u$, while $c_t(a,s)$ shrinks in the number of pulls of action $a$ in context $s$. The former is subtracted and the latter added, since $\hat\rho_t$ enters \Cref{bandit:ucb} with a negative sign; pessimism about $\rho^*$ is optimism about $q^*$.

%% file: appendix/theory.tex
\subsection{Preliminaries}\label{supp:theory:prelim}
\paragraph{Setup and notation.}
We use the same notation as other parts of our paper, which are discussed in \Cref{prelim}.
Our setting is same as the bandit setting in \Cref{prelim:bandit}.

\paragraph{Reward rate and value.}
A stochastic policy $\pi$ has reward rate
\begin{equation*}
\rho(\pi) \;:=\; \frac{\E_p\big[\sum_a \pi(a \mid x)\, \bar r(x,a)\big]}{\E_p\big[\sum_a \pi(a \mid x)\, \bar\delta(x,a)\big]},
\qquad
\rho^* \;:=\; \max_{u \in \gU} \rho(u),
\end{equation*}
where the maximum runs over the deterministic policies $\gU = \{u : \gX \to \gA\}$ of \Cref{bandit:obj}.
For $\rho \in \R$ we define the relative values
\begin{equation*}
q_\rho(x,a) \;:=\; \bar r(x,a) - \rho\, \bar\delta(x,a),
\qquad
M_\rho(x) \;:=\; \max_a q_\rho(x,a),
\qquad
F(\rho) \;:=\; \E_p\big[M_\rho(x)\big],
\end{equation*}
with $q^* := q_{\rho^*}$ and $a^*(x) := \argmax_a q^*(x,a)$.

\paragraph{Constants.}
Since $r_{\min}\geq0$, every policy's reward rate is in $[0,\rho_{\max}]$, where
\[\rho_{\max} := r_{\max} / \delta_{\min},\]
we define the contraction constant as
\[\chi := 1 - \delta_{\min} / \delta_{\max} \in [0,1),\]
the \textit{time spread}\footnote{We assume the expected time for every $(x,a)$ pair cannot be the same since it would reduce the problem into reward maximization in contextual bandit, thus in the \gls{smdp} setting we have $\delta_{\mathrm{spr}} \neq 0$. We do not count it towards our list of assumptions since it is an intrinsic part of the problem formulation.} is
\[\delta_{\mathrm{spr}} := \max_x \big[\max_a \bar\delta(x,a) - \min_a \bar\delta(x,a)\big] \in (0, \delta_{\max} - \delta_{\min}],\]
for a rate $\rho \in \R$, a greedy policy is denoted as
\[u_\rho(x) \in \argmax_a q_\rho(x,a),\]
thus $u_{\rho}(x)$ is a greedy action given reward rate $\rho$ in context $x$, under any fixed tie-breaking.
Similarly, we define the \textit{greedy margin} at $\rho$ as
\[\Delta(\rho) \;:=\; \min_x \big[\max_a q_\rho(x,a) - \max_{a \ne u_\rho(x)} q_\rho(x,a)\big] ,\]
it is clear that $\Delta(\rho) \geq 0$ generally, and $\Delta(\rho) > 0$ if and only if the greedy action is unique at every context $x$.
With $\Delta(\rho^*)$, we have the gap at the optimal rate
\[\Delta^* \;:=\; \Delta(\rho^*) \;=\; \min_x \big[q^*(x,a^*(x)) - \max_{a \ne a^*(x)} q^*(x,a)\big].\]
We define the error parameter as $\bar\varepsilon := \Delta^* / (2 \delta_{\mathrm{spr}}),$ since $\delta_{\mathrm{spr}} < \delta_{\max}$, we have $\bar\varepsilon > \Delta^* / (2 \delta_{\max}).$

\paragraph{Algorithm.}
The algorithm is softmax policy with \gls{npg}, which keeps a conditional distribution on the actions per context.
See \Cref{appendix:exp:baselines} for more details.

\begin{assumption}[A1]\label{supp:theory:a1}
Reward and time are bounded with $r \in [r_{\min}, r_{\max}]$, $\delta \in [\delta_{\min}, \delta_{\max}]$, $r_{\min} \ge 0$, $\delta_{\min} > 0$.
\end{assumption}
\begin{assumption}[A2]\label{supp:theory:a2}
$\Delta(\rho^*) = \Delta^* > 0$, which is the positive-gap condition
\citep{auer_finite-time_2002} for the relative values.
Equivalently, for every $x$, the maximizer $a^*(x)$ of $q^*(x,\cdot)$ is unique and the
gap between the best and second-best relative value is at least $\Delta^*$ at every
context.
\end{assumption}
\begin{assumption}[A3]\label{supp:theory:a3}
Contexts are i.i.d. from $p(x)$, as described in \Cref{prelim:bandit} and same as in \cite{gyorgy_continuous_2007}.
\end{assumption}

\paragraph{Scope.}
Our analysis focuses on the algorithm's behavior in expectation, i.e., $\hat{Q}_{x,a}(\rho) = q_\rho(x,a)$.

\subsection{Supporting Lemmas}\label{supp:theory:lemma}
\begin{lemma}\label{supp:theory:l1}
For $\rho < \rho'$,
\begin{equation*}
    \delta_{\min} \;\le\; \big(F(\rho) - F(\rho')\big)/(\rho' - \rho) \;\le\; \delta_{\max};
\end{equation*}
in particular $F$ is strictly decreasing.
\end{lemma}
\begin{proof}
    Let $\rho < \rho'$, we have
    \begin{equation*}
    q_\rho(x,a) - q_{\rho'}(x,a) = \big(\bar r(x,a) - \rho\, \bar\delta(x,a)\big) - \big(\bar r(x,a) - \rho'\, \bar\delta(x,a)\big)
    = (\rho' - \rho)\, \bar\delta(x,a),
\end{equation*}
let $a_\rho = \argmax_{a} q_\rho(x,a)$, $a_{\rho'} = \argmax_{a} q_{\rho'}(x,a)$, then
\[M_{\rho}(x) = q_{\rho}(x,a_{\rho}) = q_{\rho'} (x,a_{\rho}) + (\rho' - \rho)\bar\delta(x,a_{\rho}) \leq M_{\rho'}(x) + (\rho'-\rho)\delta_{\max},\]
similarly
\[M_{\rho}(x) \geq q_{\rho}(x,a_{\rho'}) = q_{\rho'}(x,a_{\rho'}) + (\rho'-\rho)\bar\delta(x,a_{\rho'}) \geq M_{\rho'}(x) + (\rho'-\rho)\delta_{\min}.\]
Rearrange, and average over every context $x$ then gives
\[(\rho'-\rho)\delta_{\min} \leq F(\rho) - F(\rho') \leq (\rho' - \rho)\delta_{\max}.\]
\end{proof}
\begin{lemma}\label{supp:theory:l2}
    $F(\rho^*) = 0$, so $\rho^*$ is the unique root of $F$; $\rho(\pi) \le \rho^*$ for every stochastic policy $\pi$.
\end{lemma}
\begin{proof}
    $F(\rho^*) = 0$ can be established by Bellman optimality equation, and $F$ is strictly decreasing, so $\rho^*$ is unique. For arbitrary policy $\pi$, we have
    \[\big(\rho(\pi) - \rho^*\big)D(\pi) = \E_p \left[\sum_{a} \pi(a\mid x)q^*(x,a)\right] \leq \E_p \left[\sum_{a} \pi(a\mid x)M_{\rho^*}(x)\right] = F(\rho^*) = 0,\]
    where $D(\pi) = \E_p \left[\sum_{a} \pi(a\mid x)\bar\delta(x,a)\right]$, thus $\rho(\pi) \leq \rho^*$.
\end{proof}
\begin{lemma}\label{supp:theory:l3}
    Let $u_\rho(x) \in \argmax_a q_\rho(x,a)$.
    If $\rho \le \rho^*$, then $\rho \le \rho(u_\rho) \le \rho^*$ and
    \begin{equation*}
        \rho^* - \rho(u_\rho) \;\le\; \chi\, (\rho^* - \rho).
    \end{equation*}
\end{lemma}
\begin{proof}
    Let $u_{\rho}(x) = \argmax_{a} q_{\rho}(x,a)$ be the greedy policy. We have
    \[F(\rho) = \big(\rho(u_{\rho}) - \rho\big) D(u_{\rho}), \quad \rho(u_{\rho}) = \rho + \frac{F(\rho)}{D(u_{\rho})},\]
    by \Cref{supp:theory:l1}, and that $F(\rho^*) = 0$, we have
    \[\delta_{\min} (\rho^* - \rho) \leq F(\rho) - F(\rho^*) = F(\rho) \leq \delta_{\max}(\rho^* - \rho),\]
    since $\rho\le\rho^*$, $F(\rho)\ge \delta_{\min}(\rho^* - \rho) \ge 0$, thus $\rho(u_{\rho}) \geq \rho$, and by \Cref{supp:theory:l2}, $\rho \leq \rho(u_{\rho}) \leq \rho^*$.
    For the contraction, we have
    \begin{align*}
    \rho^* - \rho(u_\rho)
    &= (\rho^* - \rho) - \frac{F(\rho)}{D(u_\rho)} \\
    &\le (\rho^* - \rho) - \frac{F(\rho)}{\delta_{\max}} \\
    &\le (\rho^* - \rho) - \frac{\delta_{\min}\,(\rho^* - \rho)}{\delta_{\max}} \\
    &= \Big(1 - \frac{\delta_{\min}}{\delta_{\max}}\Big)(\rho^* - \rho)
    = \chi\,(\rho^* - \rho).
\end{align*}
\end{proof}
\begin{lemma}\label{supp:theory:l4}
    For all $\rho, \rho' \in \R$, every $x$, and every $a \ne u_{\rho'}(x)$,
    \begin{equation*}
        q_\rho\big(x, u_{\rho'}(x)\big) - q_\rho(x,a) \ge \Delta(\rho') - |\rho - \rho'|\, \delta_{\mathrm{spr}} .
    \end{equation*}
    Consequently $\Delta(\rho) \ge \Delta(\rho') - |\rho - \rho'|\,\delta_{\mathrm{spr}}$ and hence, by symmetry, $|\Delta(\rho) - \Delta(\rho')| \le |\rho - \rho'|\,\delta_{\mathrm{spr}}$: the greedy margin is $\delta_{\mathrm{spr}}$-Lipschitz in the rate.
    If moreover $|\rho - \rho'|\,\delta_{\mathrm{spr}} < \Delta(\rho')$, then $u_\rho = u_{\rho'}$ point-wise, under any tie-breaking.
    Setting $\rho' = \rho^*$ gives,
    \begin{equation*}
        \Delta(\rho) \ge \Delta^* - |\rho - \rho^*|\, \delta_{\mathrm{spr}},
        \qquad\text{and}\qquad
        u_\rho = a^* \ \text{ whenever } \ |\rho - \rho^*|\, \delta_{\mathrm{spr}} < \Delta^*,
    \end{equation*}
    in particular $|\rho - \rho^*| \le \bar\varepsilon$ gives $u_\rho = a^*$ and $\Delta(\rho) \ge \Delta^*/2$.
\end{lemma}
\begin{proof}
    Compare the relative values at two reward-rates $\rho,\rho'$ similar to \Cref{supp:theory:l1} gives
    \[q_{\rho}(x,a) = q_{\rho'}(x,a) + (\rho'-\rho)\bar\delta(x,a), \quad\forall x \in \gX, a\in \gA,\]
    denote $u:= u_{\rho'}(x)$ as the greedy policy, and let $a\neq u$, and substitute the equation above gives
    \[q_{\rho}(x,u) - q_{\rho}(x,a) = \underbrace{\big[q_{\rho'}(x,u) - q_{\rho'}(x,a)\big]}_{(1)} + \underbrace{(\rho'-\rho) \big[\bar\delta(x,u) - \bar\delta(x,a)\big]}_{(2)}.\]
    By definition (see \Cref{supp:theory:prelim}), $(1) \geq \Delta(\rho')$. For $(2)$, both $\bar\delta$ terms lie within interval $[\min_b\bar\delta(x,b),\max_{b}\bar\delta(x,b)]$, whose length is at most $\delta_{\mathrm{spr}}$ (\Cref{supp:theory:prelim}), this gives
    \[\big|(\rho' - \rho)\big[\bar\delta(x, u) - \bar\delta(x,a)\big]\big|
    \le |\rho - \rho'| \, \big[\max_b \bar\delta(x,b) - \min_b \bar\delta(x,b)\big]
    \le |\rho - \rho'|\, \delta_{\mathrm{spr}},\]
    which implies
    \[q_\rho(x, u) - q_\rho(x,a) \ge \Delta(\rho') - |\rho - \rho'|\, \delta_{\mathrm{spr}}.\]
    If the right hand side is positive, then it implies that $u$ is an unique maximizer, thus $u_{\rho'}(x) = u_{\rho}(x)$, then taking the minimum over $x$ gives
    \[\Delta(\rho) \ge \Delta(\rho') - |\rho - \rho'|\, \delta_{\mathrm{spr}}, \]
    If the right hand side is non-positive, then the above inequality still holds since $\Delta(\rho) \ge 0$ by construction (\Cref{supp:theory:prelim}).
    By symmetry, we also have that $\Delta(\rho') \ge \Delta(\rho) - |\rho - \rho'|\, \delta_{\mathrm{spr}}$, which implies
    \[|\Delta(\rho) - \Delta(\rho')| \le |\rho - \rho'|\,\delta_{\mathrm{spr}}.\]
    This completes the proof for $\delta_{\mathrm{spr}}$-Lipschitz statement.
    If $|\rho-\rho'|\delta_{\mathrm{spr}} < \Delta(\rho')$, then $q_\rho(x,u) - q_\rho(x,a) > 0$, and since $u:=u_{\rho'}(x)$, and by uniqueness of the greedy action assumption (\Cref{supp:theory:a2}), we have $u_{\rho'}(x) = u_{\rho}(x)$ for every $x$, i.e., the greedy policy for the two rates are the same.

    For the last statement, substitute $\rho'=\rho^*$ gives the property
    \[\Delta(\rho) \ge \Delta^* - |\rho - \rho^*|\, \delta_{\mathrm{spr}},
        \qquad\text{and}\qquad
        u_\rho = a^* \ \text{ whenever } \ |\rho - \rho^*|\, \delta_{\mathrm{spr}} < \Delta^*,\]
    and with $\bar\varepsilon = \Delta^* / (2\delta_{\mathrm{spr}})$, if $|\rho - \rho^*| \le \bar\varepsilon$, then $|\rho-\rho^*| \delta_{\mathrm{spr}} \le \Delta^* / 2 < \Delta(\rho^*) = \Delta^*$, which implies $u_{\rho} = a^*$ and $\Delta(\rho) \geq \Delta^* / 2$.
\end{proof}
\Cref{supp:theory:l1,supp:theory:l2,supp:theory:l3} are statements on the Dinkelbach's algorithm for fractional programming similar to those in \cite{jagannathan_properties_1966,dinkelbach_nonlinear_1967, schaible_fractional_1976}.

\subsection{Main Results}
\begin{theorem}\label{supp:theory:contraction}
Let $\gamma \in [0,1]$ and $\rho_1 \in [0, \rho_{\max}]$.
For every $k \ge 1$, let $u_{\rho_k}(x) \in \argmax_a q_{\rho_k}(x,a)$ be a greedy policy under any tie-breaking and $\rho(u_{\rho_k})$ is the reward rate of the greedy policy, let $\tilde\rho_{k+1}$ be any value with the sequence satisfying the condition
\begin{equation}
\frac{\rho(u_{\rho_k})}{1+\gamma} \le \tilde\rho_{k+1} \le (1+\gamma)\,\rho(u_{\rho_k}),
\qquad
\rho_{k+1} = \mathrm{clip}_{[0,\rho_{\max}]}\big(\tilde\rho_{k+1}\big),
\label{supp:theory:eq:gamma-recursion}
\end{equation}
where $\mathrm{clip}_{[0,\rho_{\max}]}(z) := \min\{\max\{z, 0\},\, \rho_{\max}\}$.
Then, for every $k \ge 2$,
\begin{equation}
    \left(1 - \frac{R\gamma}{1+\gamma}\right) \rho^* - \left(\frac{\chi}{1+\gamma}\right)^{k-2} \rho_{\max} \le \rho_k \le (1+\gamma)\rho^*,
    \label{supp:theory:eq:thm1:bound}
\end{equation}
where $R = \delta_{\max} / \delta_{\min} \ge 1$. If $\gamma = 0$, then the sequence increases monotonically $\rho_2 \le \rho_3 \le \cdots \le \rho^*$ and the iteration reaches $\rho_k = \rho^*$ after finite steps.
\end{theorem}
\begin{proof}
    Throughout the proof, we use the notation
    \[N(\pi) := \E_p\Big[\sum_a \pi(a \mid x)\, \bar r(x,a)\Big],
    \qquad
    D(\pi) := \E_p\Big[\sum_a \pi(a \mid x)\, \bar\delta(x,a)\Big],\]
    and $\rho(\pi) = N(\pi) / D(\pi)$ by definition, and
    \begin{align*}
    \E_p\Big[\sum_a \pi(a \mid x)\, q_\rho(x,a)\Big]
    &= N(\pi) \;-\; \rho\, D(\pi) \\
    &= \rho(\pi)\, D(\pi) \;-\; \rho\, D(\pi)
    = \big(\rho(\pi) - \rho\big)\, D(\pi) .
\end{align*}
By \Cref{supp:theory:l1}, we have
\[0\le \rho(u_{\rho_k}) \le \rho^* \le \rho_{\max}.\]
We first note that the $\mathrm{clip}$ operation survives the upper and lower bounds.
For upper bounds with a nonnegative right hand side: if $z \le B$ and $B \ge 0$, we have
\[\mathrm{clip}_{[0,\rho_{\max}]}(z) \le \max\{z,0\} \le B ,\]
and for lower bounds with a right-hand side at most $\rho_{\max}$: if $z \geq b$ with $b \le \rho_{\max}$, then
\[\mathrm{clip}_{[0,\rho_{\max}]}(z)
    \ge \min\{z,\, \rho_{\max}\} \ge \min\{b,\, \rho_{\max}\} = b .\]
Since the bounds survive, for notation simplicity, we use $\rho_{k+1}$ and $\tilde\rho_{k+1}$ interchangeably.

For $k\geq 1$, using the condition on the sequence in \Cref{supp:theory:eq:gamma-recursion}, we have
\[\tilde\rho_{k+1} \le (1+\gamma)\, \rho(u_{\rho_k}) \;\le\; (1+\gamma)\, \rho^* \implies \rho_{k+1} \le (1+\gamma)\rho^*,\]
and this completes the upper bound in \Cref{supp:theory:eq:thm1:bound}. To prove the lower bound in \Cref{supp:theory:eq:thm1:bound}, we define
\[e_k := \rho^* - \rho_k,\quad B:= \frac{R\gamma \rho^*}{1+\gamma}.\]
To complete the rest of the proof, we enumerate each of the two cases when $k \ge 2$: (i) \textit{the below case} with $\rho_k \le \rho^*$. (ii) \textit{the above case} with $\rho^* < \rho_k \le (1+\gamma)\rho^*$.
For the below case with $\rho_k \le \rho^*$, using \Cref{supp:theory:l3} gives
\[\rho(u_{\rho_k})\ge \rho^* - \chi(\rho^* - \rho_k) = \rho^* - \chi e_k,\]
using the condition in \Cref{supp:theory:eq:gamma-recursion}, we have $\rho_{k+1} \ge \rho(u_{\rho_k}) / (1+\gamma)$, which gives
\[\rho^* - \rho_{k+1} \le \rho^* - \frac{\rho^* - \chi e_k}{1+\gamma} = \frac{(1+\gamma)\rho^* - \rho^* + \chi e_k}{1+\gamma} = \frac{\chi e_k + \gamma \rho^*}{1+\gamma},\]
thus we established the fact that $e_{k+1} \leq \frac{\chi e_k + \gamma \rho^*}{1+\gamma}$ for the below case.

For the above case, $\rho^* < \rho_k \le (1+\gamma)\rho^*$ implies $0 < \rho_k - \rho^* \le \gamma \rho^*$.
By \Cref{supp:theory:l1} and the fact that $\rho^* < \rho_k$, we have
\begin{equation}
    -\delta_{\max}(\rho_k - \rho^*) \le F(\rho_k) \le -\delta_{\min} (\rho_k - \rho^*) \le 0.
    \label{supp:theory:main:eq1}
\end{equation}
Combine together with the fact that $F(\rho_k) / D(u_{\rho_k}) \ge F(\rho_k) / \delta_{\min}$ gives
\begin{equation}
    \frac{F(\rho_k)}{D(u_{\rho_k})} \geq -\frac{\delta_{\max}}{\delta_{\min}} (\rho_k - \rho^*),
    \label{supp:theory:main:thm1:identity}
\end{equation}
which implies
\begin{align*}
    \rho(u_{\rho_k}) &= \rho_k + \frac{F(\rho_k)}{D(u_{\rho_k})} \ge \rho_k - \frac{\delta_{\max}}{\delta_{\min}}(\rho_k - \rho^*)\\
    &= \rho^* + (\rho_k - \rho^*) - \frac{\delta_{\max}}{\delta_{\min}}(\rho_k - \rho^*) \\
    &= \rho^* - \left(\frac{\delta_{\max}}{\delta_{\min}} - 1 \right)  (\rho_k - \rho^*) \ge \rho^* - \left(R - 1 \right)\gamma \rho^*,
\end{align*}
where the first equality comes from the proof in \Cref{supp:theory:l3}.
Use the equation above with the fact $\rho_{k+1} \ge \rho(u_{\rho_k}) / (1+\gamma)$, we have
\begin{equation}
    \begin{aligned}
    \rho^* - \rho_{k+1} &\le \rho^* - \frac{\rho(u_{\rho_k})}{1+\gamma} \\
    &\le \rho^* - \frac{\rho^* - (R-1)\gamma \rho^*}{1+\gamma}\\
    &= \rho^* \frac{(1+\gamma) - 1 + (R-1) \gamma}{ 1+\gamma} = \frac{R\gamma\rho^*}{1+\gamma} = B.
\end{aligned}
\label{supp:theory:main:thm1:eq2}
\end{equation}
Hence we established the fact that $e_{k+1} \leq B$ for the above case. The final proof is to show by induction that the following is true
\begin{equation}
    e_k \le \tilde\chi^{k-2}\rho_{\max} + B, \qquad k\ge 2,
\end{equation}
where $\tilde\chi := \chi / (1+\gamma) \in [0,1)$. For the base case $k=2$, since $\rho_2 \ge 0$, we have $\rho^* - \rho_2 \le \rho^* \le \rho_{\max}$, which gives $e_2 \le \rho_{\max} \le \rho_{\max} + B$.
For the induction step, we assume the claim is true for some $k\ge 2$ and enumerate each of the two cases $\rho_k \le \rho^*$ and $\rho_k > \rho^*$. If $\rho_k > \rho^*$, then by \Cref{supp:theory:main:thm1:eq2}, we have $e_{k+1} \leq B \leq \tilde\chi^{k-1}\rho_{\max} + B$.
If $\rho_k \le \rho^*$, then
\begin{equation}
\begin{aligned}
    e_{k+1} &\le \frac{\chi e_k + \gamma \rho^*}{1+\gamma} \le \frac{\chi \big(\tilde\chi^{k-2}\rho_{\max} + B\big) + \gamma \rho^*}{1+\gamma}\\
    &= \tilde\chi^{k-1} \rho_{\max} + \frac{\chi B + \gamma \rho^*}{1+\gamma},
    \label{supp:theory:main:thm1:eq3}
\end{aligned}
\end{equation}
\begin{equation}
    \frac{\chi B + \gamma \rho^*}{1+\gamma} = \frac{\gamma\rho^*}{(1+\gamma)^2}(\chi R + 1 + \gamma) \leq \frac{\gamma\rho^*}{(1+\gamma)^2} (R + R\gamma) = B,
    \label{supp:theory:main:thm1:eq4}
\end{equation}
where the last inequality comes from the identity that $\chi R + 1 = R$, so $\chi R + 1 + \gamma \leq R + R\gamma$. Combine \Cref{supp:theory:main:thm1:eq3} with \Cref{supp:theory:main:thm1:eq4}, we have
\[e_{k+1} \le \tilde\chi^{k-1} \rho_{\max} + \frac{\chi B + \gamma \rho^*}{1+\gamma} \le \tilde\chi^{k-1} \rho_{\max} + B,\]
rearranging the terms gives the lower bound and completes the bound
\[\rho^*\left(1 - \frac{R\gamma}{1+\gamma}\right)  - \left(\frac{\chi}{1+\gamma}\right)^{k-2}\rho_{\max} \leq \rho_k.\]
We now show that when $\gamma = 0$, then the iteration reaches $\rho^*$ in finite steps. If $\gamma = 0$, we have that every rate estimate equals to the current greedy policy's rate $\rho_{k+1} = \rho(u_{\rho_k})$, by the identity in \Cref{supp:theory:main:thm1:identity} and the upper bound $\rho_k \le \rho^*$ in \Cref{supp:theory:eq:thm1:bound}, we have
\[\rho_{k+1} - \rho_k = \frac{F(\rho_k)}{D(u_{\rho_k})} \ge \frac{\delta_{\min}}{\delta_{\max}} (\rho^* - \rho_k) \ge 0,\]
hence the sequence increases monotonically. 
To prove that it reaches $\rho^*$ within finite steps, suppose for sake of contradiction that $\{\rho_k\}$ does not reach $\rho^*$ in finite steps, then we would have $\rho_k < \rho^*$ for every iteration $k = 2,\cdots,M+2$, which implies
\[\rho_{k+1} - \rho_k = \frac{F(\rho_k)}{D(u_{\rho_k})} \ge\frac{\delta_{\min}}{\delta_{\max}} (\rho^* - \rho_k) > 0,\]
which means the sequence $\{\rho_k\}$ is strictly monotonic.
Since greedy policy is a deterministic policy, and there are $M = |\gA|^{|\gX|}$ possible deterministic policies, this means the possible number of reward rate is $| \{\rho(u) : u\in \gU\}| \leq M$.
Then we would have $\rho_2 < \rho_3 < \cdots < \rho_{M+2}$, which are $M+1$ distinct rates, a contradiction.
\end{proof}
\begin{proposition}\label{supp:theory:prop:npg}
Let $\bar\rho \in \R$ be any rate with $\Delta(\bar\rho) > 0$, write $u := u_{\bar\rho}$ for the greedy policy at $\bar\rho$, which is then unique at every context, and run the algorithm from $\theta^{(0)}$ with $\pi_0 := \pi_{\theta^{(0)}}(u(x) \mid x)$.
For every $x$ and $k \ge 0$,
\begin{equation}
1 - \pi_{\theta^{(k)}}(u(x) \mid x) \;\le\; \frac{1-\pi_0}{\pi_0}\,\exp \big({-\eta\, \Delta(\bar\rho)\, k}\big),
\label{supp:theory:eq:npg-tail}
\end{equation}
in particular $1 - \pi_{\theta^{(k)}}(u(x) \mid x) \le (|\gA|-1) \exp\big({-\eta\, \Delta(\bar\rho)\, k}\big)$ when $\pi_0 = 1/|\gA|$ (policy is initially uniform).
\end{proposition}
\begin{proof}
We use same abbreviations as before, where the greedy policy is denoted as $u = u_{\hat\rho}(x)$.
By definition of the margin, we have
\begin{equation*}
    q_{\bar\rho}(x, u) - q_{\bar\rho}(x,a)
    \ge q_{\bar\rho}(x,u) - \max_{b \ne u} q_{\bar\rho}(x,b)
    \ge \Delta(\bar\rho) > 0,
    \qquad \text{for every } a \ne u,
\end{equation*}
the algorithm updates the logit for each step by
\begin{equation*}
    \theta^{(n+1)}_x = \theta^{(n)}_x + \eta\Big(q_{\bar\rho}(x,\cdot) - \big(\vpi^{(n)\top} q_{\bar\rho}(x,\cdot)\big)\, \vone\Big),
\end{equation*}
and the logit between actions $u$ and $a$ is
\begin{align*}
    \theta^{(n+1)}_{x,u} - \theta^{(n+1)}_{x,a}
    &= \Big(\theta^{(n)}_{x,u} + \eta\, q_{\bar\rho}(x,u) - \eta\, \vpi^{(n)\top} q_{\bar\rho}(x,\cdot)\Big) \\
    &\qquad\ - \Big(\theta^{(n)}_{x,a} + \eta\, q_{\bar\rho}(x,a) - \eta\, \vpi^{(n)\top} q_{\bar\rho}(x,\cdot)\Big) \\
    &= \big(\theta^{(n)}_{x,u} - \theta^{(n)}_{x,a}\big) \;+\; \eta\big(q_{\bar\rho}(x,u) - q_{\bar\rho}(x,a)\big),
\end{align*}
write $d^{(n)} := \theta^{(n)}_{x,u} - \theta^{(n)}_{x,a}$ for the logit gap and $c \;:=\; \eta\big(q_{\bar\rho}(x,u) - q_{\bar\rho}(x,a)\big)$ for the update gap, the parameter update rule thus becomes a recursion of the form $d^{(n+1)} = d^{(n)} + c$.
Writing as a telescoping sum over $n = 0,\ldots,k-1$ gives
\begin{equation*}
    d^{(k)} - d^{(0)} = \sum_{n=0}^{k-1} \big(d^{(n+1)} - d^{(n)}\big) = \sum_{n=0}^{k-1} c \;=\; k\, c ,
\end{equation*}
thus we have
\begin{equation*}
    \theta^{(k)}_{x,u} - \theta^{(k)}_{x,a}
    = \theta^{(0)}_{x,u} - \theta^{(0)}_{x,a} + \eta\, k\big(q_{\bar\rho}(x,u) - q_{\bar\rho}(x,a)\big),
\end{equation*}
since our distribution is parameterized through softmax, we have
\begin{equation*}
    \frac{\pi_{\theta}(a \mid x)}{\pi_{\theta}(u \mid x)} = \frac{\exp(\theta_{x,a}) / \sum_{b} \exp(\theta_{x,b})}{\exp(\theta_{x,u}) / \sum_{b} \exp(\theta_{x,b})} = \exp(\theta_{x,a} - \theta_{x,u}),
\end{equation*}
fix $a \neq u$, and apply the two equations together, we have
\begin{equation*}
    \begin{aligned}
        \frac{\pi_{\theta^{(k)}}(a \mid x)}{\pi_{\theta^{(k)}}(u \mid x)}
        &= \exp\big(\theta^{(k)}_{x,a} - \theta^{(k)}_{x,u}\big) \\
        &= \exp\Big(\theta^{(0)}_{x,a} - \theta^{(0)}_{x,u} - \eta\, k\big(q_{\bar\rho}(x,u) - q_{\bar\rho}(x,a)\big)\Big) \\
        &= \frac{\pi_{\theta^{(0)}}(a \mid x)}{\pi_{\theta^{(0)}}(u \mid x)}\; \exp\Big({-\eta\, k \big(q_{\bar\rho}(x,u) - q_{\bar\rho}(x,a)\big)}\Big),
    \end{aligned}
\end{equation*}
observe that
\begin{equation*}
    -\eta\, k \big(q_{\bar\rho}(x,u) - q_{\bar\rho}(x,a)\big) \;\le\; -\eta\, k\, \Delta(\bar\rho) ,
\end{equation*}
thus
\begin{equation*}
    \frac{\pi_{\theta^{(k)}}(a \mid x)}{\pi_{\theta^{(k)}}(u \mid x)}
    \;\le\; \frac{\pi_{\theta^{(0)}}(a \mid x)}{\pi_{\theta^{(0)}}(u \mid x)}\; \exp({-\eta \Delta(\bar\rho) k}),
\end{equation*}
taking the sum over all $a \neq u$ gives the following for the left hand side
\begin{equation*}
    \sum_{a \ne u} \frac{\pi_{\theta^{(k)}}(a \mid x)}{\pi_{\theta^{(k)}}(u \mid x)}
    = \frac{\sum_{a \ne u} \pi_{\theta^{(k)}}(a \mid x)}{\pi_{\theta^{(k)}}(u \mid x)}
    = \frac{1 - \pi_{\theta^{(k)}}(u \mid x)}{\pi_{\theta^{(k)}}(u \mid x)} ,
\end{equation*}
and the following for the right hand side
\begin{equation*}
    \sum_{a \ne u} \frac{\pi_{\theta^{(0)}}(a \mid x)}{\pi_{\theta^{(0)}}(u \mid x)}\; \exp({-\eta \Delta(\bar\rho) k})
    = \frac{1 - \pi_0}{\pi_0}\; \exp({-\eta \Delta(\bar\rho) k}),
\end{equation*}
together we have
\begin{equation*}
    \frac{1 - \pi_{\theta^{(k)}}(u \mid x)}{\pi_{\theta^{(k)}}(u \mid x)}
    \;\le\; \frac{1-\pi_0}{\pi_0}\; \exp({-\eta \Delta(\bar\rho) k}),
\end{equation*}
multiplying both sides by $\pi_{\theta^{(k)}}(u \mid x)$ gives
\begin{align*}
    1 - \pi_{\theta^{(k)}}(u \mid x)
    &= \pi_{\theta^{(k)}}(u \mid x)\; \frac{1 - \pi_{\theta^{(k)}}(u \mid x)}{\pi_{\theta^{(k)}}(u \mid x)} \\
    &\le \pi_{\theta^{(k)}}(u \mid x)\; \frac{1-\pi_0}{\pi_0}\; \exp({-\eta \Delta(\bar\rho) k}) \\
    &\le \frac{1-\pi_0}{\pi_0}\; \exp({-\eta \Delta(\bar\rho) k}),
\end{align*}
and hence completes the proof.
\end{proof}
\begin{corollary}\label{supp:theory:cor:npg-inband}
Let $\bar\rho \in \R$ be any rate with smallness condition $|\bar\rho - \rho^*| \le \bar\varepsilon$, $\bar\varepsilon = \Delta^* / (2 \delta_{\mathrm{spr}})$.
Then $u_{\bar\rho} = a^*$, for every $x$, $a \ne a^*(x)$, $k \ge 0$
\begin{equation}
    1 - \pi_{\theta^{(k)}}(a^*(x) \mid x) \le \frac{1-\pi_0}{\pi_0}\, \exp\big({-\eta\, \Delta^* k/2}\big),
    \label{supp:theory:eq:npg-tail-inband}
\end{equation}
in particular $1 - \pi_{\theta^{(k)}}(a^*(x) \mid x) \le (|\gA|-1) \exp({-\eta \Delta^* k/2})$ at the zero initialization.
\end{corollary}
\begin{proof}
    Apply \Cref{supp:theory:l4} with \Cref{supp:theory:prop:npg}.
\end{proof}
\begin{theorem}\label{supp:theory:prop:why-work}
Let $\rho_1,\rho_2,\ldots$ be a sequence of reward rates.
Denote $\theta^{(k)}$ as the parameter for the softmax logit after $k$ steps from the zero initialization and $\bar\rho_k = \sum_{n \le k} \rho_n / k$, then for every $x$, $k \ge 1$, the policy $\pi^{(k)}$ is a Boltzmann policy of $q_{\bar\rho_{k}}(x,a)$ with inverse temperature $\eta k$, and if $\rho_{k+1} = \rho(u_{\bar\rho_{k}})$, $\rho_1 \in [0,\rho^*]$, then
\begin{equation}
    \rho^* - \bar\rho_k \;\le\; (\rho^* - \bar\rho_1)\, \exp(1-\chi)\,(k+1)^{-(1-\chi)} ,
    \label{supp:theory:prop:why-work:eq:km-rate}
\end{equation}
thus $\bar\rho_k$ converges to $\rho^*$ at $O(k^{-(1-\chi)})$.
Moreover, for every
\begin{equation}
    k \;\ge\; k_0 \;:=\; \Big\lceil \big(\exp(1-\chi)\,(\rho^* - \bar\rho_1) / \bar\varepsilon\big)^{1/(1-\chi)} \Big\rceil ,
    \label{supp:theory:prop:why-work:eq:entry}
\end{equation}
the greedy policy at the running average is optimal, $u_{\bar\rho_k} = a^*$ point-wise, and for every $x$
\begin{equation}
    1 - \pi^{(k)}(a^*(x) \mid x) \;\le\; (|\gA| - 1)\, \exp\big({-\eta\, \Delta^* k / 2}\big).
    \label{supp:theory:prop:why-work:eq:tail}
\end{equation}
\end{theorem}
\begin{proof}
Observe that the relative value $q_{\rho_{n}}(x,a)$ for action $a$ and context $x$ at step $n$ can be expressed as
\begin{align*}
    \sum_{n \le k} q_{\rho_n}(x,a)
    &= \sum_{n \le k} \big(\bar r(x,a) - \rho_n\, \bar\delta(x,a)\big)
    \;=\; k\, \bar r(x,a) \;-\; \Big(\sum_{n \le k} \rho_n\Big) \bar\delta(x,a) \\
    &= k\, \big(\bar r(x,a) - \bar\rho_k\, \bar\delta(x,a)\big)
    = k\, q_{\bar\rho_k}(x,a) ,
\end{align*}
By the parameter update rule for \gls{npg} with a softmax policy $\theta'_x = \theta_x + \eta (q - \langle\vpi, q\rangle \vone)$, fix the context, and expand the terms through the recursion gives
\begin{align*}
    \theta^{(k)}_x
    &= \eta \sum_{n \le k} q_{\rho_n}(x,\cdot) \;-\; \eta\Bigg(\sum_{n \le k} s_n\Bigg)\vone \qquad s_n := \big\langle \vpi^{(n-1)}, q_{\rho_n}(x,\cdot)\big\rangle \\
    &= \eta\, k\, q_{\bar\rho_k}(x,\cdot) \;+\; c_k(x)\, \vone
    \qquad c_k(x) := -\eta \sum_{n \le k} s_n ,
\end{align*}
when the logit $\theta_x^{(k)}$ is feed into the softmax function, since softmax is invariant to addition of constants, the $c_k(x)$ term is ignored, thus the policy becomes the Boltzmann policy of $\eta\, k\, q_{\bar\rho_k}(x,a)$ and completes the first part of the proof.
For the second part, first define the estimated rate as  $k\bar\rho_k = \sum_{n \le k} \rho_n$, then
\begin{equation*}
    \bar\rho_{k+1}
    = \frac{1}{k+1}\sum_{n \le k+1}\rho_n
    = \frac{k\,\bar\rho_k + \rho_{k+1}}{k+1}
    = \frac{k}{k+1}\,\bar\rho_k \;+\; \frac{1}{k+1}\,\rho_{k+1} ,
\end{equation*}
which is 
\begin{equation}
    \bar\rho_{k+1} = (1-\alpha_k)\bar\rho_k + \alpha_k\,\rho_{k+1}, \qquad \alpha_k = 1/(k+1).
    \label{supp:theory:prop:why-work:eq:convex}
\end{equation}
Since the algorithm only accepts samples selected by $\argmax_{a} \theta_{x,a}$, we assume that the rate estimate equals to the greedy policy's rate under $\bar\rho_k$, i.e., $\rho_{k+1} = \rho(u_{\bar\rho_{k}})$.
Since $\bar\rho_k \le \rho^*$ at $k=1$ by assumption, \Cref{supp:theory:l3} gives that $\rho(u_{\bar\rho_{k}}) \ge \bar\rho_k$ and $\rho^* - \rho(u_{\bar\rho_{k}}) \le \chi (\rho^* - \bar\rho_k)$.
$\bar\rho_{k+1}$ is a convex combination of $\bar\rho_k$ and $\rho(u_{\bar\rho_{k}})$ by \Cref{supp:theory:prop:why-work:eq:convex} and we have $\bar\rho_k \le \bar\rho_{k+1} \le \rho^*$.
Subtracting \Cref{supp:theory:prop:why-work:eq:convex} from $\rho^*$ gives
\begin{align*}
    \rho^* - \bar\rho_{k+1} &= \rho^* + \alpha_k \rho^* - \alpha_k \rho^* - (1-\alpha_k)\bar\rho_k - \alpha_k\rho_{k+1} \\
    &= (1-\alpha_k)(\rho^* - \bar\rho_k) + \alpha_k(\rho^* - \rho_{k+1})\\
    &\le (1-\alpha_k)(\rho^* - \bar\rho_k) + \alpha_k \chi (\rho^* - \bar\rho_k) = \big(1 - \alpha_k(1-\chi)\big)(\rho^* - \bar\rho_k).
\end{align*}
Denote $e_k = \rho^* - \bar\rho_k$, $c := 1-\chi = \delta_{\min}/\delta_{\max} \in (0,1]$, and expand the recursion gives
\begin{align*}
    e_k &\le e_1 \prod_{m=2}^{k}\big(1 - c/m\big) \le e_1 \exp\left({-c\sum_{m=2}^{k} 1/m}\right) \\
    &\le e_1 \exp\big(c - c\log(k+1)\big) = e_1 \exp(c)\,(k+1)^{-c},
\end{align*}
where we used the fact $1-z\le \exp(-z)$ and $\sum_{m\le k} 1/m \ge \log(k+1)$.

Since $\rho_{k+1} = \rho(u_{\bar\rho_{k}}) \le \rho^*$, $\bar\rho_{k+1} \le (1-\alpha_k)\bar\rho_k + \alpha_k\rho^* < \rho^*$ whenever $\bar\rho_k < \rho^*$, so by induction $\rho_1 < \rho^*$ keeps the strict inequality $\bar\rho_k < \rho^*$ to every $k > 1$.

For the last part, the second part gives $\bar\rho_k \le \bar\rho_{k+1} \le \rho^*$, so the error $\rho^* - \bar\rho_k$ is non-increasing in $k$ and it suffices to find the first index at which it falls below $\bar\varepsilon$.
By \Cref{supp:theory:prop:why-work:eq:km-rate}, $\rho^* - \bar\rho_k \le \bar\varepsilon$ holds when 
\begin{equation*}
    (k+1)^{1-\chi} \ge \exp(1-\chi)\,(\rho^* - \bar\rho_1) / \bar\varepsilon,
\end{equation*}
thus we need
\begin{equation*}
    k \ge \big(\exp(1-\chi)\,(\rho^* - \bar\rho_1) / \bar\varepsilon\big)^{1 /(1 - \chi)} - 1,
\end{equation*}
thus we can set $k_0$ to be the constant stated in \Cref{supp:theory:prop:why-work:eq:entry}, and it then holds at every subsequent step $k \geq k_0$.
When $k \geq k_0$, since $|\bar\rho_k - \rho^*| \leq \bar\varepsilon$, we can apply \Cref{supp:theory:l4}, which gives $u_{\bar\rho_k} = u_{\rho^*} = a^*$ and $ \Delta(\bar\rho_k) \ge \Delta^*/2$.
Similar to the steps in the proofs for \Cref{supp:theory:prop:npg,supp:theory:cor:npg-inband}, the probability ratio is
\begin{align*}
    \frac{\pi^{(k)}(a \mid x)}{\pi^{(k)}(a^*(x) \mid x)}
    &= \exp\Big({-\eta\, k \big(q_{\bar\rho_k}(x,a^*(x)) - q_{\bar\rho_k}(x,a)\big)}\Big) \\
    &\le \exp\big({-\eta\, k\, \Delta(\bar\rho_k)}\big) \\
    &\le \exp\big({-\eta\, \Delta^*\, k/2}\big),
\end{align*}
sum over all actions $a \ne a^*(x)$ gives
\begin{equation*}
    \frac{1 - \pi^{(k)}(a^*(x) \mid x)}{\pi^{(k)}(a^*(x) \mid x)} \le  (|\gA| - 1) \exp\big({-\eta\, \Delta^*\, k/2}\big) ,
\end{equation*}
which gives \Cref{supp:theory:prop:why-work:eq:tail} after multiplying both sides by $\pi^{(k)}(a^*(x) \mid x) \le 1$ and completes the proof.
\end{proof}

%% file: appendix/implementation.tex
\subsection{Bandit Experiments\label{appendix:exp:bandit}}
This section gives the full specification of the bandit experiments in \Cref{method:bandit}.

\subsubsection{Bandit Baselines\label{appendix:exp:baselines}}
\paragraph{\gls{c-ucb} with theoretical constants.}
The algorithm of \citet{gyorgy_continuous_2007}, implemented faithfully for all the constants.

\paragraph{\gls{c-ucb} with tuned constants.}
Identical in every respect, except that the two scale constants of \Cref{eq:b-width,eq:c-width} are replaced by free parameters $c_1$ and $c_2$ selected via grid search.

\paragraph{\gls{c-ucb} with approximation.}
\gls{c-ucb} involves enumeration over all policies for reward-rate estimation.
This is infeasible when we have many arms and contexts.
As an approximation, we draw 2048 policies uniformly random from policies constructed from visited arms and contexts, and enumerate on these 2048 policies.

\paragraph{\Gls{npg}.}
The reward-rate learner keeps one independent softmax head per context: a tabular parameter matrix $\theta \in \R^{|\gX| \times (K+1)}$ with one free logit per $(x,a)$ pair, initialized to zero so that $\pi_{\theta}(a \mid x) = 1/(K+1)$ is uniform, and
\begin{equation}
    \pi_{\theta}(a \mid x) \;=\; \frac{\exp (\theta_{x,a})}{\sum_{b} \exp (\theta_{x,b})}.
    \label{bandit:softmax}
\end{equation}
At each round the learner observes $x$, samples $a \sim \pi_{\theta}(\cdot \mid x)$, and receives $(r, \delta)$, and its regret is incurred by its own action choices.
The update rule is the natural gradient of the relative reward
\begin{equation}
    \theta_x \gets \theta_x + \eta\, \vz, \quad \nabla_{\theta_x}\, J(\theta) = \mF(\theta_x) \vz,
\end{equation}
where the Fisher information matrix $\mF$ is
\begin{equation}
    \mF(\theta_x) = \diag(\vpi) - \vpi\vpi^\top,
\end{equation}
which simplifies to adding the advantage to the logits directly.
Writing $\vpi = \pi_\theta(\cdot \mid x)$, $\vs$ for the vector with entries $s_a = \E[\, r - \rho^*\delta \mid x, a \,]$, $J = \vpi^{\!\top}\vs$ and $\va = \vs - \big(\vpi^{\!\top}\vs\big)\vone$ be the advantages
\begin{align}
    \frac{\partial \pi_a}{\partial \theta_c}
      &= \pi_a\big(\indic\{a = c\} - \pi_c\big)
      \quad\Longrightarrow\quad
      \nabla_{\theta_x} J = \vpi \odot \va,
    \label{bandit:npg-vanilla}
\end{align}
solving for $\mF \vz = \vpi \odot \va$ gives $\vz = \va$ as one of the solutions, thus the update rule is
\[\theta_x \leftarrow \theta_x + \eta \, A_{\theta}(x,\cdot).\]
The advantage is needed for every arm while a bandit learner observes only the arm it played, so we keep a per-cell buffer of the raw sums $S^r_{x,a}$ and $S^\delta_{x,a}$ with counts $n_{x,a}$, and use them to estimate the expected relative reward under current estimated rate $\hat{\rho}_t$,
\begin{equation}
    \widehat{Q}_{x,a}
    \;=\;
    \begin{cases}
        \big(S^r_{x,a} - \hat{\rho}_t\, S^\delta_{x,a}\big) / n_{x,a}, & n_{x,a} > 0,\\[2pt]
        0, & n_{x,a} = 0,
    \end{cases}
    \qquad
    \theta_x \;\gets\; \theta_x
      + \eta \big( \widehat{Q}_x - (\vpi^{\!\top}\widehat{Q}_x) \vone \big).
    \label{bandit:update}
\end{equation}
The rate $\hat{\rho}_t$ is supplied by the online \gls{niw} estimator.

\paragraph{Vanilla policy gradient.}
Everything is kept same as \gls{npg}, except the gradient update is
\[\theta_x \leftarrow \theta_x + \eta \nabla_{\theta_x} J(\theta),\]
which gives the update rule
\[\theta_x \leftarrow \theta_x + \eta \;\vpi \odot \big(\hat{Q}_x - (\vpi^\top \hat{Q}_x)\vone\big).\]

\paragraph{The \gls{niw} rate estimator.}
The estimate $\hat{\rho}_t$ is fitted online on pairs of accumulated reward and accumulated time, mirroring the way the rate is formed from per-slot accumulators in MLE-Bench and NanoGPT (\Cref{method:mle}).
It is refitted every $8$ steps, and $\hat{\rho}_t$ is the 95th percentile among posterior-predictive samples (see \Cref{appendix:exp:rr-pred}).
The accumulators are \emph{gated}, in that a sample enters an accumulator only when the arm played was the current $\argmax_{a} \, \theta_{a,x}$, so that the estimate describes the greedy policy.
They are also \emph{age-bounded}, in that a slot restarts once it is older than $A$ steps, with $A$ tuned with hyperparameter search.
We note that this estimator does not guarantee exact estimation of the greedy policy's reward rate.
Rather, it serves as a test case to see how our algorithm behaves when the greedy reward rate has approximation error, which is the case in empirical settings.

\subsubsection{Experiment Setup}
\paragraph{Setup.}
At each decision epoch the learner selects an action $a$ and observes a reward $r \ge 0$ together with a time $\delta > 0$. The pair is drawn jointly, so the reward and the time it costs may be dependent. Performance is measured against the optimal reward rate $\rho^*$, and we report
the realized regret of \Cref{bandit:regret} at step $t$.
Each context and action pair consists of a mean reward and time parameter, forming a table of size $(K, |\gX|)$.
The setup for these mean parameters are described below.
To ensure we do not introduce bias in terms of ordering of the arms and contexts, we perform a permutation on $K$ arm indices and $|\gX|$ context indices for every trial.

\paragraph{Arm design.}
The $K+1$ actions available in a context are the \textsc{skip} action and two structurally different kinds of arm, which are specified in opposite directions.
The four designed arms fix the pair $(\E[r(a)], \E[\delta(a)])$ and let the rate follow: 
\begin{itemize}
    \item \textsc{best} arm $(0.50, 1.0)$ at rate $0.500$.
    \item \textsc{decoy} $(0.57, 1.2)$ at rate $0.475$, carrying higher raw reward at a $5\%$ relative rate deficit.
    \item \textsc{trap} $(0.90, 3.0)$ at rate $0.300$, the highest raw reward in the context at the worst rate.
    \item \textsc{fast} arm $(0.21, 0.6)$ at rate $0.350$ that is cheap but weak.
\end{itemize}
The \textsc{decoy} and the \textsc{trap} both exist to punish an agent that maximizes reward rather than reward per unit of time.
The remaining $K-4$ filler arms fix the rate and derive the reward and time from it
\begin{equation}
\begin{gathered}
  \zeta_j = \beta_x \Big(0.42 - \tfrac{j}{K-5}\big(0.42 - 0.18\big)\Big),
    \quad
    \E[\delta(4+j)] = d_j, \\
    d_j \sim \mathrm{Unif}(0.70,\, 2.00),
    \quad
    \E[r(4+j)] = \zeta_j\, d_j,
    \label{bandit:filler}
\end{gathered}
\end{equation}
for $j = 0, \dots, K-5$, where $\beta_x$ is the reward scale of the context (see below) and each $d_j$ is drawn independently.
This means the reward rate is monotone but reward and time are unordered, so neither quantity on its own identifies a good arm.
No filler arm is ever optimal, because the sequence's maximum is at $0.42\,\beta_x$, strictly below the \textsc{best} arm at $0.50\,\beta_x$.
Filler arms therefore widen the action set at controlled rate spacing rather than competing for the optimum, which makes $K$ an exploration-difficulty axis.

\paragraph{Context Design.}
A context $x \in \gX$ is drawn uniformly at each epoch and adds a \textsc{skip} action with $(0, 0.35)$.
A context is summarized by the reward scale $\beta_x$, which multiplies every reward and leaves time untouched; since the rate is reward divided by time, $\beta_x$ scales the rate as well.
Four contexts are designed: 
\begin{itemize}
    \item \textsc{marginal} ($\beta_x = 1.00$) reproduces the arm design above and serves as the baseline.
    \item \textsc{rich} ($\beta_x = 1.50$) holds the highest multiplier.
    \item \textsc{poor} ($\beta_x = 0.32$) so that \textsc{skip} can be considered.
    \item \textsc{coupling} is not a rescaling at all. Instead, it replaces the four designed arms with $(1.80, 2.00)$, $(1.00, 1.00)$, $(1.60, 3.0)$ and $(0.20, 0.6)$, a slow-but-rich arm at rate $0.90$ against a fast-but-lean arm at rate $1.00$.
\end{itemize}
Maximizing the rate within \textsc{coupling} takes the fast arm while the optimal policy takes the slow one, which is why the problem does not decompose into $|\gX|$ independent bandits.
This property holds as long as $\rho^* < 0.8$, which is the case for all of our simulations.
Any context beyond these four is filler, where $\beta_x$'s are uniformly binned midpoints between $[0.42, 1.35]$, the range between \textsc{poor} and \textsc{rich}.

\paragraph{Environment families.}
Four families vary the dependence between reward and time while holding the marginal means fixed.
Let $\mu^r_a = \E[r_a]$ and $\mu^\delta_a = \E[\delta_a]$ be the designed means of arm $a$.
We fix $\sigma_r = 0.20$, a time noise fraction $\varsigma = 0.15$, a log-scale spread $\sigma_{\log} = 0.5$, and a time floor $\delta_{\min} = 0.5$.
Every family draws a pair of standard normals $z_1, \xi \sim \gN(0,1)$ and couples them through a Gaussian copula
\begin{equation}
    z_2 \;=\; \varrho\, z_1 \;+\; \sqrt{1-\varrho^{2}}\;\xi,
    \qquad \mathrm{corr}(z_1, z_2) = \varrho,
    \label{bandit:copula}
\end{equation}
so that $\varrho$ alone controls the reward--time dependence while both marginals are unchanged.

\textbf{E1 (independent), $\varrho = 0$.} Reward and time are rectified normals drawn independently,
\begin{equation}
    r_a = \max\big(\mu^r_a + \sigma_r z_1, 0\big), \qquad
    \delta_a = \max\big(\mu^\delta_a + \varsigma\,\mu^\delta_a z_2,\;
                        \delta_{\min}\big).
    \label{bandit:e1}
\end{equation}

\textbf{E2 (positively coupled), $\varrho = +0.8$} and \textbf{E3 (negatively coupled), $\varrho = -0.8$.} Identical marginals to \Cref{bandit:e1}, with $z_2$ generated by \Cref{bandit:copula}.
Under E2 a high-reward draw tends to also be slow, so the empirical rate $r/\delta$ is compressed and a learner that ranks arms by raw reward is punished less; under E3 high reward tends to arrive fast, which inflates the apparent spread in rates.

\textbf{E4 (lognormal), $\varrho = +0.5$ on the log scale.} Both coordinates are lognormal with heavy right tails. Note that we take the log on the mean parameter so that the expectation $\E[r_a], \E[\delta_a]$ still equals to $\mu^r_a, \mu^\delta_a$:
\begin{equation}
    r_a = \exp\!\Big(\log \mu^r_a - \tfrac{\sigma_{\log}^{2}}{2}
          + \sigma_{\log} z_1\Big), \qquad
    \delta_a = \max\!\Big(
          \exp\!\big(\log \mu^\delta_a - \tfrac{\sigma_{\log}^{2}}{2}
          + \sigma_{\log} z_2\big),\; \delta_{\min}\Big).
    \label{bandit:e4}
\end{equation}
The rectification and the floor $\delta_{\min}$ make the realized marginals differ slightly from $(\mu^r_a, \mu^\delta_a)$, so we do not rely on the design values when computing $\rho^*$ (See \Cref{bandit:fixedpoint}).
Every reported mean is instead computed in closed form, using
\begin{equation}
    \E[\max(X, c)] = c\,\Phi(a) + m\big(1 - \Phi(a)\big) + s\,\phi(a),
    \quad a = \tfrac{c - m}{s},
\end{equation}
\begin{equation}
    \E[\max(Y, c)] = c\,\Phi(b) + m\big(1 - \Phi(b - \sigma_{\log})\big),
    \quad b = \tfrac{\log c - \mu}{\sigma_{\log}},
    \label{bandit:closedform}
\end{equation}
for a rectified normal $X \sim \gN(m, s^2)$ and a floored lognormal $Y$ with log-scale $\sigma_{\log}$.
$\phi, \Phi$ are PDF and CDF of Gaussian, respectively.
The means are therefore deterministic and do not require Monte Carlo approximation, and $\rho^*$ follows from \Cref{bandit:fixedpoint}.

\paragraph{Ground truth.}
The general average-reward \gls{smdp} optimality equation reads
\begin{equation}
  V^*(x)=\max_{a\in\mathcal{A}}\Big\{\E[r(x,a)]-\rho^{*}\,\E[\delta(x,a)]
        +\sum_{x'\in\mathcal{X}}p(x'\mid x,a)\,V^*(x')\Big\}.
  \label{eq:smdp}
\end{equation}
In the bandit case the next context is drawn independently of the current context and of the
action taken, so $p(x'\mid x,a)=p(x')$ and the continuation term collapses to a constant
$c:=\sum_{x'}p(x')V^*(x')$ that depends on neither $x$ nor $a$. Writing
\begin{equation}
  q^{*}(x,a):=\E[r(x,a)]-\rho^{*}\,\E[\delta(x,a)],
  \label{eq:relval}
\end{equation}
\Cref{eq:smdp} becomes
\begin{equation}
  V^*(x)=\max_{a\in\mathcal{A}}q^{*}(x,a)+c .
  \label{eq:separated}
\end{equation}
Averaging against $p(x)$ on both sides and cancelling $c$ gives the condition that
$\rho^{*}$ is the unique root of
\begin{equation}
    F(\rho) \;=\; \sum_{x} p(x)\, \max_{a\in\mathcal{A}}\,
    \big( \E[r(x,a)] - \rho\, \E[\delta(x,a)] \big) \;=\; 0.
    \label{bandit:fixedpoint}
\end{equation}
$F$ is a maximum of affine functions of $\rho$, hence convex, piecewise linear and strictly decreasing, so the root is unique and we find $\rho^{*}$ by bisection. Empirically, the residual is small $|F(\rho^{*})| \sim 10^{-17}$.

\paragraph{Evaluation.}
Every instance is run to horizon $T = 30{,}000$ over all four families, in two studies whose arm and context counts differ: 
\begin{itemize}
    \item The small-scale study at $|\gX| = 4$ with $K \in \{5, 10, 20, 30\}$.
    \item The larger-scale study at $|\gX| \in \{8, 16, 32, 64\}$ with $K \in \{70, 100, 130, 160\}$. 
\end{itemize}
Hyperparameters are selected by grid search on tuning seeds disjoint from the evaluation seeds, taking the \emph{median} final regret over the tuning seeds, and the selected configuration is then run on $100$ held-out evaluation seeds. 
Every step that took place counts as part of the regret calculation, including the burn-in phase (for instance, to have a nonzero value of counts for \gls{c-ucb}).
Confidence bands are pointwise percentile bootstraps of the mean over trials.

\subsection{Reward-rate Predictor\label{appendix:exp:rr-pred}}
We introduce the reward-rate estimator we use for \gls{mle}.
For $f_{\rho}$, we use a \gls{niw} prior with joint Gaussian likelihood \citep{murphy_machine_2012} fitted from historical samples observed in buffer $\gB$.
We model the pair as $(r_{\mathrm{acc}}, \log \delta_{\mathrm{acc}})$, where $r_{\mathrm{acc}}$ and $\delta_{\mathrm{acc}}$ are the accumulated reward and time, and we take the logarithm because we need time to be positive. 
Writing $z_i = (r_{\mathrm{acc},i},\, \log \delta_{\mathrm{acc},i})$ for the $i^\text{th}$ of the $n$ samples available at step $t$, where $n$ is the batch size.
The likelihood and its conjugate prior are
\begin{equation}
    z_i \mid \mu, \Sigma \sim \gN(\mu, \Sigma),
    \qquad
    (\mu, \Sigma) \sim \mathrm{NIW}\big(\mu_{t-1},\, \lambda\kappa_{t-1},\, \lambda\nu_{t-1},\, \lambda\Psi_{t-1}\big),
    \label{method:niw}
\end{equation}
where $\mu_{t-1}$ is the prior mean of $\mu$, $\Psi_{t-1}$ is the prior scatter matrix of $\Sigma$, and $\kappa_{t-1}$ and $\nu_{t-1}$ are pseudo-counts.
The prior at step $t$ is the previous step's posterior discounted by a forgetting factor $\lambda \in (0,1]$, which shrinks those pseudo-counts so that old evidence decays geometrically, and the posterior predictive is a Student-$t$ with $\nu_t - 1$ degrees of freedom, where $\nu_t = \lambda\nu_{t-1} + n$ is the posterior count after this step's $n$ samples are absorbed.
Reward rate is estimated by sampling $M$ batches of $n$ draws each from this predictive. Writing the $m^\text{th}$ batch as $\{(x_{m,j},\, y_{m,j})\}_{j=1}^{n}$, each batch yields one rate, and $\hat\rho_{t+1}$ is the 95th percentile of them,
\begin{equation}
    \rho_m = \frac{\sum_{j=1}^{n} x_{m,j}}{\sum_{j=1}^{n} \exp(y_{m,j})},
    \qquad
    \hat\rho_{t+1} = \mathrm{Quantile}_{0.95}\big(\{\rho_m\}_{m=1}^{M}\big) .
    \label{method:rate}
\end{equation}

\subsection{MLE-Bench \label{appendix:exp:mle}}
\subsubsection{Environment, Reward and Self-improvement\label{appendix:exp:mle:env}}
\paragraph{Execution.}
Every action $a$ is a code block generated by the policy (Qwen3.5-4B).
The code block is executed in its own sandbox subprocess environment with its own working directory, under a memory budget of 16GB and a compute budget of 2 CPUs.
The subprocess is killed when it reaches timeout $\delta_{\max}$.
We set $\delta_{\max} = 600$ seconds, except for nomad2018, pizza, tabular-2021 and tabular-2022, where the dataset is large enough that $600$ seconds admits almost no honest attempt, so we use $\delta_{\max} = 1200$ seconds for these tasks.
Time $\delta$ is floored at \texttt{clip\_time} $= 0.1$ seconds before it is used anywhere, so that no sample can claim an unboundedly large rate.
We train with a batch size of 128.
Thus, assuming full parallelization at each step, we require 16$\times$128 = 2048GB of memory with 256 CPUs.
This amount of resource is large in academic setting, so we use a first-in-first-out queue that processes the 128 samples by 16 samples running parallel at a time through Ray.

\paragraph{Reward.}
A run is a \emph{valid submission} when the script terminates before $\delta_{\max}$, generates a submission file, and that file is accepted by the MLE-Bench grader.
The grader returns the competition's own metric.
For the $15$ tasks whose metric is higher-is-better and bounded, we use performance metric (such as AUC) directly as the reward.
For the remaining $7$ tasks, where the metric is lower-is-better and has no bounded value (such as log loss), we use the position of the submission on the MLE-Bench leaderboard normalized as percentiles as reward.
When there is a tie on the leaderboard, we use the mid-point position.
\begin{equation}
    r_{t+1} \;=\;
    \frac{1}{N}\sum_{i=1}^{N}
    \Big( \indic\{\ell_i \prec g\} \;+\; \tfrac12\, \indic\{\ell_i = g\} \Big)
    \;\in\; [0, 1],
    \label{mle:reward}
\end{equation}
where $\ell_1, \ldots, \ell_N$ are the leaderboard entries and $\prec$ orders them in the competition's own direction.
This restores the higher-is-better direction and bounds the reward in $[0,1]$.
The 7 tasks using leaderboard score are exactly the seven marked with $(\downarrow)$ in \Cref{tab:mlebench_meanatT}.

\paragraph{Failures.}
If a script fails to obtain a valid submission, we define it as a failure.
We use a penalty of $-10$ for invalid submissions.
Writing $s$ for a valid submission's score of \Cref{mle:reward}, the reward functions for vanilla and \gls{ours} are
\begin{equation}
    \tilde r^{\,\mathrm{RPG}}_{t+1}
    =
    \begin{cases}
        s - \hat\rho_{t+1}\,\delta_{t+1}, & \text{valid},\\[2pt]
        -10 - \hat\rho_{t+1}\,\delta_{\max}, & \text{invalid},
    \end{cases}
    \qquad
    \tilde r^{\,\mathrm{vanilla}}_{t+1}
    =
    \begin{cases}
        s, & \text{valid},\\[2pt]
        -10, & \text{invalid}.
    \end{cases}
    \label{mle:arms}
\end{equation}
For \gls{ours}, we charge $\hat\rho_{t+1}\,\delta_{\max}$ for invalid submissions so that every invalid sample's reward is strict below the valid one.

\paragraph{Self-improvement buffer.}
The transition stored in $\gB$ is the extracted code, the grader's message, the score $s$ of the submission, and the measured time $\delta_{t+1}$.
The buffer holds a capacity of 10 times the batch size, i.e.\ $1{,}280$ solutions.
At every step, it is ranked by the relative reward
\begin{equation}
    \mathrm{key}(a, r, \delta) \;=\; r \;-\; \hat\rho_{t+1}\,\delta,
    \label{mle:key}
\end{equation}
with $\hat\rho_{t+1}$ obtained from the current step's reward rate estimate.
The vanilla arm uses the identical buffer with $\hat\rho = 0$.
The top $128$ solutions ranked by the key fill the $128$ self-improvement prompts of the next step, cycling from the top when the buffer holds $<128$ samples.
A valid sample whose score strictly exceeds every retained solution's receives an additional self-improvement bonus of $0.5$, which shapes the advantage of both arms but enters neither \Cref{mle:key} nor the rate estimate.

\subsubsection{Prompts\label{appendix:exp:mle:prompts}}
\paragraph{Base prompt.}
The initial state $s_0$ is a chat-formatted prompt carrying the data path, the task description, the metric, the available packages, the required submission format, and a schema snippet of the training file.
It also carries the single efficiency instruction referred to in \Cref{method:mle}.
See \Cref{mlebench:prompt} for an example.
\begin{figure}[htbp]
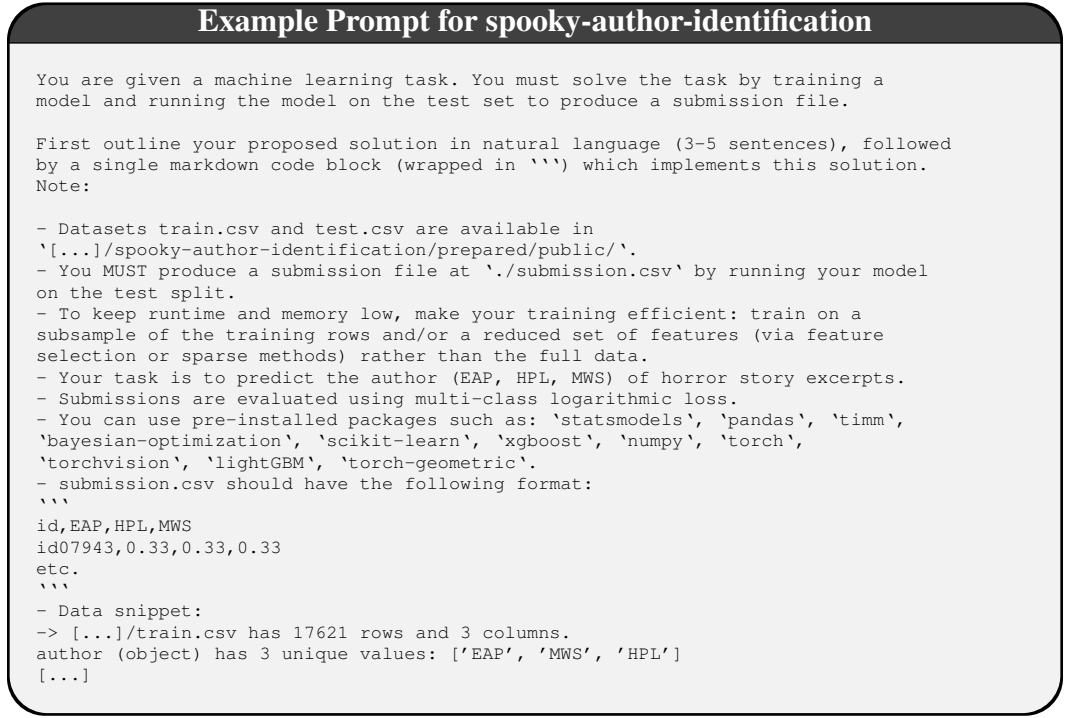

\begin{promptbox}{Example Prompt for spooky-author-identification}
\begin{promptlisting}
You are given a machine learning task. You must solve the task by training a
model and running the model on the test set to produce a submission file.

First outline your proposed solution in natural language (3-5 sentences), followed
by a single markdown code block (wrapped in ```) which implements this solution.
Note:

- Datasets train.csv and test.csv are available in
`[...]/spooky-author-identification/prepared/public/`.
- You MUST produce a submission file at `./submission.csv` by running your model
on the test split.
- To keep runtime and memory low, make your training efficient: train on a
subsample of the training rows and/or a reduced set of features (via feature
selection or sparse methods) rather than the full data.
- Your task is to predict the author (EAP, HPL, MWS) of horror story excerpts.
- Submissions are evaluated using multi-class logarithmic loss.
- You can use pre-installed packages such as: `statsmodels`, `pandas`, `timm`,
`bayesian-optimization`, `scikit-learn`, `xgboost`, `numpy`, `torch`,
`torchvision`, `lightGBM`, `torch-geometric`.
- submission.csv should have the following format:
```
id,EAP,HPL,MWS
id07943,0.33,0.33,0.33
etc.
```
- Data snippet:
-> [...]/train.csv has 17621 rows and 3 columns.
author (object) has 3 unique values: ['EAP', 'MWS', 'HPL']
[...]
\end{promptlisting}
\end{promptbox}
\caption{\textbf{Example prompt on MLE-Bench}.}
\label{mlebench:prompt}
\end{figure}

\paragraph{Self-improvement prompt.}
The next state $s_{t+1}$ replaces the opening instruction with the self-improvement phrase $i_0$ and appends the selected solution and its grader message.
Every other line is byte-identical to the base prompt, so the two states differ only in what the self-improvement loop contributes.
The differing lines are shown in \Cref{mlebench:self-improve-prompt}.
\texttt{previous\_plan\_code} is filled with the executed code of the selected transition, which is the action $a$.
\texttt{previous\_plan\_error} contains the grader's own message.
For a valid submission, \texttt{previous\_plan\_error} reports the achieved time $\delta$ and performance score $s$, which is how the agent is told the number it is being asked to beat.
For an invalid one, it reports the failure and the traceback.
We set the prompt length to $2560$ tokens and response length to $2500$ tokens.

\begin{figure}[htbp]
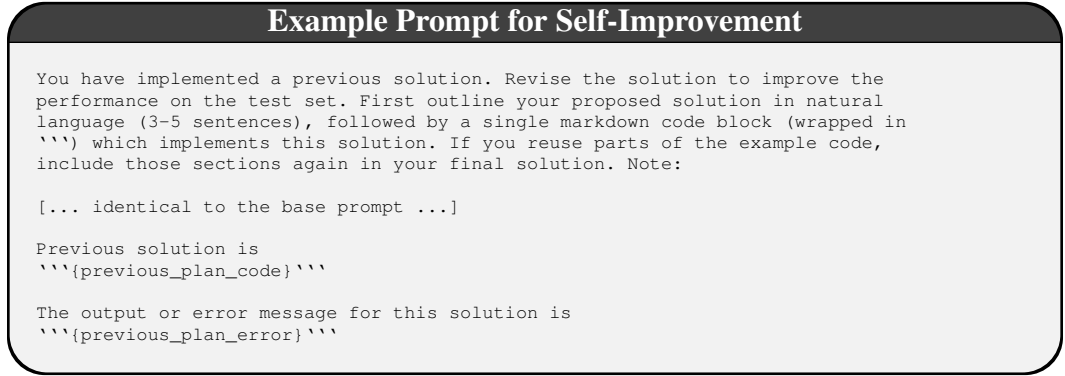

\begin{promptbox}{Example Prompt for Self-Improvement}
\begin{promptlisting}
You have implemented a previous solution. Revise the solution to improve the
performance on the test set. First outline your proposed solution in natural
language (3-5 sentences), followed by a single markdown code block (wrapped in
```) which implements this solution. If you reuse parts of the example code,
include those sections again in your final solution. Note:

[... identical to the base prompt ...]

Previous solution is
```{previous_plan_code}```

The output or error message for this solution is
```{previous_plan_error}```
\end{promptlisting}
\end{promptbox}
\caption{\textbf{Example prompt for self-improvement}. \texttt{previous\_plan\_code} is filled with the executed code of the selected transition, which is the action $a$. \texttt{previous\_plan\_error} contains the grader's own message. }
\label{mlebench:self-improve-prompt}
\end{figure}

\subsubsection{Hyperparameters\label{appendix:exp:mle:hparams}}
See \Cref{tab:mle:hparams}.

\begin{table}[htbp]
\centering
\caption{\textbf{MLE-Bench configuration}. The upper block is shared by \gls{ours} and vanilla. The lower block is where they differ.}
\label{tab:mle:hparams}
{\small
\begin{tabular}{llr}
\toprule
Group & Setting & Value \\
\midrule
Model   & actor and reference & Qwen3.5-4B \\
        & critic & Qwen3.5-4B, token-classification head \\
        & rollout engine & vLLM, \texttt{bfloat16} \\
        & attention, checkpointing & \texttt{sdpa}, on \\
\addlinespace
\gls{ppo} & advantage estimator & GAE, $\gamma = 1.0$, $\lambda = 1.0$ \\
        & actor / critic learning rate & $1\times10^{-6}$ / $1\times10^{-5}$ \\
        & train batch $=$ mini-batch & $128$ \\
        & micro-batch (actor, critic, reference) & $4$ \\
        & clip ratio, value clip & $0.2$, $0.5$ \\
        & gradient clip & $1.0$ \\
        & entropy coefficient & $0.001$ \\
        & KL coefficient to the reference & $0.001$, fixed \\
\addlinespace
Rollout & temperature, top-$p$, top-$k$ & $0.7$, $0.8$, $20$ \\
        & presence penalty & $1.5$ \\
        & samples per prompt & $1$ \\
        & prompt / response length & $2560$ / $2500$ tokens \\
\addlinespace
Environment & execution timeout $\delta_{\max}$ & $600$\,s ($1200$\,s on four tasks) \\
        & time floor & $0.1$\,s \\
        & memory / CPU per sample & $16$\,GB / $2$ \\
        & invalid-submission penalty & $-10$ \\
\addlinespace
Buffer  & capacity & $10 \times 128 = 1{,}280$ solutions \\
        & selection width & top $128$ \\
        & self-improvement bonus & $0.5$ \\
\addlinespace
Compute & per run & $2\times$ A100 80GB, 32 CPU, 450\,GB \\
        & repetitions & $3$ per task per arm \\
        & tasks & $22$ MLE-Bench Lite competitions \\
\midrule
\gls{ours} only & rate estimator & online \gls{niw}, \Cref{method:niw} \\
        & forgetting factor $\lambda$ & $0.3$ \\
        & predictive batches $M$ & $1{,}000$ \\
        & percentile to be selected for $M$ batches & 95 \\
        & time charged to an invalid sample & $\delta_{\max}$ \\
\bottomrule
\end{tabular}}
\end{table}

\FloatBarrier
\subsection{NanoGPT\label{appendix:exp:nanogpt}}
\subsubsection{Environment and Grader\label{appendix:exp:nanogpt:env}}
\paragraph{The task.}
An action is a modification to a single self-contained \texttt{train\_gpt.py}, launched as \texttt{torchrun -{}-nproc\_per\_node=8 train\_gpt.py} on eight A100-SXM4-80GB GPUs and graded against the rules of the modded-NanoGPT speedrun \citep{jordan_modded_2024}.
The FineWeb-10B dataset is fixed, and the validation loss must be the mean cross-entropy over exactly $10{,}485{,}760$ validation tokens.
A validation pass must be printed every $125$ training steps and again at the last step.
A script that alters the validation tokens or the logging interval is scored as a failure, so the agent cannot design hacks.
The elapsed time $\delta$ is the whole process's wall clock, measured from launch to exit.
The sandbox process is killed at timeout of $\delta_{\max} = 1500$ seconds.

\paragraph{Outcomes.}
A crash, a parse failure, or a timeout with no usable validation line is an invalid attempt.
A run that is killed at timeout but has printed real validation lines is kept as \emph{valid} and scored on its reached validation loss and charged its measured elapsed time.
This is the only case in which the timeout and a valid outcome coexist.
Times are floored at $0.1$ seconds.

\paragraph{Seeds.}
The buffer is initialized with past leaderboard submissions to modded NanoGPT, re-executed on our own eight-A100 grader so that their recorded time is measured on the same clock the agent is scored on.
These past leaderboard submissions serve as the roots where every self-improvement chain descends from.
Unlike in MLE-Bench, we do not have an initial state $s_0$ since the buffer is non-empty from the first step.

\subsubsection{Delta Proposals and the Implementer\label{appendix:exp:nanogpt:impl}}
\paragraph{Implementer.}
A \texttt{train\_gpt.py} used for self-improvement prompts is around $560$ lines.
Asking the policy (Qwen3.5-4B) to generate such scripts would use a lot of response tokens, slow down training, and the majority of the focus would be on correct code transcription.
To alleviate this issue, we ask the policy to propose a \textit{change specification} in three fixed sections: \textsc{diagnosis}, \textsc{change}, and \textsc{risk}.
For each section, it needs to point out the corresponding snippet in the parent code block to be improved upon.
The implementer is \texttt{GPT 5.6 Terra} at temperature $0$ with a $16{,}384$-token budget, and it is given the parent script, the change specification, and asked for creating a search-and-replace edit blocks.
The edit blocks then go through regular expression to replace the parent code with the new edits, which we call the output script as \textit{patch}.
A reply by the implementer is accepted only if every edit block matches the parent exactly and the validation logging lines remain untouched (to prevent model from hacking the validation logs).

\paragraph{Judge.}
A separate judge, the same \texttt{GPT 5.6 Terra} model, gives three votes at temperature $0$ to check whether the patch script violates the above constraints.
If any of the violations occur, it is scored as invalid.

\subsubsection{Reward\label{appendix:exp:nanogpt:reward}}
\paragraph{\gls{ours}.}
Writing $L$ for the achieved validation loss and $L^\star = 3.28$ for the target, an attempt's reward is the loss reduction below a fixed reference $L_{\mathrm{ref}} = 10$, plus the two threshold bonuses described in \Cref{method:mle},
\begin{equation}
    r_{t+1} \;=\;
    \begin{cases}
        \max\big(0,\, L_{\mathrm{ref}} - L\big) \;+\; B \;+\; w\,\min\big(L^\star - L,\; 0.03\big), & L \le L^\star,\\[3pt]
        \max\big(0,\, L_{\mathrm{ref}} - L\big), & L > L^\star,\\[3pt]
        -10, & \text{invalid},
    \end{cases}
    \label{nanogpt:reward}
\end{equation}
with $B = 5$ and $w = 100$.
The relative reward is same as \Cref{mle:arms}: a valid attempt pays $\hat\rho_{t+1}\,\delta_{t+1}$, and an invalid one is charged $\hat\rho_{t+1}\delta_{\max}$.
The reward rate is fitted on $\max(0, L_{\mathrm{ref}} - L)$ only, with the bonuses removed, so that a threshold bonus cannot inflate the price of a second.

\paragraph{Vanilla.}
The baseline optimizes the benchmark's own objective, the time to reach the target, and also similarly aims to reach further below the $3.28$ target loss, with an exchange rate of $w = 20{,}000$ seconds per unit of loss.
\begin{equation}
    r^{\,\mathrm{vanilla}}_{t+1} \;=\;
    \begin{cases}
        -\,\delta_{t+1} \;+\; w\,\min\big(L^\star - L,\; 0.03\big), & L \le L^\star,\\[3pt]
        -\,\delta_{\max} \;-\; w\,\min\big(L - L^\star,\; 0.03\big), & \text{valid}, \; L > L^\star,\\[3pt]
        -\,2\,\delta_{\max}, & \text{invalid}.
    \end{cases}
    \label{nanogpt:vanilla}
\end{equation}
The three cases do not overlap, so every crossing attempt still outranks every non-crossing one and every valid attempt outranks every invalid one.
The second case exists to encourage scripts that are near the target loss to cross the target.

\paragraph{Buffer.}
The buffer is same as \Cref{appendix:exp:mle:env} with the key of \Cref{mle:key} replaced by each arm's own objective, $r - \hat\rho\,\delta$ for \gls{ours} and $r^{\,\mathrm{vanilla}}$ for the baseline.
During selection from the buffer to form the current state $s_t$, half of samples are reserved for the best samples in the buffer ranked by the key (for instance, since we have a batch size of 64, 32 samples would be selected this way).
For the other half, it is ranked first by the number of changed lines and then by the key.
The second half serves to encourage for exploration on making more significant code changes.

\subsubsection{Prompts and an Example Action\label{appendix:exp:nanogpt:prompts}}
See \Cref{nanogpt:prompt,nanogpt:judge,nanogpt:action-response}.
\begin{promptbox}[breakable]{Example Prompt for NanoGPT}
\begin{promptlisting}
You are competing in the NanoGPT speedrun (modded-nanogpt, track_1_short): train a neural network to at most 3.28 mean cross-entropy loss on the FineWeb validation set, using 8x NVIDIA A100-SXM4-80GB GPUs, in the least possible total wall-clock time. The score is the TOTAL WALL-CLOCK SECONDS OF THE WHOLE SCRIPT PROCESS, measured by our harness from process launch to process exit. You are scored on TWO things at once and must improve BOTH: the total wall-clock seconds, AND how far the final validation loss goes below 3.28 (worth roughly 20 seconds of wall-clock per 0.001 of loss, credited down to 3.25). Reaching 3.28 is required but not sufficient -- do not stop optimising the model the moment the script crosses it. A change that trades one against the other is NOT progress and is a wasted attempt: not a speedup paid for with loss margin, and not a loss improvement paid for with time. Propose the change that moves both in the right direction together.

Below is an existing solution script and the result it achieved. Study it and propose ONE change that makes the whole process finish sooner AND drives the final validation loss further below 3.28. Both must improve; a proposal that helps one at the other's expense does not count.

You do NOT write code. You write a CHANGE SPECIFICATION, which a separate engineer will implement exactly as you describe it. The engineer implements ONLY what the CHANGE section says: if the CHANGE section is missing, vague, or names only a flag or a constant, that is what gets built. Your CHANGE section is your entire contribution.

Hardware facts you must design for: A100 is Ampere (sm_80). It has NO FP8 tensor cores, so any FP8 code path (float8 matmuls, FP8 lm-head, FP8 scaling) is unavailable and must be avoided or replaced; bf16 and TF32 tensor cores are the fast paths. HBM2e at roughly 2 TB/s, NVLink between GPUs. FlashAttention-3 is Hopper-only and unavailable; FlashAttention-2, PyTorch SDPA, FlexAttention, Triton and torch.compile are available.

Constraints the script must continue to satisfy:
- Launched with: torchrun --nproc_per_node=8 train_gpt.py. Single self-contained file; it must not read any sibling source files.
- Training data: data/fineweb10B/fineweb_train_*.bin relative to the working directory (equivalently os.environ.get("DATA_PATH", ".") + "/data/fineweb10B/"); validation: data/fineweb10B/fineweb_val_*.bin. Do not modify the train or validation token streams. Validation loss must be the mean cross-entropy over exactly 10485760 tokens of the validation set.
- Each .bin shard has a 256-int32 header (magic 20240520, version, token count) followed by uint16 GPT-2 tokens; shard files are 100M tokens each.
- Validation cadence is enforced: a validation pass exactly every 125 training steps (when step 
- Print these exact log line formats to stdout: after each training step, "step:<step>/<train_steps> train_time:<ms>ms step_avg:<ms>ms" (train_time as an integer number of milliseconds, step_avg with 2 decimals); at every validation, "step:<step>/<train_steps> val_loss:<loss> train_time:<ms>ms step_avg:<ms>ms" (val_loss with 4 decimals), for example: step:1390/1390 val_loss:3.2775 train_time:79532ms step_avg:57.22ms. The printed train_time is informational only; the score is the harness-measured wall-clock seconds of the whole process.
- torch._inductor.config.coordinate_descent_tuning is banned.
- No network access, no downloads.
- The whole process (startup + compilation + training + every validation + exit) is killed at a hard wall cap of 1500 seconds and scored as a failure; compilation alone can take several minutes, so leave margin.

Existing solution and its result:

{previous_plan_code}

Result of that solution: {previous_plan_error}

What counts as a change. The change must be a MECHANISM, not a setting. 

Not acceptable, and treated as a wasted attempt: changing torch.compile arguments (mode, dynamic, fullgraph, backend), DDP arguments, or any hyperparameter on its own (batch size, sequence length, iteration count, learning rates, warmup/cooldown lengths, momentum, Newton-Schulz iteration count, gradient-accumulation factor); deleting or moving a line; renaming; changing where the printed timer starts. 

Acceptable: something that changes what the program computes or how data moves, for example a Triton fused kernel (RMSNorm, rotary, cross-entropy, Newton-Schulz orthogonalization, lm-head); a different attention backend or block-mask construction for sm_80; CUDA graphs or regional/selective compilation that removes recompiles; bf16/TF32 precision changes replacing an FP8 path; DDP bucket restructuring, sharded reduce-scatter/all-gather or communication overlap; async pinned H2D loading or loader prefetch; architecture changes (attention structure, value embeddings, U-net skips, logit softcap, sliding/varlen document-aware attention, MLP variants, embedding sparsification, parameter tying); optimizer mechanics (Muon/NorMuon variants, per-parameter-group treatment, cautious weight decay).

Write your reply as exactly these three sections, in this order, and nothing else. Do not think out loud before the first section, do not restate the task or the script, and do not write out the whole file. Start immediately with the word DIAGNOSIS.

DIAGNOSIS
At most 120 words. What specifically is costing wall-clock time in the script above, and what evidence in the script or its result points at it. Be concrete about where.

CHANGE
At most 400 words. A numbered list of the exact edits: for each, the function, class, variable or line region to touch, what it should become, and exact values for any constant. New code must be described precisely enough that there is only one reasonable way to write it; short illustrative fragments are fine. Every edit must serve the ONE mechanism named in the first line of this section.

RISK
At most 60 words. What could break, and what the implementer should check to avoid it.

Your reply is cut off at a hard token limit. A reply that has not reached the CHANGE section by then is worthless, so write the CHANGE section before anything long.
\end{promptlisting}
\end{promptbox}
\captionof{figure}{\textbf{Example prompt for NanoGPT}.}
\label{nanogpt:prompt}

\begin{promptbox}[breakable]{Judge Prompt for NanoGPT}
\begin{promptlisting}
You are the maintainer of the modded-nanogpt speedrun leaderboard, reviewing
one submitted training script (train_gpt.py, launched with torchrun on 8x
NVIDIA A100-SXM4-80GB GPUs). Your ONLY job is to decide whether it violates
this rule:

  "Do not modify the train or validation data pipelines. You may change the
  batch size, sequence length, attention structure, model architecture,
  optimizer, precision, schedules, kernels, compilation, logging, timing --
  anything -- just do not change the underlying streams of tokens, do not
  change what the printed validation loss MEANS, and do not change WHEN
  validation runs."

The validation loss must remain: the mean cross-entropy of next-token
prediction over EXACTLY `val_tokens` (10,485,760) tokens of the validation
files (data/fineweb10B/fineweb_val_*.bin), read in their natural order,
correctly averaged over everything that was summed (all validation batches
AND all data-parallel ranks). The training stream must remain the tokens of
data/fineweb10B/fineweb_train_*.bin, unmodified (order, content, and shard
reading semantics; changing how MANY tokens are consumed per step via batch
geometry is allowed). Validation must run exactly every 125 training steps
(val_loss_every == 125, i.e. at step 
printing the standard val line each time.

VIOLATIONS include (non-exhaustive -- judge the semantics, not the syntax):
  * evaluating fewer validation tokens than val_tokens while still
    normalizing by the full count (e.g. the val loop no longer iterates
    over all val_steps, or val_steps is inflated relative to the batches
    actually evaluated);
  * removing or corrupting the cross-rank reduction of the validation loss
    (each rank sees 1/world_size of the tokens; the printed number must be
    the average over ranks);
  * dividing the validation loss by extra factors (world_size, an extra
    val_steps, a constant);
  * computing the "validation" loss on training data, synthetic tokens,
    a subset chosen by loss, or anything other than the validation files;
  * altering, filtering, reordering, or synthesizing the token streams;
  * printing a val_loss value that is not the number computed by the
    validation procedure above;
  * changing the validation cadence: val_loss_every different from 125,
    validating on a different schedule, skipping validations, validating
    only at the end, or adding / removing / moving the validation call
    (including gating it behind any condition other than the standard
    step 

NOT violations (do not flag these):
  * changed batch size, sequence length, gradient accumulation, number of
    training steps, early stopping of TRAINING when a target is reached;
  * architecture / attention / optimizer / LR schedule / precision /
    torch.compile changes, custom kernels;
  * harness bookkeeping: saving checkpoints to verify_ckpt/, build_model()/
    eval_forward() helpers;
  * changes to the timing bookkeeping (where the clock starts / stops, what
    train_time prints): the printed time is informational only, the score
    is the harness-measured wall-clock of the whole process -- but the val
    line must remain parseable in its standard format;
  * refactors of the data loader that preserve which tokens are delivered
    in which order.

Read the script carefully, especially the validation section and every line
between computing per-batch losses and printing `val_loss:`. Verify the
arithmetic: what is summed, over how many batches, divided by what, reduced
across ranks how. Then check the condition under which validation runs. If
the script is so mangled it cannot run, judge only the pipeline rule (a
crash is handled elsewhere).

Reply with ONLY a JSON object, no other text:
{"verdict": "intact" | "modified", "reason": "<one sentence naming the exact
line/mechanism if modified, or 'pipeline preserved' if intact>"}

The submitted script:

```python
{code}
```
\end{promptlisting}
\end{promptbox}
\captionof{figure}{\textbf{Judge prompt for NanoGPT}.}
\label{nanogpt:judge}

\begin{figure}[htbp]
\label{nanogpt:action-response}
\begin{promptbox}{Example Action for NanoGPT}
\begin{promptlisting}
DIAGNOSIS
The script incurs massive startup and per-step overhead due to repeated `torch.compile` recompilation (only done once at start, but the model structure is complex), inefficient data loading via standard file I/O without async overlap, and suboptimal memory management where pinned tensors are allocated fresh per shard. The A100 lacks FP8, forcing full bf16/TF32 paths which are slower than Hopper's FlashAttention-3 would be. The Muon optimizer adds significant compute cost per step for orthogonalization on gradients, slowing convergence and increasing wall-clock time per validation step. The current loader does not prefetch aggressively enough relative to the GPU bandwidth limits of 8x A100s, causing stalls during the initial training steps before warmup kicks in.

CHANGE
1. Replace the sequential `_load_data_shard` call with an asynchronous pre-fetcher that loads the next two shards while the current one is being processed by the CPU/GPU pipeline, using `threading.Thread` or `concurrent.futures` to hide I/O latency behind the first micro-batch computation. Specifically, modify `DistributedDataLoader.advance()` to submit the next shard load immediately upon receiving the current batch, ensuring the CUDA stream remains busy.
2. Implement a custom Triton kernel for the Newton-Schulz orthogonalization step within `zeropower_via_newtrschulz5`. This replaces the Python loop with a fused kernel that operates directly on pinned memory buffers passed from the optimizer, reducing Python interpreter overhead and enabling better occupancy on the A100's tensor cores compared to the current CPU-bound gradient processing inside the optimizer step.
3. Introduce selective compilation: compile only the forward pass of the GPT model (excluding the backward pass) with `mode='reduce-overhead'` and cache it globally, while keeping the backward pass uncompiled or compiled separately with `mode='default'` to avoid recompiling the heavy gradient accumulation logic repeatedly if any dynamic shapes occur (though static here). Ensure the `compile_forward` function wraps only `model.forward` and is called exactly once before DDP wrapping.
4. Modify the training loop to perform a "warmup" phase where the sliding window starts at 1 block but the learning rate scheduler applies a linear ramp-up for the first 50 steps instead of immediate decay, allowing the network to stabilize without the aggressive orthogonalization pressure initially, effectively smoothing the loss landscape faster.
5. Add explicit `torch.cuda.synchronize()` calls only around the validation start and end points, removing them from the inner training loop to reduce synchronization overhead, relying on the harness's global timing rather than per-step syncs that don't contribute to the score.

RISK
Incorrectly implementing the async loader may cause race conditions where tokens are loaded out of order or double-loaded, breaking the validation token count requirement. The Triton kernel must strictly match the float precision (bf16) and output shape of the existing Python loop; mismatched dtypes will corrupt gradients. Ensure the warmup schedule doesn't delay validation too much, violating the `step 
\end{promptlisting}
\end{promptbox}
\caption{\textbf{Example action for NanoGPT}}
\label{nanogpt:action-response}
\end{figure}

\FloatBarrier
\subsubsection{Hyperparameters\label{appendix:exp:nanogpt:hparams}}
See \Cref{tab:nanogpt:hparams}.
\begin{table}[htbp]
\centering
\caption{\textbf{NanoGPT configuration}. The upper block is shared by \gls{ours} and vanilla; the lower block is the complete set of settings on which they differ.}
\label{tab:nanogpt:hparams}
{\small
\begin{tabular}{llr}
\toprule
Group & Setting & Value \\
\midrule
Model   & actor and reference & Qwen3.5-4B \\
        & critic & Qwen3.5-4B, token-classification head \\
        & rollout engine & vLLM, \texttt{bfloat16} \\
\addlinespace
\gls{ppo} & advantage estimator & GAE, $\gamma = 1.0$, $\lambda = 1.0$ \\
        & actor / critic learning rate & $1\times10^{-6}$ / $1\times10^{-5}$ \\
        & train batch $=$ mini-batch & $64$ \\
        & clip ratio, value clip & $0.2$, $0.5$ \\
        & gradient clip & $1.0$ \\
        & entropy coefficient & $0.001$ \\
        & KL coefficient to the reference & $0.001$, fixed \\
\addlinespace
Rollout & temperature, top-$p$, top-$k$ & $0.7$, $0.8$, $20$ \\
        & presence penalty & $1.5$ \\
        & samples per prompt & $1$ \\
        & prompt / response length & $40{,}960$ / $2{,}560$ tokens \\
\addlinespace
Environment & wall-clock timeout $\delta_{\max}$ & $1500$\,s \\
        & time floor & $0.1$\,s \\
        & target loss $L^\star$, credited floor & $3.28$, $3.25$ \\
        & grader & $8\times$ A100 80GB per attempt \\
        & validation cadence & every $125$ steps \\
\addlinespace
Implementer & model, temperature & \texttt{GPT 5.6 Terra}, $0$ \\
        & response budget, attempts & $16{,}384$ tokens, $\le 3$ \\
        & judge & same model, $3$ votes at temperature $0$ \\
\addlinespace
Buffer  & selection width & top $64$ \\
        & large-edit quota & $0.5$ of slots, $\ge 20$ changed lines \\
\addlinespace
Compute & trainer & $4\times$ A100 80GB \\
        & graders & $5$ workers, $8\times$ A100 80GB each \\
        & repetitions & $2$ repetitions, each ran for $10$ steps \\
\midrule
\gls{ours} only & reward reference $L_{\mathrm{ref}}$ & $10$ \\
        & crossing bonus $B$, slope $w$ & $5$, $100$ \\
        & invalid reward & $-10$ \\
        & rate estimator & online \gls{niw}, valid rows only \\
        & forgetting factor, batches $M$ & $0.3$, $1{,}000$ \\
\addlinespace
Vanilla only & objective & time to reach $L^\star$ \\
        & depth exchange rate $w_t$ & $20{,}000$\,s per unit loss \\
        & invalid reward & $-2\,\delta_{\max}$ \\
\bottomrule
\end{tabular}}
\end{table}

%% file: appendix/experiments.tex
\subsection{Bandit Experiments, $|\gX| = 4$\label{appendix:exp:exact}}

\subsubsection{Regret Curves\label{appendix:exp:curves}}
\Cref{fig:bandit:summary} reports only the endpoint of each run at $T=30{,}000$.
\Cref{fig:bandit:curves:ctx} give the regret-per-step curves behind those endpoints.

\begin{figure}[htbp]
    \centering
    \includegraphics[width=\textwidth]{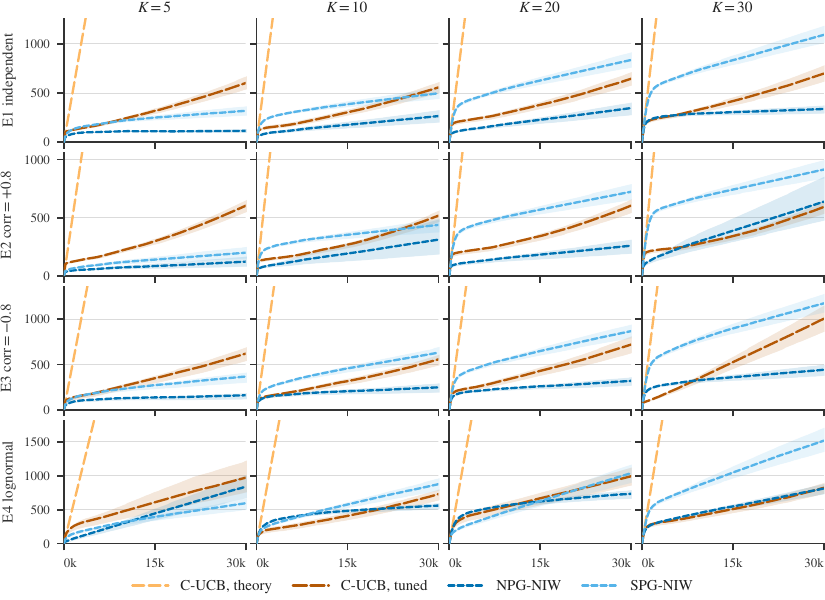}
    \caption{\textbf{Regret against step for $|\gX| = 4$ contexts}.}
    \label{fig:bandit:curves:ctx}
\end{figure}

\subsubsection{Hyperparameter Selection\label{appendix:exp:tuning}}
Every method that carries a hyperparameter is tuned via grid search.
The tuning seeds are disjoint from the $100$ evaluation seeds, and selection is by median final regret over the tuning seeds.
The hyperparameter search for every baseline are shown in \Cref{fig:bandit:tuning:ucb,fig:bandit:tuning:niw,fig:bandit:tuning:qvan}.
\begin{figure}[htbp]
    \centering
    \includegraphics[width=\textwidth]{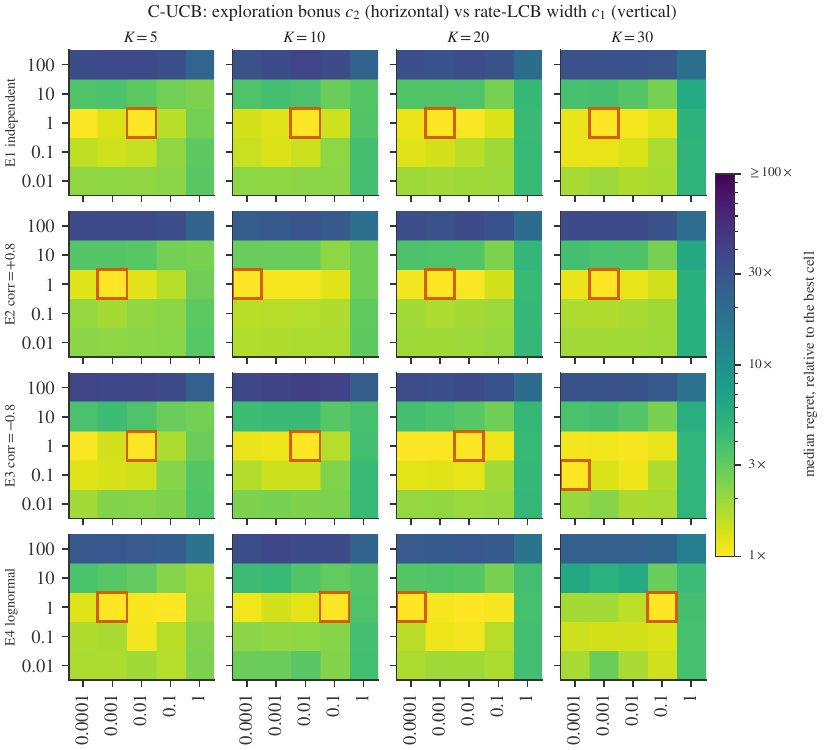}
    \caption{\textbf{Tuning grid for \gls{c-ucb}}. The horizontal axis is $c_2$ and the vertical axis is $c_1$.}
    \label{fig:bandit:tuning:ucb}
\end{figure}

\begin{figure}[htbp]
    \centering
    \includegraphics[width=\textwidth]{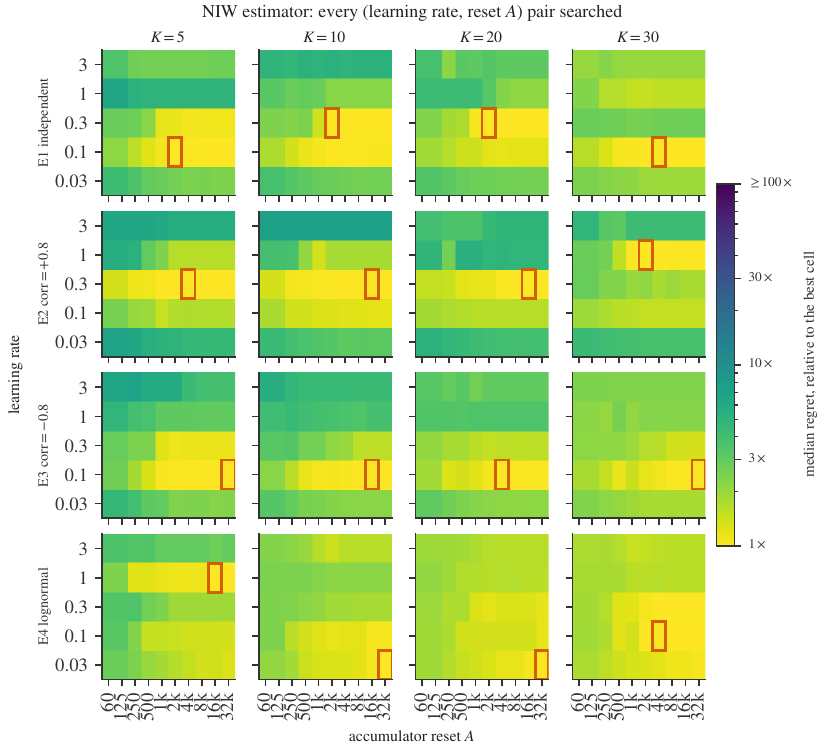}
    \caption{\textbf{Tuning grid for the \gls{niw}-\gls{npg}}, over the learning rate and accumulation window.}
    \label{fig:bandit:tuning:niw}
\end{figure}

\begin{figure}[htbp]
    \centering
    \includegraphics[width=\textwidth]{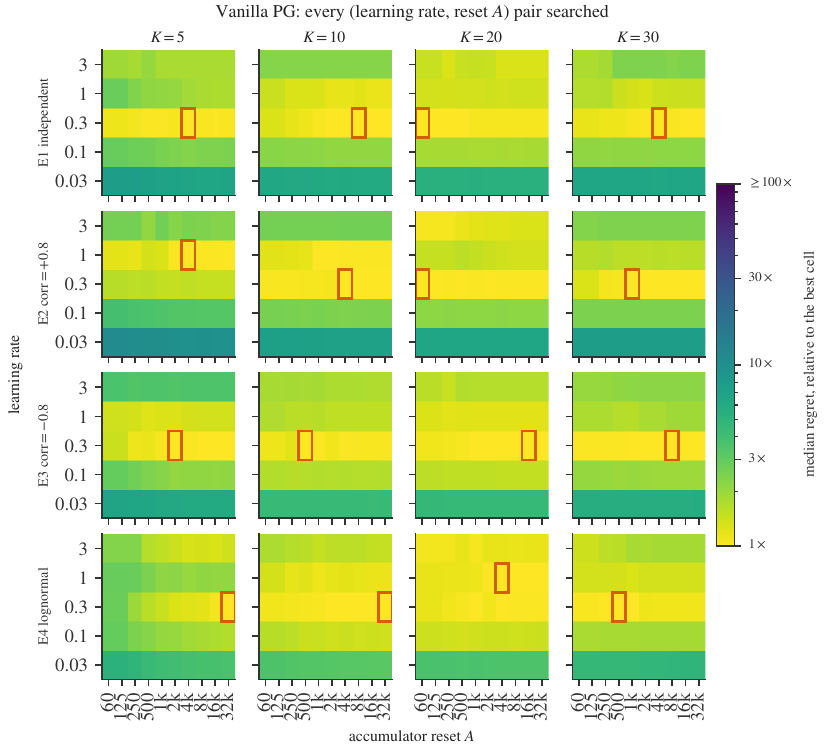}
    \caption{\textbf{Tuning grid for the SPG-NIW}, over the same learning rate and accumulation window grid as \Cref{fig:bandit:tuning:niw}.}
    \label{fig:bandit:tuning:qvan}
\end{figure}

\clearpage
\subsection{Bandit Experiments, $|\gX| > 4$\label{appendix:exp:approx}}

\subsubsection{Regret at the Horizon\label{appendix:exp:approx:summary}}
\Cref{fig:bandit:approx:summary} reports the endpoint ($T=30{,}000$) of each run at $|\gX| \in \{8, 16, 32, 64\}$.

\begin{figure}[htbp]
    \centering
    \includegraphics{figures/bandit/legend_strip_all.png}
    \begin{subfigure}{0.49\textwidth}
        \centering
        \includegraphics[width=\textwidth]{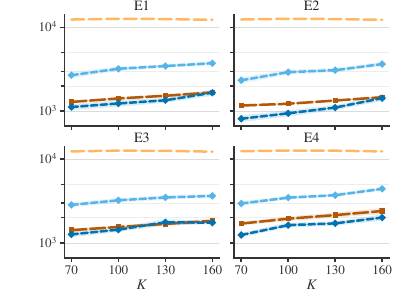}
        \caption{$|\gX| = 8$.}
        \label{fig:bandit:approx:summary:x8}
    \end{subfigure}
    \hfill
    \begin{subfigure}{0.49\textwidth}
        \centering
        \includegraphics[width=\textwidth]{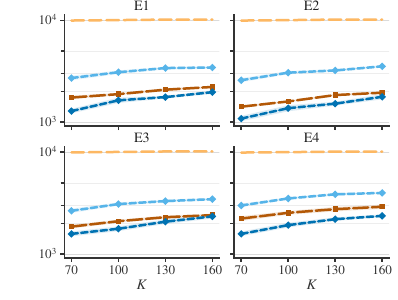}
        \caption{$|\gX| = 16$.}
        \label{fig:bandit:approx:summary:x16}
    \end{subfigure}
    \\[0.4em]
    \begin{subfigure}{0.49\textwidth}
        \centering
        \includegraphics[width=\textwidth]{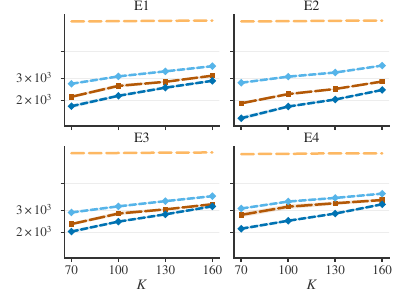}
        \caption{$|\gX| = 32$.}
        \label{fig:bandit:approx:summary:x32}
    \end{subfigure}
    \hfill
    \begin{subfigure}{0.49\textwidth}
        \centering
        \includegraphics[width=\textwidth]{figures/bandit/approx/summary_half_X64.pdf}
        \caption{$|\gX| = 64$.}
        \label{fig:bandit:approx:summary:x64}
    \end{subfigure}
    \caption{\textbf{Regret at the horizon against the number of arms $K$}, one panel per environment family, for the approximate \gls{c-ucb} setting.}
    \label{fig:bandit:approx:summary}
\end{figure}

\subsubsection{Regret Curves\label{appendix:exp:approx:curves}}
\Cref{fig:bandit:approx:summary,fig:bandit:summary} only contains regret at the endpoint $T=30{,}000$.
\Cref{fig:bandit:approx:curves} gives the regret per step at each context count.

\begin{figure}[htbp]
    \centering
    \begin{subfigure}{0.49\textwidth}
        \centering
        \includegraphics[width=\textwidth]{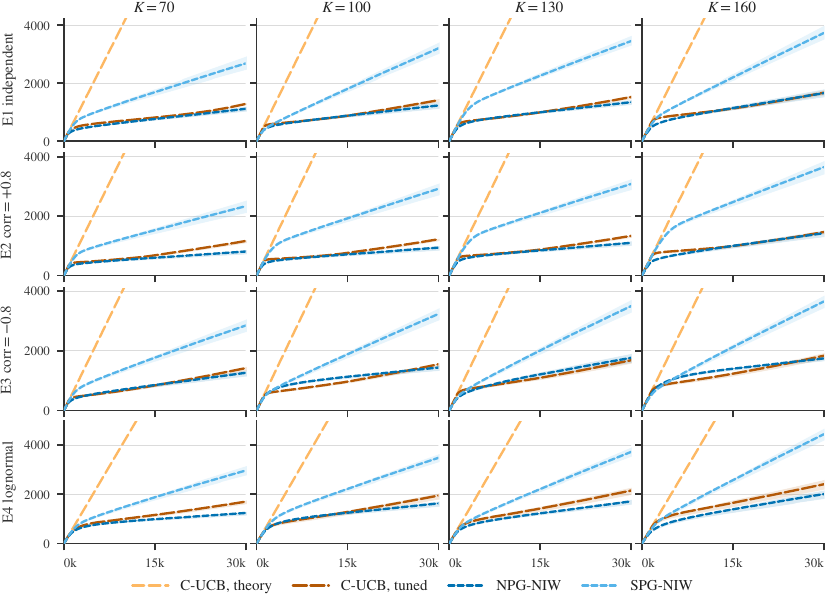}
        \caption{$|\gX| = 8$.}
    \end{subfigure}
    \hfill
    \begin{subfigure}{0.49\textwidth}
        \centering
        \includegraphics[width=\textwidth]{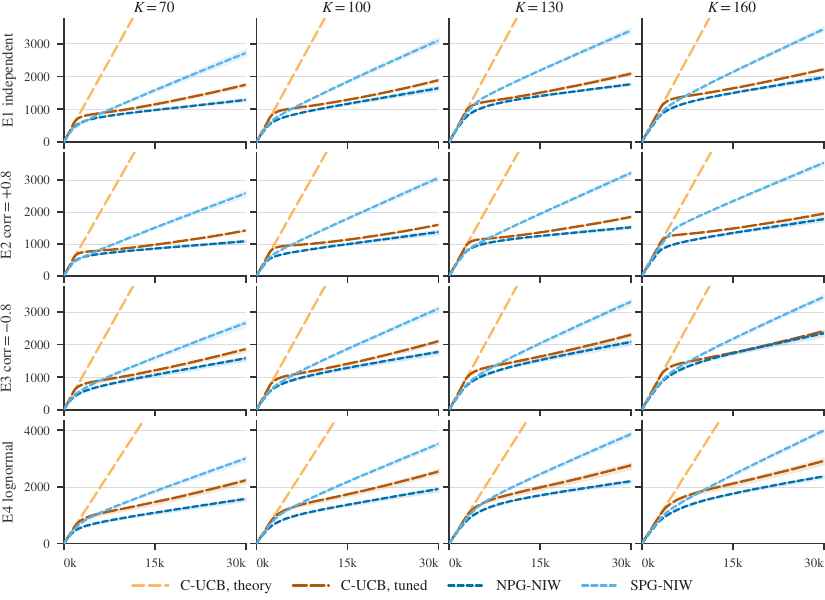}
        \caption{$|\gX| = 16$.}
    \end{subfigure}
    \\[0.4em]
    \begin{subfigure}{0.49\textwidth}
        \centering
        \includegraphics[width=\textwidth]{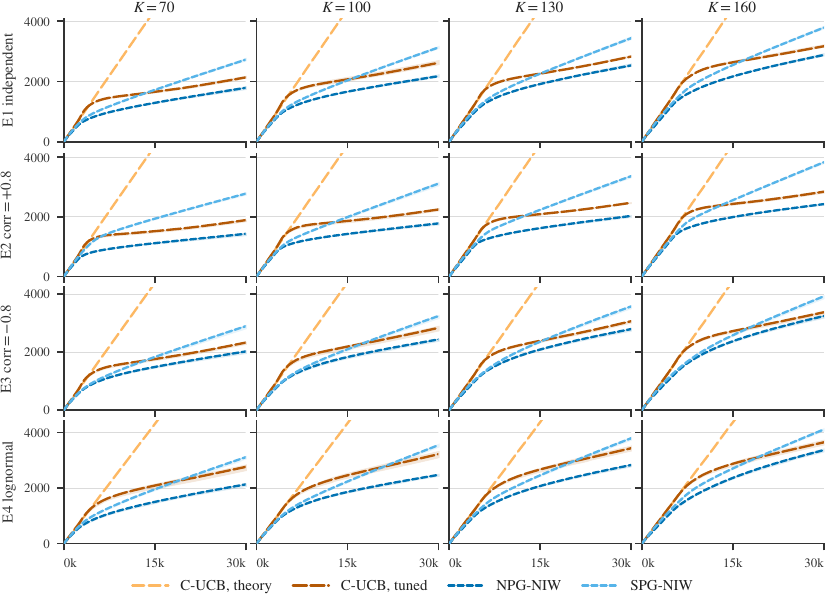}
        \caption{$|\gX| = 32$.}
    \end{subfigure}
    \hfill
    \begin{subfigure}{0.49\textwidth}
        \centering
        \includegraphics[width=\textwidth]{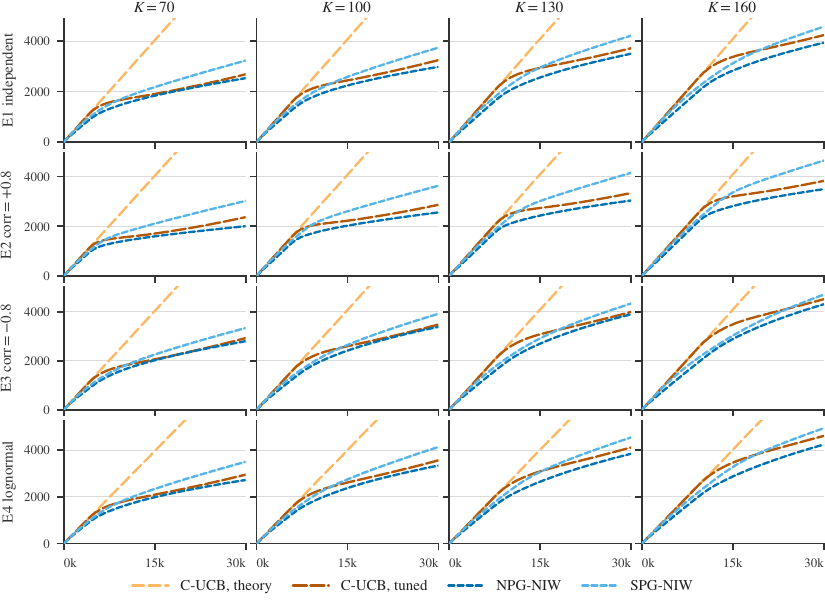}
        \caption{$|\gX| = 64$.}
    \end{subfigure}
    \caption{\textbf{Regret against step for the large context setting}, with environment families down the rows and arm counts across the columns.}
    \label{fig:bandit:approx:curves}
\end{figure}

\subsubsection{Hyperparameter Selection\label{appendix:exp:approx:tuning}}
The tuning grids are produced by the same procedure as \Cref{appendix:exp:tuning}, one grid per context count (\Cref{fig:bandit:approx:tuning:ucb,fig:bandit:approx:tuning:niw,fig:bandit:approx:tuning:qvan}).

\begin{figure}[htbp]
    \centering
    \begin{subfigure}{0.49\textwidth}
        \centering
        \includegraphics[width=\textwidth]{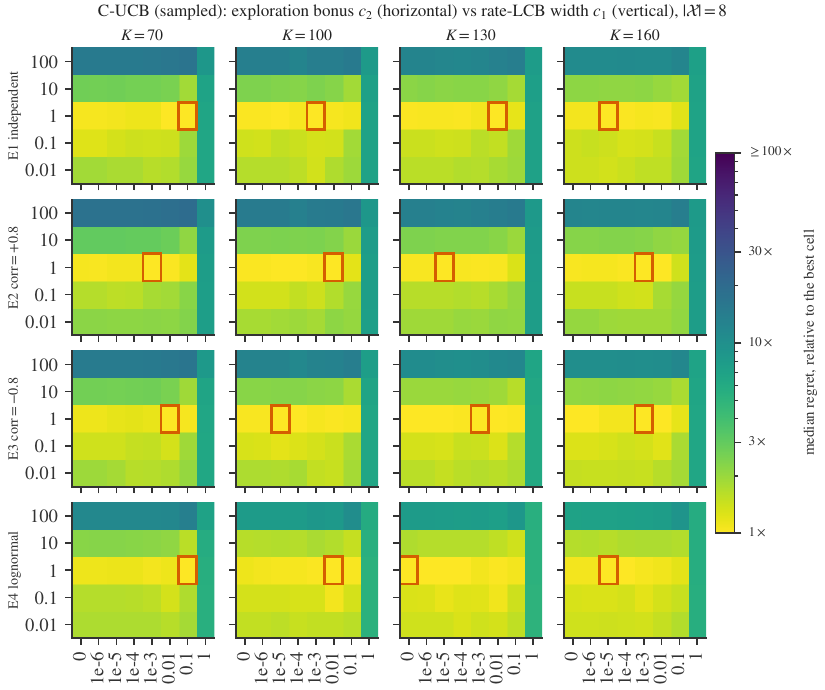}
        \caption{$|\gX| = 8$.}
    \end{subfigure}
    \hfill
    \begin{subfigure}{0.49\textwidth}
        \centering
        \includegraphics[width=\textwidth]{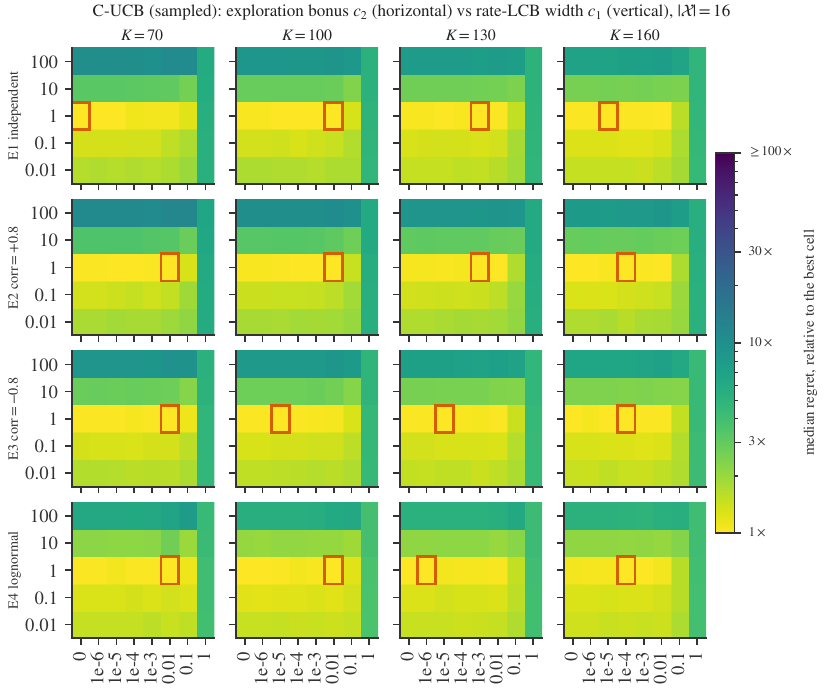}
        \caption{$|\gX| = 16$.}
    \end{subfigure}
    \\[0.4em]
    \begin{subfigure}{0.49\textwidth}
        \centering
        \includegraphics[width=\textwidth]{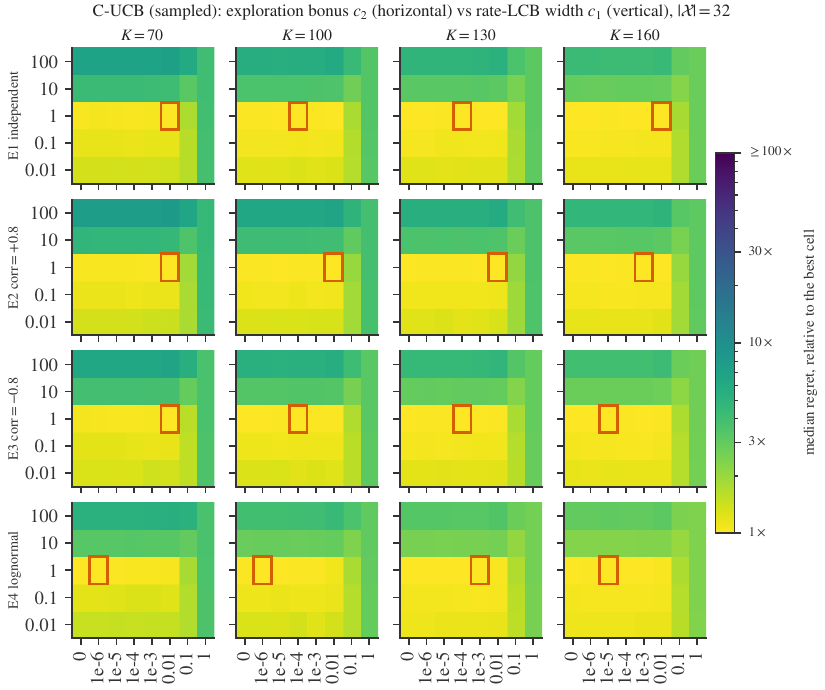}
        \caption{$|\gX| = 32$.}
    \end{subfigure}
    \hfill
    \begin{subfigure}{0.49\textwidth}
        \centering
        \includegraphics[width=\textwidth]{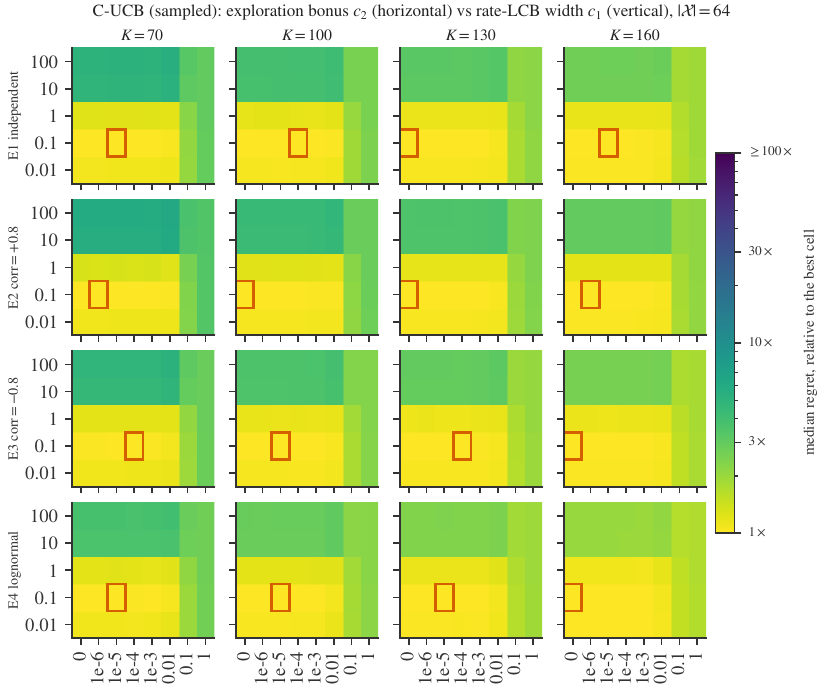}
        \caption{$|\gX| = 64$.}
    \end{subfigure}
    \caption{\textbf{Tuning grid for approximate \gls{c-ucb}}. The horizontal axis is $c_2$ and vertical axis is $c_1$.}
    \label{fig:bandit:approx:tuning:ucb}
\end{figure}

\begin{figure}[htbp]
    \centering
    \begin{subfigure}{0.49\textwidth}
        \centering
        \includegraphics[width=\textwidth]{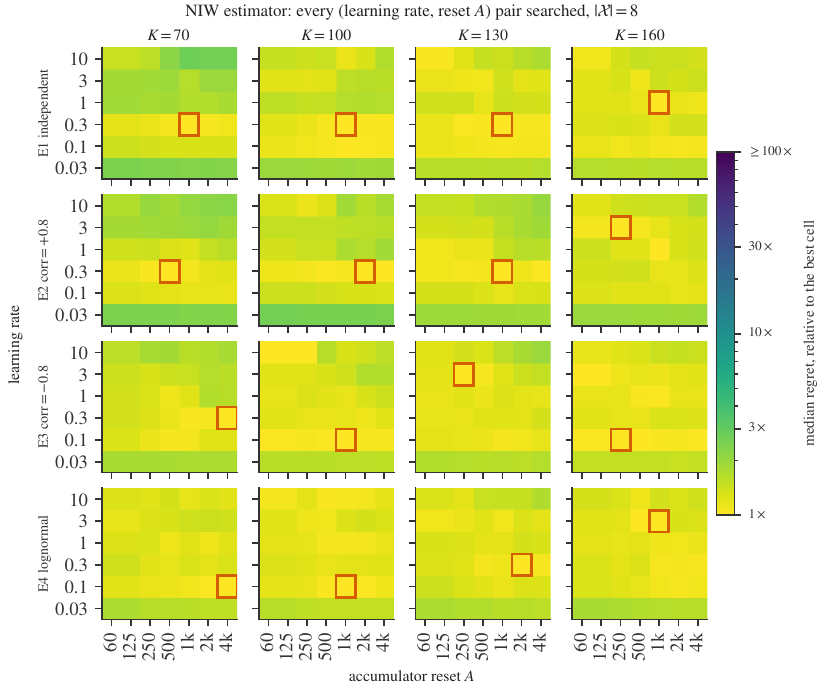}
        \caption{$|\gX| = 8$.}
    \end{subfigure}
    \hfill
    \begin{subfigure}{0.49\textwidth}
        \centering
        \includegraphics[width=\textwidth]{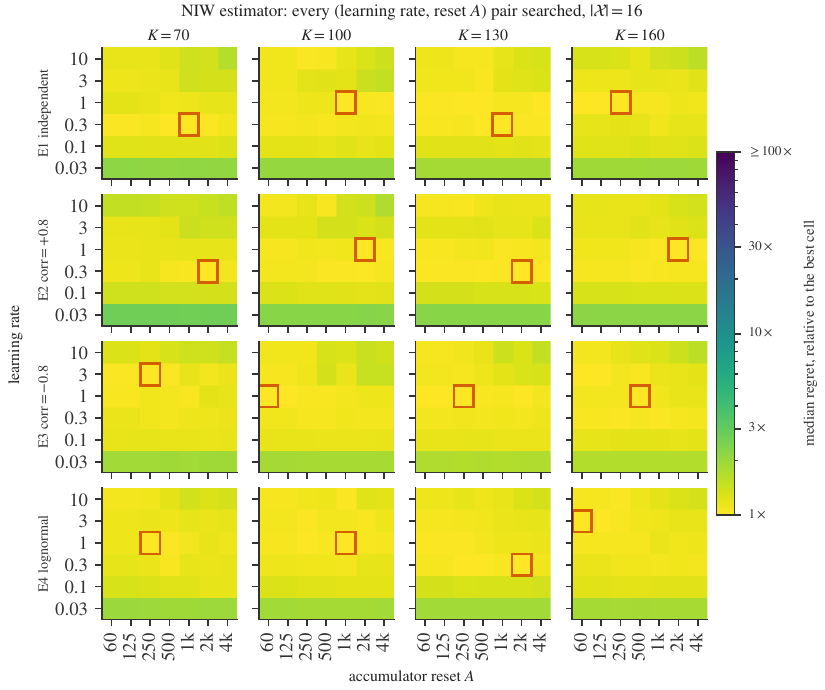}
        \caption{$|\gX| = 16$.}
    \end{subfigure}
    \\[0.4em]
    \begin{subfigure}{0.49\textwidth}
        \centering
        \includegraphics[width=\textwidth]{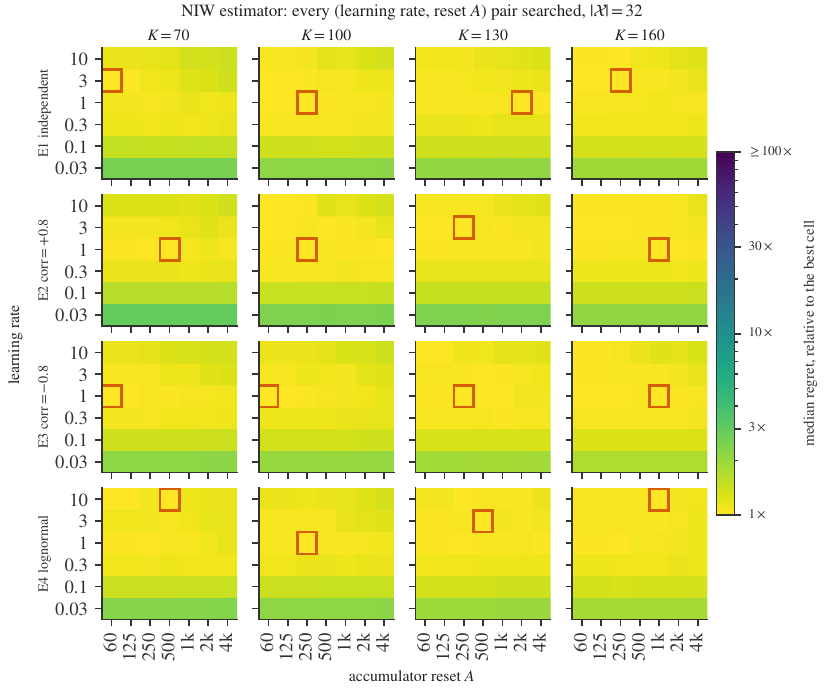}
        \caption{$|\gX| = 32$.}
    \end{subfigure}
    \hfill
    \begin{subfigure}{0.49\textwidth}
        \centering
        \includegraphics[width=\textwidth]{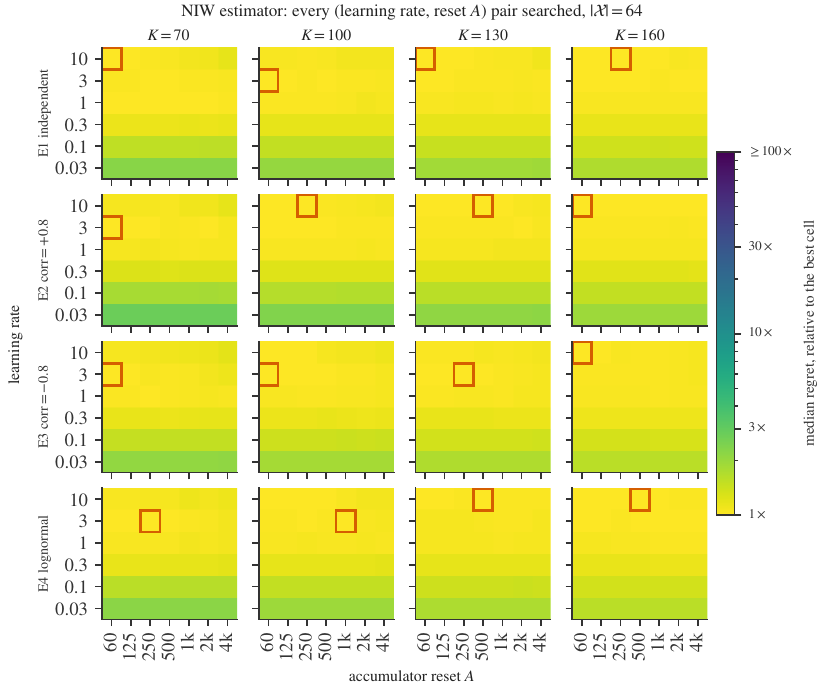}
        \caption{$|\gX| = 64$.}
    \end{subfigure}
    \caption{\textbf{Tuning grid for the \gls{niw}-\gls{npg}}, over the learning rate and accumulation window.}
    \label{fig:bandit:approx:tuning:niw}
\end{figure}

\begin{figure}[htbp]
    \centering
    \begin{subfigure}{0.49\textwidth}
        \centering
        \includegraphics[width=\textwidth]{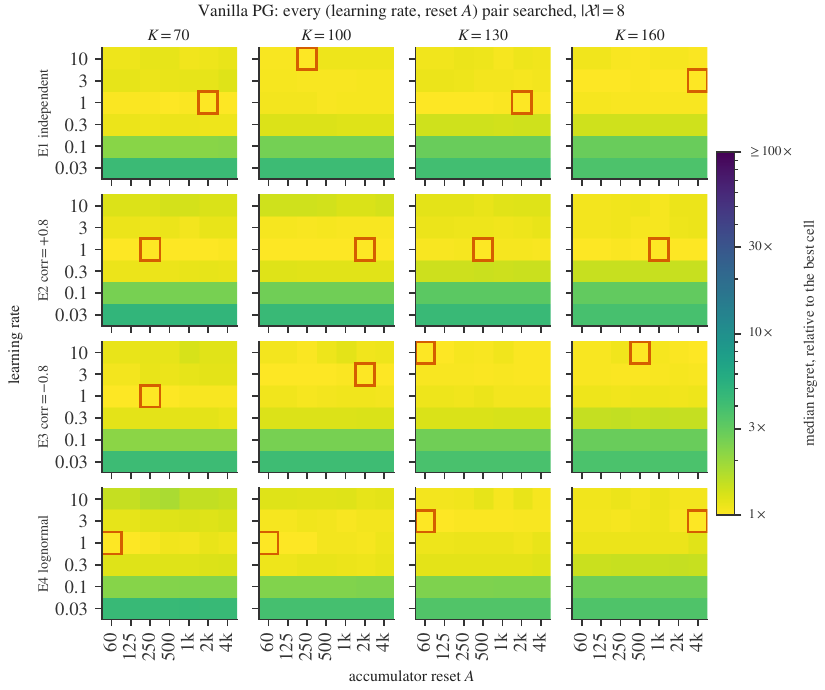}
        \caption{$|\gX| = 8$.}
    \end{subfigure}
    \hfill
    \begin{subfigure}{0.49\textwidth}
        \centering
        \includegraphics[width=\textwidth]{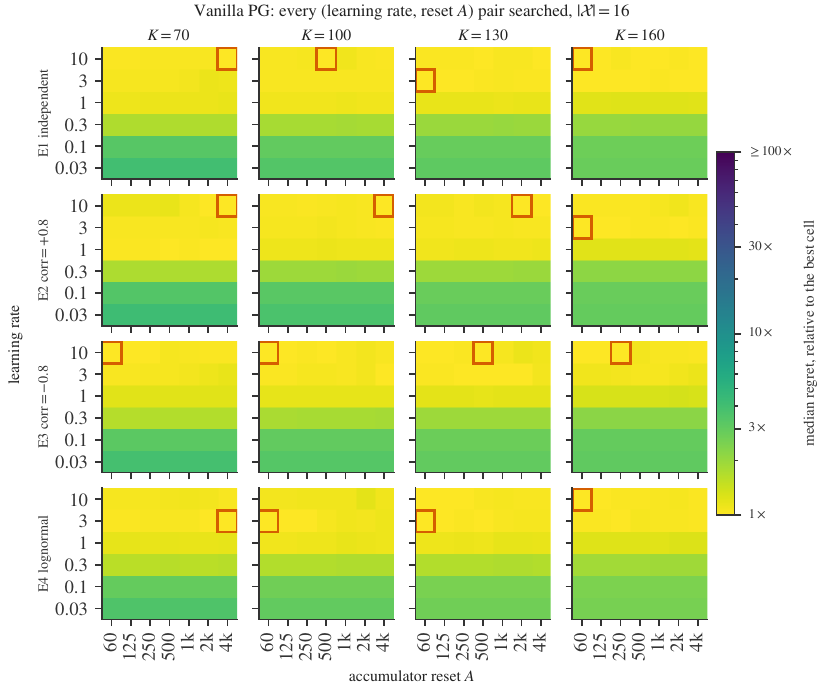}
        \caption{$|\gX| = 16$.}
    \end{subfigure}
    \\[0.4em]
    \begin{subfigure}{0.49\textwidth}
        \centering
        \includegraphics[width=\textwidth]{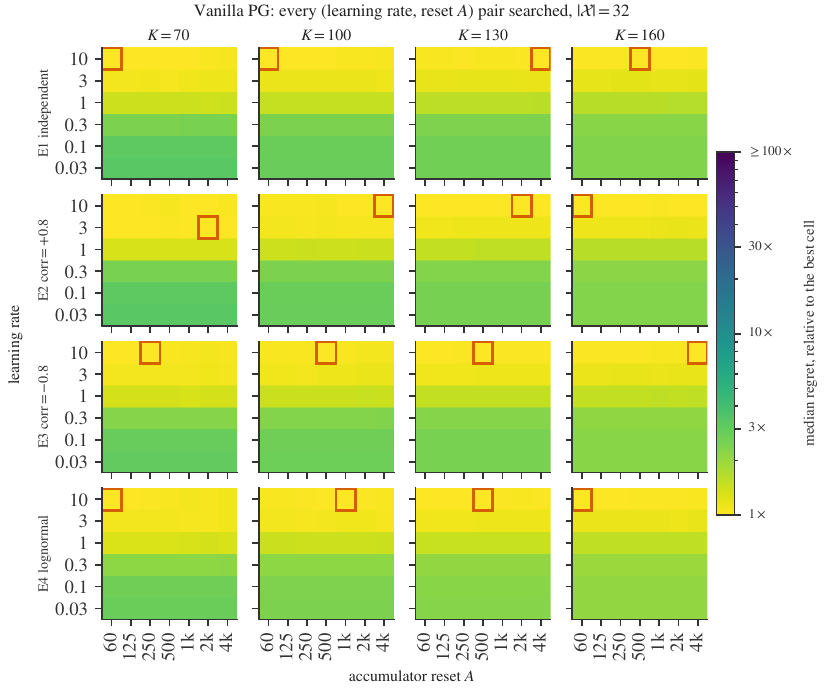}
        \caption{$|\gX| = 32$.}
    \end{subfigure}
    \hfill
    \begin{subfigure}{0.49\textwidth}
        \centering
        \includegraphics[width=\textwidth]{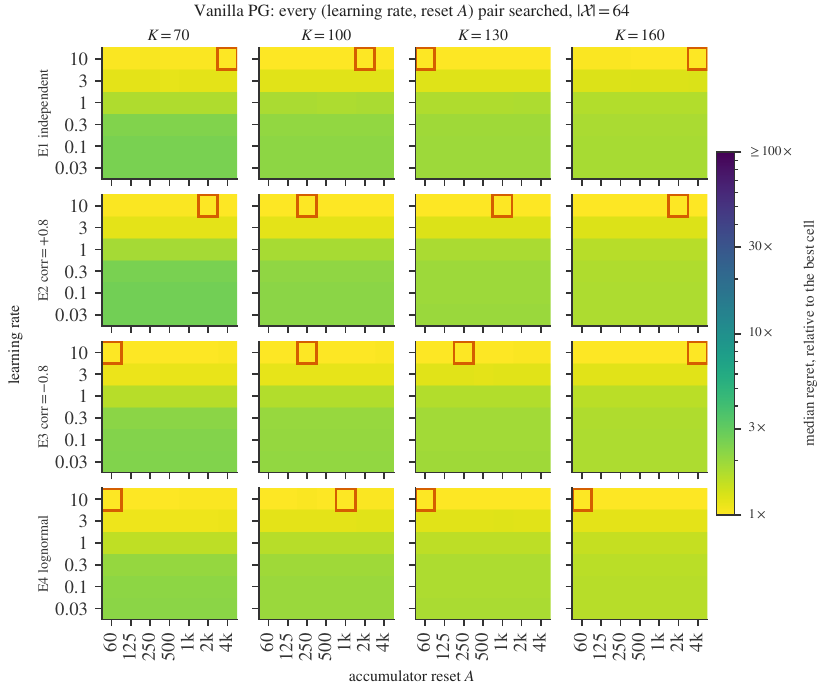}
        \caption{$|\gX| = 64$.}
    \end{subfigure}
    \caption{\textbf{Tuning grid for the SPG-NIW}, over the same learning rate and accumulation window grid as \Cref{fig:bandit:approx:tuning:niw}.}
    \label{fig:bandit:approx:tuning:qvan}
\end{figure}

\FloatBarrier
\subsection{\gls{mle} Agent\label{appendix:exp:mlebench}}
\paragraph{Competition identifiers.}
\Cref{tab:mlebench_meanatT} abbreviates the MLE-Bench Lite
competition ids to fit; in full they are:

aerial-cactus-identification,
aptos2019-blindness-detection,
denoising-dirty-documents,
dog-breed-identification,
dogs-vs-cats-redux-kernels-edition,
histopathologic-cancer-detection,
detecting-insults-in-social-commentary,
jigsaw-toxic-comment-classification-challenge,
leaf-classification,
mlsp-2013-birds,
nomad2018-predict-transparent-conductors,
new-york-city-taxi-fare-prediction,
random-acts-of-pizza,
plant-pathology-2020-fgvc7,
ranzcr-clip-catheter-line-classification,
siim-isic-melanoma-classification,
spooky-author-identification,
tabular-playground-series-dec-2021,
tabular-playground-series-may-2022,
text-normalization-challenge-english-language,
text-normalization-challenge-russian-language, and
the-icml-2013-whale-challenge-right-whale-redux.

\begin{figure}[htbp]
    \centering
    \includegraphics[width=\textwidth]{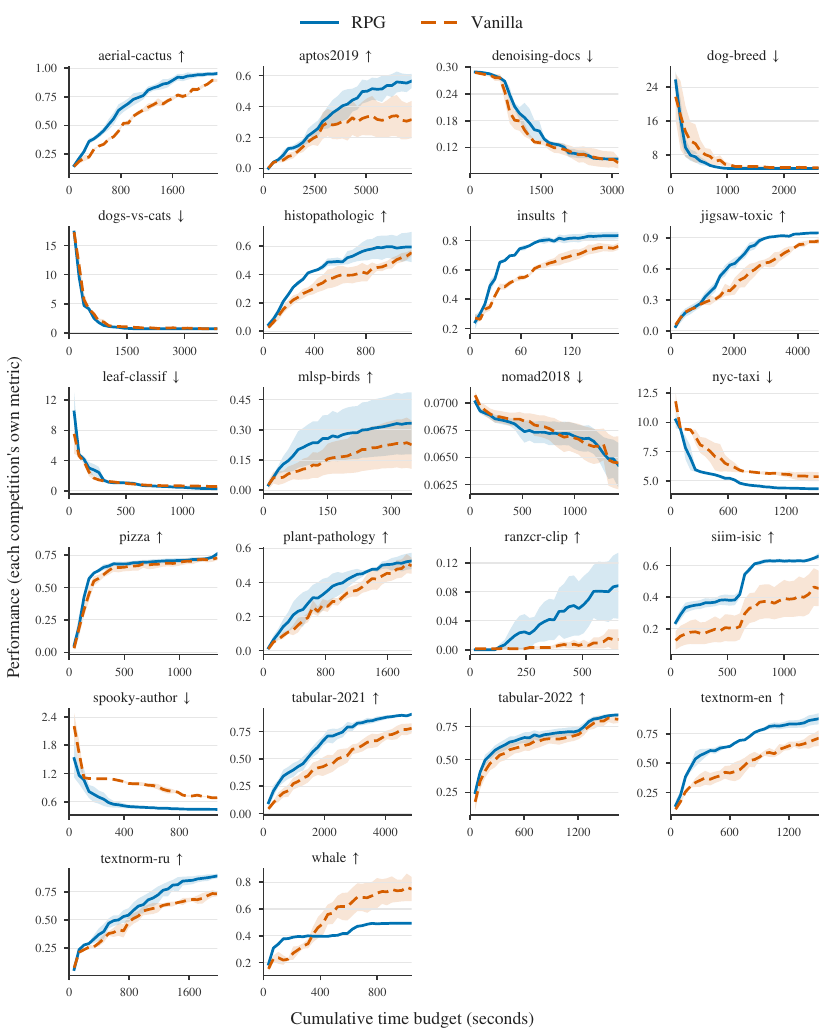}
    \caption{\textbf{Per-task learning curves on MLE-Bench Lite}, the evaluation method is same as
    \Cref{tab:mlebench_meanatT}. \Cref{tab:mlebench_meanatT} reports each task at the final time budget, the figure shows the entire learning process.}
    \label{fig:mlebench:per_task}
\end{figure}

\begin{figure}[htbp]
    \centering
    \includegraphics[width=\textwidth]{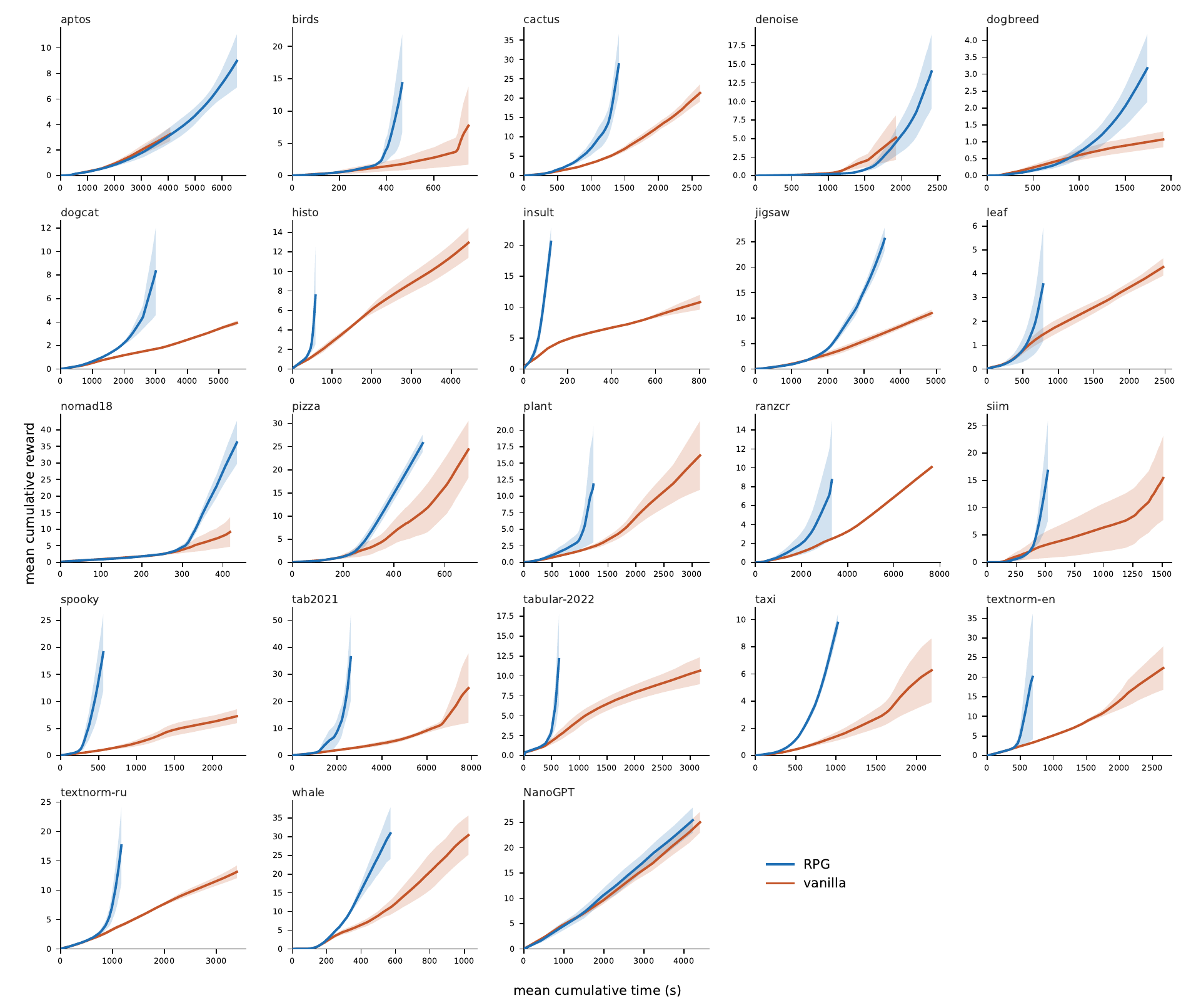}
    \caption{\textbf{Mean cumulative time (x-axis) and mean cumulative reward (y-axis).} \gls{ours} optimizes for long-term reward per unit of time, so its curve goes higher as mean cumulative time increases.}
    \label{fig:mlebench:cumreward}
\end{figure}

\begin{figure}[htbp]
    \centering
    \includegraphics[width=0.8\textwidth]{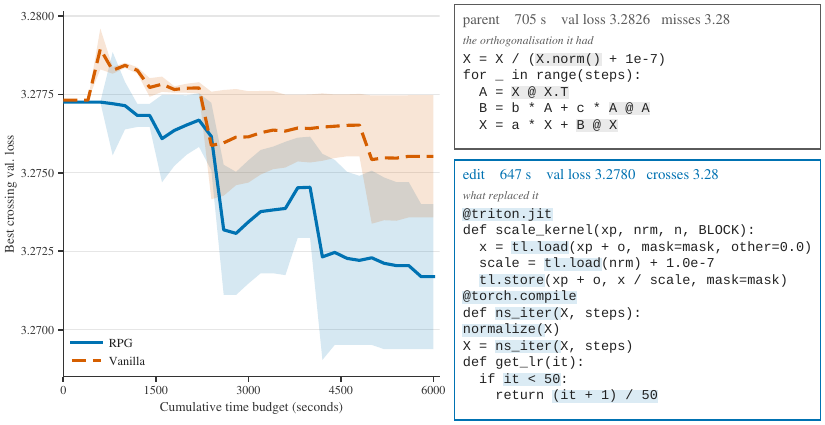}
    \caption{\textbf{Budget curve and an example self-improvement step on NanoGPT, with root script included for reference}. 
    (Left) At the final budget, \gls{ours} improves over vanilla by $85.7\%$. 
    Time 0 reflects the root parent script's validation loss.
    We do not count the root script's loss in the best crossings to reflect the dynamics of the policies.
    (Right) Example of a self-improvement edit by \gls{ours}, which rewrites Newton--Schulz in the Muon optimizer as a triton kernel with a compiled iteration and adds a learning-rate warmup. The new edit crosses the loss target using shorter time.}
    \label{fig:nanogpt:budget-root}
\end{figure}